\documentclass[twoside, 11pt]{article}
\usepackage{jmlr2e}
\usepackage{latexsym,amsmath,amssymb,epic,eepic,multirow}
\usepackage{multirow}
\usepackage{subfigure}

\RequirePackage[OT1]{fontenc}
\RequirePackage{amsmath,natbib}
\RequirePackage[colorlinks]{hyperref}
\RequirePackage{hypernat}
\usepackage{amsfonts}
\usepackage{amssymb}
\usepackage{amstext}
\usepackage{parskip}
\bibpunct{[}{]}{;}{a}{,}{,}

\usepackage{times}
\usepackage{graphicx}
\usepackage{epsfig}

\newcommand{\R}{\ensuremath{\mathbb R}}

\newcommand{\Set}{\ensuremath{\mathcal S}}
\newcommand{\F}{\ensuremath{\mathcal F}}

\newcommand{\PDcone}{\ensuremath{{\mathcal{S}^p_{++}}}}
\newcommand{\SubE}{\ensuremath{{\mathcal{S}^p_{E}}}}
\newcommand{\prob}[1]{\ensuremath{{\mathbb P}\left(#1\right)}}

\newcommand{\expct}[1]{\ensuremath{{\mathbb E}#1}}
\newcommand{\expect}[2]{\ensuremath{\text{{\bf E}$_{#1}\left[#2\right]$}}}
\newcommand{\size}[1]{\ensuremath{\left|#1\right|}}

\newcommand{\argmin}{\operatorname{argmin}}
\newcommand{\ul}{\underline}
\newcommand{\e}{\epsilon}
\newcommand{\ve}{\varepsilon}

\newcommand{\half}{\ensuremath{\frac{1}{2}}}

\newcommand{\silent}[1]{}

\newcommand{\ip}[1]{\langle{\,#1\,}\rangle}

\newcommand{\C}{{\mathcal C}}
\newcommand{\event}{{\mathcal E}}

\newcommand{\M}{{\mathcal M}}
\newcommand{\T}{{\mathcal T}}
\newcommand{\TT}{{\mathbb{T}}}
\newcommand{\RE}{{\mathcal R}}
\newcommand{\W}{{\mathbb W}}
\newcommand{\X}{{\mathcal X}}

\newcommand{\var}[1]{\ensuremath{{\textsf{Var}} \left(#1\right)}}

\newcommand{\basepen}{\ensuremath{\lambda_{\sigma,a,p}}}

\newcommand{\init}{\text{\rm init}}
\newcommand{\drop}{{\mathcal D}}
\newcommand{\sel}{\ensuremath{I}}

\newcommand{\inv}[1]{\frac{1}{#1}}
\newcommand{\abs}[1]{\left\lvert#1\right\rvert}
\newcommand{\twonorm}[1]{\left\lVert#1\right\rVert_2}
\newcommand{\fnorm}[1]{\left\lVert#1\right\rVert_F}
\newcommand{\shtwonorm}[1]{\lVert#1\rVert_2}
\newcommand{\norm}[1]{\left\lVert#1\right\rVert}

\newcommand{\func}[1]{\ensuremath{\mathrm{#1}}}
\newcommand{\diag}{\func{diag}}
\newcommand{\offdiag}{\func{offd}}
\newcommand{\mvec}{\func{vec}}
\def\lin{\mathop{\text{lin}\kern.2ex}}
\def\dom{\mathop{\text{dom}\kern.2ex}}
\newcommand{\tr}{{\rm tr}}

\newcommand{\bias}{\text{\rm bias}}

\newcommand{\bens}{\begin{eqnarray*}}
\newcommand{\eens}{\end{eqnarray*}}
\newcommand{\ben}{\begin{eqnarray}}
\newcommand{\een}{\end{eqnarray}}

\newcommand{\beq}{\begin{equation}}
\newcommand{\eeq}{\end{equation}}

\newcommand{\Var}{{\rm Var}}

\def\supp{\mathop{\text{\rm supp}\kern.2ex}}
\def\conv{\mathop{\text{\rm conv}\kern.2ex}}
\def\argmin{\mathop{\text{arg\,min}\kern.2ex}}
\def\half{\frac{1}{2}}
\let\hat\widehat
\let\tilde\widetilde

{\end{eqnarray}\end{subequations}\hskip-4.0pt}%
\newcommand{\qed}{$\blacksquare$}
\newenvironment{proofof}[1]{\hspace*{20pt}{\it Proof}{ of #1}.\hskip10pt}{\qed\vskip5pt}
\newenvironment{proofof2}{}{\qed\vskip5pt}
\jmlrheading{Volume}{Year}{Page}{Date}{Date}{Shuheng Zhou, Philipp R\"utimann, Min Xu, and Peter B\"uhlmann}

\ShortHeadings{High-dimensional Covariance Estimation}{Zhou, R\"utimann, Xu, and B\"uhlmann}

\firstpageno{1}

\begin{document}

\title{High-dimensional Covariance Estimation 
Based On Gaussian Graphical Models}

\author{\name Shuheng Zhou \email shuhengz@umich.edu \\  
        \addr Department of Statistics \\ 
              University of Michigan\\
              Ann Arbor, MI 48109-1041, USA\\
        \name Philipp R\"utimann \email rutimann@stat.math.ethz.ch \\
        \addr Seminar for Statistics \\
        			ETH Z\"urich \\
        			8092 Z\"urich, Switzerland\\
        \name Min Xu \email minx@cs.cmu.edu \\
        \addr Machine Learning Department\\
        			Carnegie Mellon University\\
        			Pittsburgh, PA 15213-3815, USA\\
        \name Peter B\"uhlmann \email buhlmann@stat.math.ethz.ch \\
        \addr Seminar for Statistics \\
        			ETH Z\"urich \\
        			8092 Z\"urich, Switzerland}
        					          
\editor{}

\maketitle
\begin{abstract}		
\noindent
Undirected graphs are often used to describe
high dimensional distributions. Under sparsity
conditions, the graph can be estimated using $\ell_1$-penalization methods. 
We propose and study the following method. We combine a multiple
regression approach with ideas of thresholding and refitting: 
first we infer a sparse
undirected graphical model structure via thresholding of each 
among many $\ell_1$-norm
penalized regression functions; we then estimate the covariance matrix
and its inverse using the maximum likelihood estimator.
We show that under suitable conditions, this approach yields consistent
estimation in terms of graphical structure and fast convergence rates with
respect to the operator and Frobenius norm for the covariance matrix and
its inverse. 
We also derive an explicit bound for the Kullback Leibler divergence.
\end{abstract}


\section{Introduction}
There have been a lot of recent activities for estimation of high-dimensional
covariance and inverse covariance matrices where the dimension $p$ of the
matrix may greatly exceed the sample size $n$. High-dimensional covariance
estimation can
be classified into two main categories, one which relies on a natural
ordering among the variables~\citep{WP03,BL04,HPL06,FB07,BL08,LRZ08} 
and one where no natural ordering is given and estimators are permutation 
invariant with respect to indexing the 
variables~\citep{YL07,FHT07,ABG08,BGA08,RBLZ08}.
We focus here on the latter class with permutation
invariant estimation and we aim for an estimator which is accurate for both
the covariance matrix $\Sigma$ and its inverse, the precision matrix
$\Sigma^{-1}$.
A popular approach for obtaining a permutation invariant estimator which is
sparse in the estimated precision matrix $\hat{\Sigma}^{-1}$ is given by the
$\ell_1$-norm regularized maximum-likelihood estimation, also known as the
GLasso~\citep{YL07,FHT07,BGA08}. The GLasso approach is
simple to use, at least when relying on publicly available software such as
the \texttt{glasso} package in \texttt{R}. Further improvements have been
reported when using some SCAD-type 
penalized maximum-likelihood estimator~\citep{LaFa09} or an 
adaptive GLasso procedure~\citep{FFW09},
which can be thought of as a two-stage procedure. 
It is well-known from linear regression that such
two- or multi-stage methods effectively address some bias problems which
arise from $\ell_1$-penalization~\citep{Zou06,CT07,Mei07,ZouLi08,BuehMeier08,Zhou09th,Zhou10}.

In this paper we develop a new method for estimating graphical 
structure and parameters for multivariate Gaussian distributions using
a multi-step procedure, which we call G{\bf e}lat{\bf o} 
(Graph {\bf e}stimation with Lasso and Thresh{\bf o}lding).
Based on an $\ell_1$-norm regularization and thresholding method
in a first stage, we infer a sparse undirected graphical model, 
i.e. an estimated Gaussian conditional independence graph, 
and we then perform unpenalized maximum likelihood 
estimation (MLE) for the covariance $\Sigma$ and its 
inverse $\Sigma^{-1}$ based on the estimated graph.
We make the following theoretical contributions:
(i) Our method allows us to select a graphical structure 
which is sparse. In some sense we select only the important edges 
even though there may be many non-zero edges in the graph.
(ii) Secondly, we evaluate the quality of the graph we have selected by
showing consistency and establishing a fast rate of convergence with
respect to the operator and 
Frobenius norm for the estimated inverse covariance matrix;
under sparsity constraints, 
the latter is of lower order than the corresponding results for the 
GLasso~\citep{RBLZ08} and for the SCAD-type estimator~\citep{LaFa09}.
(iii) We show predictive risk consistency and provide a rate of convergence
of the estimated covariance matrix.
(iv) Lastly, we show general results for the MLE, 
where only {\em approximate} graph structures are given as input.
Besides these theoretical advantages, 
we found empirically that our graph based method performs 
better in general, and sometimes substantially better than the GLasso,
while we never found it clearly worse. Moreover, we compare it with an
adaptation of the method Space \citep{PWZZ09}.
Finally, our algorithm is simple and is comparable to the GLasso
both in terms of computational time and implementation complexity.

There are a few key motivations and consequences for proposing 
such  an approach based on graphical modeling.
We will theoretically show that there are cases where our graph
based method can accurately estimate conditional independencies 
among variables, i.e. the zeroes of $\Sigma^{-1}$, 
in situations where GLasso fails.  The fact that GLasso easily fails to 
estimate the zeroes of $\Sigma^{-1}$ has been recognized by~\cite{Mei08}
and it has been discussed in more details in~\cite{RWRY08}.
Closer relations to existing work are primarily regarding our first stage
of estimating the structure of the graph. We follow the nodewise regression
approach from~\cite{MB06} but we make use of recent results
for variable selection in linear models assuming the much weaker restricted
eigenvalue condition~\citep{BRT09,Zhou10} instead of the restrictive
neighborhood stability condition~\citep{MB06} or the
equivalent irrepresentable condition~\citep{ZhaoYu06}. 
In some sense, the novelty of our theory extending beyond~\cite{Zhou10}
is the analysis for covariance and inverse covariance estimation 
and for risk consistency based on an estimated sparse graph as we 
mentioned above. 
Our regression and thresholding results build 
upon analysis of the thresholded Lasso estimator as studied in~\cite{Zhou10}.
Throughout our analysis, the sample complexity is one of the key 
focus point, which builds upon results in~\cite{Zhou10sub,RZ11}.
Once the zeros are found, a constrained maximum likelihood
estimator of the covariance can be computed, which was shown
in~\cite{ChDR07}; it was unclear what the properties of such a procedure 
would be. Our theory answers such questions.
As a two-stage method, our approach is also related to the adaptive 
Lasso~\citep{Zou06} which has been analyzed for high-dimensional 
scenarios in~\cite{HMZ08,ZGB09,GBZ10}. 
Another relation can be made to the method by~\cite{RuBu09}
for covariance and inverse covariance estimation based on a
directed acyclic graph. This relation has only methodological character:
the techniques and algorithms used in~\cite{RuBu09}
are very different and from a practical point of view, their approach 
has much higher degree of complexity in terms of computation and 
implementation, since estimation of an equivalence class of directed acyclic
graphs is difficult and cumbersome. There has also been work that focuses
on estimation of sparse directed Gaussian graphical model. \cite{V2010}
proposes a multiple regularized regression procedure for estimating a
precision matrix with sparse Cholesky factors, which correspond to a sparse
directed graph. He also computes non-asymptotic Kullback Leibler risk bound
of his procedure for a class of regularization functions. It is important
to note that directed graph estimation requires a fixed good ordering of
the variables a priori.


\noindent{\bf Notation.} 
We use the following notation.
Given a  graph $G =(V, E_0)$, where $V = \{1, \ldots, p\}$ is 
the set of vertices and $E_0$ is the set of undirected edges.
we use $s^i$ to denote the degree for node $i$, that is, 
the number of edges in $E_0$ connecting to node $i$. 
For an edge set $E$, we let $|E|$ denote its size. 
We use $\Theta_0 = \Sigma_0^{-1}$ and $\Sigma_0$ to 
refer to the true precision and 
covariance matrices respectively from now on.
We denote the number of non-zero elements of $\Theta$ by $\supp(\Theta)$.
For any 
matrix $W = (w_{ij})$, let $|W|$ denote the determinant of $W$, ${\rm tr}(W)$
the trace of $W$.
Let $\varphi_{\max}(W)$ and $\varphi_{\min}(W)$ be the
largest and smallest eigenvalues, respectively. 
We write
$\diag (W)$ for a diagonal matrix with the same diagonal as $W$
and $\offdiag(W) = W - \diag(W)$.
The matrix Frobenius norm is given by
$\norm{W}_F = \sqrt{\sum_i\sum_j w_{ij}^2}$. The operator norm
$\twonorm{W}^2$ is given by $\varphi_{\max}(WW^T)$.  We write $| \cdot |_1$
for the $\ell_1$ norm of a matrix vectorized, i.e., for a matrix
$|W|_1 = \norm{\mvec W}_1 = \sum_{i}\sum_j |w_{ij}|$, and sometimes
write $\norm{W}_0$ for the number of non-zero entries in the matrix.
For an index set $T$ and a matrix $W = [w_{ij}]$, write $W_T \equiv
(w_{ij} I( (i,j) \in T))$, where $I(\cdot)$ is the indicator function.

\section{The model and the method}

We assume a multivariate Gaussian model
\begin{eqnarray}
\label{eq::rand-des}
X = (X_1,\ldots ,X_p) \sim {\cal N}_p(0,\Sigma_0), \; \; \text{ where } \; \Sigma_{0,ii} =1.
\end{eqnarray}
The data is generated by
$X^{(1)},\ldots ,X^{(n)}\ \mbox{i.i.d.} \sim {\cal N}_p(0,\Sigma_0)$.
Requiring the mean vector and all variances being equal to zero and one
respectively is not a real restriction and in practice, we can
easily center and scale the data.
We denote the concentration matrix by $\Theta_0 = \Sigma_0^{-1}$.    

Since we will use a nodewise regression procedure, as described below in
Section \ref{subsec.nodewisereg}, we consider a regression formulation of
the model. Consider many regressions, where we regress one variable against
all others:
\begin{eqnarray}
\label{eq::regr}
& &X_i = \sum_{j \neq i} \beta_j^i X_j + V_i\ (i=1,\ldots ,p), \; \; \text{ where } \\
\label{eq::regressions}
& &V_i \sim {\cal N}(0, \sigma_{V_i}^2)\ \mbox{independent of}\ \{X_j; j \neq
i\}\ (i=1,\ldots ,p). 
\end{eqnarray}
There are explicit relations between the regression coefficients, error
variances and the concentration matrix $\Theta_0 = (\theta_{0,ij})$:
\begin{eqnarray}\label{regr-betatheta}
\beta_j^i = - \theta_{0,ij}/\theta_{0,ii},\ \Var(V_i) 
:= \sigma_{V_i}^2 = 1/\theta_{0, ii}\ (i,j=1,\ldots ,p).
\end{eqnarray}
Furthermore, it is well known that for Gaussian distributions, conditional
independence is encoded in $\Theta_0$, and due to (\ref{regr-betatheta}),
also in the regression coefficients: 
\begin{eqnarray}\nonumber 
& &X_i\ \mbox{is conditionally dependent of}\ X_j\ \mbox{given}\ \{X_k;\ k
\in \{1,\ldots ,p\} \setminus \{i,j\}\} \\
\label{eq::cond-beta}
& \Longleftrightarrow & \theta_{0,ij} \not= 0 
\; \; \Longleftrightarrow \; \beta_i^j \not= 0\ \mbox{and}\ \beta_j^i \not= 0.
\end{eqnarray}
For the second equivalence, we assume that
$\Var(V_i) = 1/\theta_{0,ii} >0$ and $\Var(V_j) = 1/\theta_{0,jj} > 0$.
Conditional (in-)dependencies can be conveniently encoded by an 
undirected graph, the conditional independence graph which we denote
by  $G= (V, E_0)$. The set of vertices is $V =\{1,\ldots ,p\}$ and 
the set of undirected edges $E_0 \subseteq V \times V$
is defined as follows:
\begin{eqnarray}
\nonumber
& &\mbox{there is an undirected edge between nodes $i$ and $j$}\\
\label{edge-set}
& \Longleftrightarrow & \theta_{0,ij} \not=
0  \; \; \Longleftrightarrow \; \; \beta_i^j \not= 0\ \mbox{and}\ \beta_j^i \not= 0.
\end{eqnarray}
Note that on the right hand side of the second equivalence, we could
replace the word "and" by "or".
For the second equivalence, we assume $\Var(V_i), \Var(V_j) >0$ following the
remark after~\eqref{eq::cond-beta}.

We now define the sparsity of the concentration matrix $\Theta_0$ or the
conditional independence graph. The definition is different than simply
counting the non-zero elements of $\Theta_0$, for which we have 
$\supp(\Theta_0) = p + 2|E_0|$. We consider instead the
number of elements which are sufficiently large. For each $i$, define
the number $s^{i}_{0,n}$ as 
the smallest integer such that the following holds:
\ben 
\label{eq::cond-regr1}
\sum_{j=1, j \not= i}^p \min\{ \theta_{0,ij}^2, \lambda^2
\theta_{0,ii} \} &  \leq & s^{i}_{0,n} \lambda^2 \theta_{0,ii},
\; \text{ where } \; \lambda = \sqrt{2 \log(p)/n},
\een
where  {\em essential sparsity} $s^{i}_{0,n}$ at row $i$ describes the number 
of  ``sufficiently large'' non-diagonal elements $\theta_{0,ij}$ relative to 
a given $(n, p)$ pair and $\theta_{0,ii}, i=1, \ldots, p$.
The value $S_{0,n}$ in (\ref{eq::cond-regr2}) is
summing {\em essential sparsity} across all rows of $\Theta_0$,
\begin{eqnarray}
\label{eq::cond-regr2}
S_{0,n} & := & \sum_{i=1}^p s^{i}_{0,n}.
\end{eqnarray}
Due to the expression of $\lambda$, 
the value of $S_{0,n}$ depends on $p$ and $n$. 
For example, if all non-zero non-diagonal elements $\theta_{0,ij}$ of 
the $i$th row are larger in absolute value than $\lambda \sqrt{\theta_{0,ii}}$, 
the value $s^{i}_{0,n}$ coincides with the node degree $s^i$.
However, if some (many) of the elements
$|\theta_{0,ij}|$ are non-zero but small, $s^{i}_{0,n}$ is (much)
smaller than its node degree $s^i$; As a consequence, if some (many) of 
$|\theta_{0,ij}|, \forall i, j, i \not= j$ are non-zero but small, 
the value of  $S_{0,n}$ is also (much) smaller than 
$2 |E_0|$, which is the ``classical'' sparsity for 
the matrix $(\Theta_0 - \diag(\Theta_0))$.
See Section~\ref{sec:append-pre} for more discussions.

\subsection{The estimation procedure}
\label{subsec.nodewisereg}
The estimation of $\Theta_0$ and $\Sigma_0 = \Theta_0^{-1}$ is pursued in two
stages. We first estimate the undirected graph with edge set $E_0$ as in
(\ref{edge-set}) and we then use the maximum likelihood estimator
based on the estimate $\hat{E}_n$, that is, the non-zero elements of
$\hat{\Theta}_n$ correspond to the estimated edges in $\hat{E}_n$. 
Inferring the edge set $E_0$ can be based on the following approach as
proposed and theoretically justified in~\cite{MB06}:
perform $p$ regressions using the Lasso to obtain $p$ vectors of 
regression coefficients $\hat{\beta}^1, \ldots ,\hat{\beta}^p$
where for each $i$, 
$\hat{\beta}^i = \{\hat{\beta}^i_j;\ j \in \{1,\ldots ,p\} \setminus i \}$; 
Then estimate the edge set by the ``OR'' rule, 
\begin{eqnarray}\label{eq::est-edge}
\mbox{estimate an edge between nodes $i$ and $j$} \Longleftrightarrow
\hat{\beta}_j^i \neq 0\ \mbox{or}\ \hat{\beta}_i^j \neq 0.
\end{eqnarray}
\noindent{\bf \large{Nodewise regressions for inferring the graph.}}
\label{sec:node-regres}
In the present work, we use the Lasso in combination with
thresholding~\citep{Zhou09th,Zhou10}. 
Consider the Lasso for each of the nodewise regressions
\begin{eqnarray}
\label{eq::lasso}
{\beta}_{\init}^i = \mbox{argmin}_{\beta^i} \sum_{r=1}^n (X_i^{(r)} - \sum_{j
    \neq i} \beta_j^i X_j^{(r)})^2 + \lambda_n \sum_{j \neq i}
  |\beta_j^i|\ \ \text{ for } i=1,\ldots ,p, 
\end{eqnarray}
where $\lambda_n > 0$ is the same regularization parameter for all 
regressions. Since the Lasso typically estimates too many components with
non-zero estimated regression coefficients, we use thresholding
to get rid of variables with small regression coefficients from solutions 
of \eqref{eq::lasso}:
\begin{eqnarray}
\label{eq::est-beta}
\hat{\beta}_j^i(\lambda_n, \tau) = {\beta}_{j,\init}^i(\lambda_n)
I(|{\beta}_{j,\init}^i(\lambda_n)| > \tau), 
\end{eqnarray}
where $\tau > 0$ is a thresholding parameter.
We obtain the corresponding estimated edge
set as defined by (\ref{eq::est-edge}) using the estimator in (\ref{eq::est-beta})
and we use the notation
\begin{eqnarray}
\label{eq::est-edgefinal}
\hat{E}_n(\lambda_n, \tau).
\end{eqnarray} 
We note that the estimator depends on two tuning parameters $\lambda_n$ 
and $\tau$.

The use of thresholding has clear benefits from a theoretical point of
view: the number of false positive selections may be much larger without
thresholding (when tuned for good prediction).
and a similar statement would hold when comparing the adaptive Lasso with the standard Lasso. 
We refer the interested reader to \cite{Zhou09th,Zhou10} and \cite{GBZ10}.  

\noindent{\bf \large{Maximum likelihood estimation based on graphs.}}
\label{subsec.mle} 
Given a conditional independence graph with edge set $E$, we estimate the
concentration matrix by maximum likelihood. Denote by $\hat{S}_n = n^{-1}
\sum_{r=1}^n X^{(r)} (X^{(r)})^T$ the sample covariance matrix (using that the
mean vector is zero) and by
\ben
\label{eq::Gamma-n}
\hat{\Gamma}_n =
 \diag(\hat{S}_n)^{-1/2} (\hat{S}_n) \diag(\hat{S}_n)^{-1/2}
\een
the sample correlation matrix. The estimator for the concentration matrix
in view of~\eqref{eq::rand-des} is:
\begin{eqnarray}
\label{eq::est-mle}
& &\hat{\Theta}_n(E) = \mbox{argmin}_{\Theta \in {\cal M}_{p,E}}
\left( \mathrm{tr}(\Theta \hat{\Gamma}_n) - \log|\Theta|\right), \text{ where } \nonumber\\
& &{\cal M}_{p,E} = \{\Theta \in \R^{p \times p};
\ \Theta \succ 0 \; \mbox{ and}\  \theta_{ij} = 0 \text{ for all }
(i,j) \not \in E, \; \text{ where } \; i \not=j\}
\end{eqnarray}
defines the constrained set for positive definite $\Theta$. If $n \ge q^*$ where $q^*$ is 
the maximal clique size of a minimal chordal cover of the graph with edge set $E$, 
the MLE exists and is unique, see, for example~\citet[Corollary 2.3]{Uhler11}. 
We note that our theory guarantees that $n \ge q^*$ holds
with high probability for $G = (V, E)$, where $E =  \hat{E}_n(\lambda_n,\tau))$, 
under Assumption (A1) to be introduced in the next section. 
The definition in (\ref{eq::est-mle})
is slightly different from the more usual estimator which uses the sample
covariance $\hat{S}_n$ rather than $\hat{\Gamma}_n$. Here, the sample
correlation matrix reflects the fact that we typically work with
standardized data where the variables have empirical variances equal to
one. The estimator in (\ref{eq::est-mle}) is constrained leading to
zero-values corresponding to ${E}^c =  \{(i,j): i, j = 1, \ldots, p, i \not= j, (i, j) \not\in {E} \}$.

If the edge set  $E$ is sparse having relatively 
few edges only, the estimator in (\ref{eq::est-mle}) 
is already sufficiently regularized by the constraints and
hence, no additional penalization is used at this stage. 
Our final estimator for the concentration matrix is the combination of
(\ref{eq::est-edgefinal}) and (\ref{eq::est-mle}):
\begin{eqnarray}\label{eq::est-final}
\hat{\Theta}_n = \hat{\Theta}_n(\hat{E}_n(\lambda_n,\tau)).
\end{eqnarray}

\noindent{\bf \large{Choosing the regularization parameters.}}
We propose to select the parameter $\lambda_n$ via cross-validation 
to minimize the squared test set error among all $p$ regressions:
\begin{eqnarray*}
\hat{\lambda}_n = \mbox{argmin}_{\lambda} \sum_{i=1}^p 
\left(\mbox{CV-score($\lambda$) of $i$th regression}\right),
\end{eqnarray*}
where CV-score($\lambda$) of $i$th regression is with respect to the squared
error prediction loss.
Sequentially proceeding, we then select $\tau$ by cross-validating the
multivariate Gaussian log-likelihood, from (\ref{eq::est-mle}). 
Regarding the type of cross-validation, we usually use
the 10-fold scheme. Due to the sequential nature of choosing the
regularization parameters, the number of candidate estimators is given by
the number of candidate values for $\lambda$ plus the number of candidate
value for $\tau$. In Section \ref{sec.numerical}, we describe the grids of
candidate values in more details.  
We note that for our theoretical results, we do not analyze the implications 
of our method using estimated $\hat{\lambda}_n$ and $\hat{\tau}$.

\section{Theoretical results}
\label{sec:theory}
In this section, we present in Theorem~\ref{thm:main} convergence rates 
for estimating the  precision and the covariance matrices with respect to the 
Frobenius norm; in addition, we show a risk consistency result for an
oracle risk to be defined in~\eqref{eq::future-risk}. 
Moreover, in Proposition~\ref{prop:bias}, we show that the model we select is 
sufficiently sparse while at the same time, the bias term we introduce via 
sparse approximation is sufficiently bounded.
These results illustrate the classical bias and variance tradeoff.
Our analysis is non-asymptotic in nature;  however, we first formulate our 
results from an asymptotic point of view for simplicity. 
To do so, we consider a triangular array of data generating random variables
\begin{eqnarray}
\label{data}
X^{(1)},\ldots ,X^{(n)}\ \mbox{i.i.d.} \sim {\cal N}_p(0,\Sigma_0),\ n
= 1,2,\ldots 
\end{eqnarray}
where $\Sigma_0=\Sigma_{0,n}$ and 
$p = p_n$ change with $n$. Let $\Theta_0: = \Sigma_0^{-1}$.
We make the following assumptions. 
\begin{enumerate}
\item[(A0)]
The size of the neighborhood for each node $i \in V$
is upper bounded by an integer $s < p$ and the sample size satisfies
for some constant $C$
$$n \geq C s \log (p/s).$$
\item[(A1)]
The dimension and number of sufficiently strong non-zero edges
$S_{0,n}$ as in~\eqref{eq::cond-regr2} satisfy:
dimension $p$ grows with $n$ following $p = o(e^{c n})$ for some constant
$0 < c <1$ and
\begin{eqnarray*}
S_{0,n} = o(n/\log\max(n,p))\ (n \to \infty).
\end{eqnarray*}
\item[(A2)]
The minimal and maximal 
eigenvalues of the true covariance matrix $\Sigma_0$ are bounded: for some constants 
$M_{\mathrm{upp}} \geq M_{\mathrm{low}} > 0$, we have
\begin{eqnarray*}
  \varphi_{\mathrm{min}}(\Sigma_0) \ge M_{\mathrm{low}} > 0 
\; \text{ and } \;
\varphi_{\mathrm{max}}(\Sigma_0) \leq  M_{\mathrm{upp}}  \leq \infty.
\end{eqnarray*}
Moreover, 
throughout our analysis, we assume the following.
There exists $v^2 > 0$ such that for all $i$, and $V_i$ as 
defined in~\eqref{eq::regressions}: \; $\Var(V_i) = 1/\theta_{0,ii} \geq v^2.$
\end{enumerate}
Before we proceed, we need some definitions.
Define for $\Theta \succ 0$
\begin{equation}
\label{eq::future-risk}
R(\Theta) = {\rm tr}(\Theta\Sigma_0) - \log |\Theta|,
\end{equation}
where minimizing~\eqref{eq::future-risk}
without constraints gives $\Theta_0$.
Given \eqref{eq::cond-regr2},~\eqref{eq::cond-regr1}, and $\Theta_0$, define
\ben
\label{eq::global-cons}
C^2_{\diag} := \min\{\max_{i=1,...p} \theta_{0,ii}^2,
\max_{i=1, \ldots, p} \left(s^{i}_{0,n}/{S}_{0,n}\right)\cdot
\fnorm{\diag(\Theta_0)}^2 \}.
\een
We now state the main results of this paper.
We defer the specification on various tuning parameters, namely, 
$\lambda_n, \tau$ to Section~\ref{sec:proof-main-outline},
where we also provide an outline for Theorem~\ref{thm:main}.
\begin{theorem}
\label{thm:main}
Consider data generating random variables as in (\ref{data}) 
and assume that (A0), (A1), and (A2) hold.
We assume $\Sigma_{0, ii} =1$ for all $i$.
Then, with probability at least $1 - d/p^{2}$, for some small constant $d >2$, 
we obtain under appropriately chosen $\lambda_n$ and $\tau$,
an edge set $\hat{E}_n$ as in (\ref{eq::est-edgefinal}), such that
\ben
\label{eq::thm-E-bounds}
& &|\hat{E}_n| \le 4 S_{0,n},\; \text{ where } \; 
|\hat{E}_n \setminus E_{0}| \le 2 S_{0,n};
\een
and for $\hat\Theta_n$ and $\hat\Sigma_n =  (\hat\Theta_n)^{-1}$ as 
defined in~\eqref{eq::est-final}
the following holds, 
\bens
\nonumber
\twonorm{\hat{\Theta}_n - \Theta_{0}} \leq 
\|\hat{\Theta}_n - \Theta_{0}\|_F & = & 
O_P\left(\sqrt{{S_{0,n} \log\max(n,p)}/{n}}\right), \\
\nonumber
\twonorm{\hat{\Sigma}_n - \Sigma_0} \leq
\|\hat{\Sigma}_n - \Sigma_0\|_F & = & 
O_P\left(\sqrt{{S_{0,n} \log\max(n,p)}/{n}}\right), \\
\label{eq::KL}
R(\hat \Theta_n) - R(\Theta_0) & =  & 
O_P\left( S_{0,n} \log \max(n,p)/{n} \right)
\eens
where the constants hidden in the $O_P()$ notation depend on $\tau$, $M_{\mathrm{low}}, M_{\mathrm{upp}}$, $C_{\diag}$ as in~\eqref{eq::global-cons}, and constants concerning sparse and restrictive eigenvalues of $\Sigma_0$ 
(cf. Section~\ref{sec:proof-main-outline} and~\ref{sec:append-p-regres}).
\end{theorem}

We note that convergence rates for the estimated covariance matrix 
and for predictive risk depend on the rate in Frobenius norm of 
the estimated inverse covariance matrix.
The predictive risk can be interpreted as follows.
Let $X \sim \mathcal{N}(0, \Sigma_0)$ with $f_{\Sigma_0}$ denoting its density.
Let $f_{\hat\Sigma_n}$ be the density for $\mathcal{N}(0,\hat\Sigma_n)$ and
$D_\text{KL}(\Sigma_0 \| \hat\Sigma_n)$ denotes the 
Kullback Leibler (KL) divergence from $\mathcal{N}(0, \Sigma_0)$ 
to $\mathcal{N}(0,  \hat\Sigma_n)$.
Now, we have for $\Sigma, \hat\Sigma_n \succ 0$,
\ben 
\label{eq::kld}
R(\hat\Theta_n) - R(\Theta_0) :=  
2 \expect{0}{\log f_{\Sigma_0}(X) - \log f_{\hat\Sigma_n}(X)} :=
2 D_\text{KL}(\Sigma_0 \| \hat\Sigma_n) \geq 0.\een 
Actual conditions and non-asymptotic results that are involved in the
G{\bf e}lat{\bf o}  estimation appear in Sections~\ref{sec:append-p-regres},
~\ref{sec:append-proof-bias}, and~\ref{sec:append-frob-missing}
respectively.

\begin{remark}
Implicitly in (A1), we have specified a lower bound on 
the sample size to be $n = \Omega\left(S_{0,n} \log \max(n,p)\right)$.
For the interesting case of $p > n$, a sample size of
\ben
n = \Omega\left(\max(S_{0,n} \log p, s \log( p/s)) \right)
\een
is sufficient in order to achieve the rates in Theorem~\ref{thm:main}.
As to be shown in our analysis, the lower bound on $n$ is slightly
different for each Frobenius norm bound to hold from a non-asymptotic 
point of view (cf. Theorem~\ref{thm::frob-missing} 
and~\ref{thm::frob-missing-cov}).
\end{remark}
Theorem \ref{thm:main} can be interpreted as follows.
First, the cardinality of the estimated edge set
exceeds $S_{0,n}$  at most by a factor 4, where $S_{0,n}$ as 
in~\eqref{eq::cond-regr2} is the number of 
sufficiently strong edges in the model, while the number of false 
positives is bounded by $2 S_{0,n}$. 
Note that the factors $4$ and $2$ can be replaced by some other constants, 
while achieving the same bounds on the 
rates of convergence (cf. Section~\ref{sec:proof-main}).
We emphasize that we achieve these two goals by
 sparse model selection, where only important 
edges are selected even though there are many more non-zero edges 
in $E_0$, under conditions that are much weaker than (A2).
More precisely, (A2) can be replaced by conditions on sparse and 
restrictive eigenvalues (RE) of $\Sigma_0$. 
Moreover, the bounded neighborhood constraint (A0) is required 
only for regression analysis (cf. Theorem~\ref{thm:RE-oracle-all}) 
and for bounding the bias due to sparse approximation as in 
Proposition~\ref{prop:bias}.
This is shown in Sections~\ref{sec:append-p-regres}  and \ref{sec:append-proof-bias}.
Analysis follows from~\cite{Zhou09th,Zhou10} with earlier 
references to~\cite{CT07,MY09,BRT09} for estimating sparse regression
coefficients. 

We note that  the conditions that we use are indeed similar to those in~\cite{RBLZ08}, 
with (A1) being much more relaxed when $S_{0,n} \ll |E_0|$.
The convergence rate with respect to the Frobenius norm should be
compared to the rate
$O_P(\sqrt{{|E_{0}| \log\max(n,p)}/{n}})$ in case $\diag(\Sigma_{0})$ is known, 
which is the rate in~\cite{RBLZ08} for the GLasso and for SCAD~\citep{LaFa09}. 
In the scenario where $|E_{0}| \gg S_{0,n}$, 
i.e. there are many weak edges,  the rate in Theorem \ref{thm:main} 
is better than the one established for GLasso~\citep{RBLZ08} or for 
the SCAD-type estimator~\citep{LaFa09}; hence we require a smaller 
sample size in order to yield an accurate estimate of
$\Theta_0$.


\begin{remark}
\label{remark:general-Fnorms}
For the general case where $\Sigma_{0, ii}, i=1, \ldots, p$ are not assumed 
to be known,
we could achieve essentially the same rate as stated in
Theorem~\ref{thm:main} for $\shtwonorm{\hat{\Theta}_n - \Theta_{0}}$ and 
$\shtwonorm{\hat{\Sigma}_n - \Sigma_0}$ under $(A_0), (A_1)$ and $(A_2)$ 
following analysis in the present work (cf. Theorem~\ref{thm::frob-missing-omega})
and that in~\citet[Theorem 2]{RBLZ08}. 
Presenting full details for such results are beyond the scope of the
current paper.
We do provide the key technical lemma which is essential for showing such
bounds based on estimating the inverse of the correlation matrix in 
Theorem~\ref{thm::frob-missing-omega}; see also Remark~\ref{remark:general-opnorm}
which immediately follows.

In this case, for the Frobenius norm and the risk to converge to zero, 
a too large value of $p$ is not allowed.
Indeed, for the Frobenius norm and the risk to converge, 
$(A1)$ is to be replaced by:
$$(A3) \; \; \; \; p \asymp n^c \text{ for some constant } 0< c < 1 \text{ and }
\; p+ S_{0,n} = o(n/\log\max(n,p)) \; \text{ as } \; n \to \infty.$$
In this case, we have 
\bens
\nonumber
\|\hat{\Theta}_n - \Theta_{0}\|_F & = & 
O_P\left(\sqrt{{(p+ S_{0,n}) \log\max(n,p)}/{n}}\right), \\
\nonumber
\|\hat{\Sigma}_n - \Sigma_0\|_F & = & 
O_P\left(\sqrt{{(p+S_{0,n}) \log\max(n,p)}/{n}}\right), \\
\label{eq::KL}
R(\hat \Theta_n) - R(\Theta_0) & =  & 
O_P\left((p+ S_{0,n}) \log \max(n,p)/{n} \right).
\eens
Moreover, in the refitting stage, we could achieve these rates with the 
maximum likelihood estimator based on the sample covariance 
matrix $\hat{S}_n$ as defined in~\eqref{eq::est-mle-0}:
\begin{eqnarray}
\label{eq::est-mle-0}
& &\hat{\Theta}_n(E) = \mbox{argmin}_{\Theta \in {\cal M}_{p,E}}
\left( \mathrm{tr}(\Theta \hat{S}_n) - \log|\Theta|\right), \text{ where } \nonumber\\
& &{\cal M}_{p,E} = \{\Theta \in \R^{p \times p};
\ \Theta \succ 0 \; \mbox{ and}\  \theta_{ij} = 0 \text{ for all }
(i,j) \not \in E, \; \text{ where } \; i \not=j\}
\end{eqnarray}

A real high-dimensional scenario where $p \gg n$ is excluded in order to achieve
Frobenius norm consistency.
 This restriction comes from the nature of the 
Frobenius norm and when considering e.g. the operator norm, such
restrictions can indeed be relaxed as stated above.
\end{remark}
It is also of interest to understand the bias of the estimator 
caused by using the estimated edge set $\hat{E}_n$ instead of the true 
edge set $E_{0}$.
This is the content of Proposition~\ref{prop:bias}.
For a given $\hat{E}_n$, denote by 
$$\tilde{\Theta}_0 = \mathrm{diag}(\Theta_0) + (\Theta_0)_{\hat{E}_n} =
\mathrm{diag}(\Theta_0) + \Theta_{0, \hat{E}_n \cap E_0},$$
where the second equality holds since $\Theta_{0, E_0^c} = 0$.
Note that the quantity $\tilde{\Theta}_0$ is identical to $\Theta_0$ on 
$\hat{E}_n$ and on the diagonal, and it equals zero on 
$\hat{E}_n^c = \{(i,j): i, j = 1, \ldots, p, i \not= j, (i, j) \not\in \hat{E}_n \}$.
Hence, the quantity $\Theta_{0, \drop} := \tilde{\Theta}_0 - \Theta_0$
measures the bias caused by a potentially wrong edge set $\hat{E}_n$; 
note that $\Tilde{\Theta}_0 = \Theta_0$ if $\hat{E}_n = E_0$. 
\begin{proposition}
\label{prop:bias}
Consider data generating random variables as in expression (\ref{data}). 
Assume that (A0), (A1), and (A2) hold.
Then  we have for choices on $\lambda_n, \tau$ as in
Theorem~\ref{thm:main} and $\hat{E}_n$ in (\ref{eq::est-edgefinal}),
\bens
\fnorm{\Theta_{0, \drop}} := 
\|\tilde{\Theta}_0 - \Theta_0\|_F = O_P\left(\sqrt{S_{0,n} \log\max(n,p)/n}\right).
\eens
\end{proposition}
We note that we achieve essentially the same rate for
$\|(\tilde{\Theta}_0)^{-1} - \Sigma_0\|_F$; see Remark~\ref{eq::Sigma-drop}.
We give an account on how results in Proposition~\ref{prop:bias} are 
obtained in Section~\ref{sec:proof-main-outline}, with its non-asymptotic 
statement appearing in  Corollary~\ref{cor:drop-norm}.

\subsection{Discussions and connections to previous work}
It is worth mentioning that consistency in terms of operator and
Frobenius norms does not depend too strongly on the 
property to recover the true underlying edge set $E_0$ in the refitting 
stage. Regarding the latter, suppose we obtain with high probability the
screening property 
\begin{eqnarray}
\label{eq::screen}
E_0 \subseteq E,
\end{eqnarray}
when assuming that all non-zero regression coefficients $|\beta^i_j|$
are sufficiently large ($E$ might be an estimate and hence random).
Although we do not intend to make precise 
the exact conditions and choices of tuning parameters in regression and 
thresholding in order to achieve~\eqref{eq::screen}, 
we state Theorem~\ref{thm::frob},
in case~\eqref{eq::screen}  holds with the following condition:
the number of false positives is bounded as
$\size{{E} \setminus E_0} = O(S).$
\begin{theorem}
\label{thm::frob}
Consider data generating random variables as in expression (\ref{data}) 
and assume that (A1) and (A2) hold, where we replace
$S_{0,n}$ with $S := \size{E_0} = \sum_{i=1}^p s^i$. We assume $\Sigma_{0,
  ii} = 1$ for all $i$. 
Suppose on some event $\event$, such that $\prob{\event} \geq 1- d/p^2$ 
for a small constant $d$, we obtain an edge set ${E}$ such that 
$E_0 \subseteq E$ and $|{E} \setminus E_0| = O(S)$.
Let $\hat\Theta_n(E)$ be the minimizer as defined 
in~\eqref{eq::est-mle}. Then, we have
$\|\hat \Theta_n(E) -\Theta_0\|_F = O_P\left(\sqrt{S\log\max (n,p)/{n}}\right)$.
\end{theorem}
It is clear that this bound corresponds to exactly that of~\cite{RBLZ08}
for the GLasso estimation under appropriate choice of the penalty parameter
for a general $\Sigma \succ 0$ with $\Sigma_{ii} =1$ for all $i$ 
(cf. Remark~\ref{remark:general-Fnorms}).
We omit the proof as it is more or less a modified version of 
Theorem~\ref{thm::frob-missing}, which proves the stronger bounds 
as stated in Theorem~\ref{thm:main}.
We note that the maximum node-degree bound  in (A0) is not needed for 
Theorem~\ref{thm::frob}.

We now make some connections to previous work.
First, we note that to obtain with high probability the exact edge recovery,
${E} = E_0$, we need again sufficiently large non-zero edge weights and 
some restricted eigenvalue (RE) conditions on the covariance matrix as 
defined in Section~\ref{sec:append-pre} even for the multi-stage procedure.
An earlier example is shown in~\cite{ZGB09}, where the second stage 
estimator $\hat \beta$ corresponding to~\eqref{eq::est-beta}
is obtained with nodewise regressions using adaptive Lasso~\citep{Zou06} rather 
than thresholding as in the present work in order to recover the edge set 
${E}_0$ with high probability under an assumption which is stronger than (A0).
Clearly, given an  accurate $\hat{E}_n$, under (A1) and (A2) one can 
then apply Theorem~\ref{thm::frob} to accurately estimate 
$\hat{\Theta}_n$.
On the other hand, it is known that GLasso necessarily needs more 
restrictive conditions on $\Sigma_0$ than the nodewise regression approach 
with the Lasso, as discussed in~\cite{Mei08} and~\cite{RWRY08} in
order to achieve exact edge recovery.

Furthermore, we believe it is straightforward to show that Gelato works under the
RE conditions on $\Sigma_0$ and with a smaller sample size than the analogue without the
thresholding operation in order to achieve 
{\em nearly exact recovery} of the support in the sense that 
$E_0 \subseteq \hat{E}_n$ and $\max_{i} |\hat{E}_{n,i} \setminus E_{0,i}|$ 
is small, that is, the number of extra estimated 
edges at each node $i$ is  bounded by a small constant.
This is shown essentially in~\citet[Theorem1.1]{Zhou09th} for a 
single regression. Given such properties of $\hat{E}_n$, we can again apply 
Theorem~\ref{thm::frob} to obtain $\hat{\Theta}_n$ under (A1) and (A2).
Therefore, Gelato requires relatively weak assumptions on $\Sigma_0$ in
order to achieve the best sparsity and bias tradeoff as illustrated in 
Theorem~\ref{thm:main} and Proposition~\ref{prop:bias} 
when many signals are weak, and Theorem~\ref{thm::frob} when all signals in $E_0$ are strong.

\subsection{An outline for Theorem~\ref{thm:main}}
\label{sec:proof-main-outline}
Let $s_0 = \max_{i=1, \ldots, p} s^i_{0, n}$.
We note that although sparse eigenvalues 
$\rho_{\max}(s), \rho_{\max}(3s_0)$ and restricted eigenvalue for 
$\Sigma_0$ (cf. Section~\ref{sec:append-pre})
are parameters that are unknown, we only need them to appear in 
the lower bounds for $d_0$, $D_4$, and hence also that for 
$\lambda_n$ and $t_0$ that appear below.
We simplify our notation in this section to keep it consistent with
our theoretical non-asymptotic analysis to appear
toward the end of this paper.

\noindent{\bf Regression.}
We choose for some $c_0 \geq 4 \sqrt{2}$, $0< \theta < 1$, and 
$\lambda =\sqrt{\log (p) /n}$,
$$\lambda_n = d_0 \lambda, \; \text{ where }\; 
d_0 \geq c_0 (1+\theta)^2 \sqrt{\rho_{\max}(s) \rho_{\max}(3s_0)}.$$
Let $\beta^i_{\init}, i = 1, \ldots, p$ be the optimal solutions 
to~\eqref{eq::lasso} with $\lambda_n$ as chosen above.
We first prove an oracle result on nodewise regressions in Theorem~\ref{thm:RE-oracle-all}.

\noindent{\bf Thresholding.}
We choose for some constants $D_1, D_4$ to be defined in 
Theorem~\ref{thm:RE-oracle-all},
$$t_0 = f_0 \lambda :=  D_4 d_0 \lambda \; \; \; \text{ where } D_4 \geq D_1$$
where $D_1$ depends on restrictive eigenvalue of $\Sigma_0$;
Apply~\eqref{eq::est-beta} with $\tau = t_0$ and 
$\beta^i_{\init}, i = 1, \ldots, p$ for thresholding
our initial regression coefficients.
Let
\bens
\label{eq::drop-i-set}
\drop^{i} = \{j: j \not= i, \; \abs{\beta^{i}_{j, \init}} < t_0 = f_0 \lambda\},
\eens
where bounds on $\drop^i, i=1, \ldots, p$ are given in 
Lemma~\ref{lemma:threshold-RE}. In view of \eqref{eq::est-edge},
we let
\ben
\label{eq::all-drop}
\drop =\{(i,j): i\not=j: (i,j) \in \drop^i \cap \drop^j\}.
\een

\noindent{\bf Selecting edge set $E$.}
Recall for a pair $(i, j)$  we take the {\it OR rule} as in~\eqref{eq::est-edge} 
to decide if it is to be included  in the edge set $E$: for $\drop$ as defined
in~\eqref{eq::all-drop}, define
\ben
\label{eq::non-drop}
E := \{(i,j): i, j = 1, \ldots, p, i \not= j, (i, j) \not\in \drop\}.
\een
to be the subset of pairs of  non-identical vertices of $G$ which
do not appear in $\drop$; 
Let 
\ben
\label{eq::new-posi} 
\tilde{\Theta}_0 = \diag(\Theta_0) + \Theta_{0, E_0 \cap E}
\een
for $E$ as in~\eqref{eq::non-drop},
which is identical to $\Theta_{0}$ on all diagonal entries and 
entries indexed by ${E_0 \cap E}$, with the rest being set to zero.
As shown in the proof of Corollary \ref{cor:drop-norm}, by thresholding, 
we have identified a {\it sparse subset} of edges $E$
of size at most  $4S_{0,n}$, such that
the corresponding bias 
$\fnorm{\Theta_{0,\drop}} :=   \|\tilde{\Theta}_0  -\Theta_0\|_F$ is 
relatively small, i.e., as bounded in Proposition~\ref{prop:bias}.

\noindent{\bf Refitting.}
In view of Proposition~\ref{prop:bias}, we aim to recover $\tilde{\Theta}_0$ 
given a sparse subset $E$; toward this goal, we use~\eqref{eq::est-mle} 
to obtain the final estimator $\hat{\Theta}_n$ and $\hat{\Sigma}_n =
(\hat{\Theta}_n)^{-1}$.
We give a more detailed account of this procedure in 
Section~\ref{sec:append-frob-missing}, with a focus on 
elaborating the bias and variance tradeoff.
We show the rate of convergence in Frobenius norm 
for the estimated $\hat{\Theta}_n$ and $\hat{\Sigma}_n$ in 
Theorem~\ref{thm::frob-missing-omega}, \ref{thm::frob-missing}
and~\ref{thm::frob-missing-cov}, 
and the bound for Kullback Leibler  divergence in 
Theorem~\ref{thm::frob-missing-risk} respectively.

\subsection{Discussion on covariance estimation based on maximum likelihood}
\label{sec:dis-cov}
The maximum likelihood estimate minimizes over all $\Theta \succ 0$,
\ben
\label{eq::sample-risk}
\hat{R}_n(\Theta) =
 {\rm tr}(\Theta \hat{S}_n) - \log|\Theta|
\een
where $\hat{S}_n$ is the sample covariance 
matrix. Minimizing $\hat{R}_n(\Theta)$ without constraints gives
$\hat\Sigma_n = \hat{S}_n$. 
We now would like to minimize~\eqref{eq::sample-risk}
under the constraints that some pre-defined subset 
$\drop$ of edges are set to zero. 
Then the follow relationships hold regarding  
$\hat{\Theta}_n(E)$ defined in~\eqref{eq::est-mle-0} and its inverse 
$\hat{\Sigma}_n$, and $\hat{S}_n$: for $E$ 
as defined in~\eqref{eq::non-drop},
\bens
\hat{\Theta}_{n,ij} & = & 0, \; \forall (i, j) \in \drop \; \text{ and } \\
\hat{\Sigma}_{n,ij} & = & \hat{S}_{n, ij}, \; \forall (i, j) \in E \cup \{(i,i), i = 1, \ldots, p\}.
\eens
Hence the entries in the covariance matrix $\hat{\Sigma}_n$ 
for the chosen set of edges in $E$ and the diagonal entries are set to 
their corresponding values in $\hat{S}_n$.
Indeed, we can derive the  above relationships via the
Lagrange form, where we add Lagrange constants $\gamma_{jk}$
for edges in $\drop$,
\ben
\label{eq::lang}
\ell_C(\Theta) = \log |\Theta| - {\rm tr}(\hat{S}_n \Theta) - \sum_{(j,k) \in \drop}
\gamma_{jk} \theta_{jk}.
\een
Now the gradient equation of~\eqref{eq::lang} is:
$$\Theta^{-1} - \hat{S}_n - \Gamma = 0,$$
where $\Gamma$ is a matrix of Lagrange parameters
such that $\gamma_{jk} \not= 0$ for all $(j, k) \in \drop$ and 
$\gamma_{jk} = 0$ otherwise.

Similarly, the follow relationships hold regarding  
$\hat{\Theta}_n(E)$ defined in~\eqref{eq::est-mle} in case $\Sigma_{0,ii} =1$
for all $i$,
where $\hat{S}_n$ is replaced with $\hat{\Gamma}_n$, and
its inverse $\hat{\Sigma}_n$, and $\hat{\Gamma}_n$: for $E$ 
as defined in~\eqref{eq::non-drop},
\bens
\hat{\Theta}_{n,ij} & = & 0, \; \forall (i, j) \in \drop \; \text{ and } \\
\hat{\Sigma}_{n,ij} & = & \hat{\Gamma}_{n, ij} = \hat{S}_{n, ij}/\hat{\sigma}_i \hat{\sigma}_j, 
\; \forall (i, j) \in E \; \text{ and }  \\
 \hat\Sigma_{n,ii} & = & 1, \; \forall i = 1, \ldots, p.
\eens
Finally,  we state Theorem~\ref{thm::frob-missing-omega},
which yields a general bound on estimating the inverse of the correlation matrix, 
when $\Sigma_{0,11}, \ldots, \Sigma_{0,pp}$ take
arbitrary unknown values in $\R^{+} = (0, \infty)$.
The corresponding estimator is based on estimating the inverse of 
the correlation matrix, which we denote by $\Omega_0$.
We use the following notations.
Let $\Psi_0 = (\rho_{0, ij})$ be the true correlation matrix and let $\Omega_0 = \Psi_0^{-1}$.
Let $W = \diag(\Sigma_0)^{1/2}$. Let us denote the diagonal entries of $W$ with
$\sigma_{1}, \ldots, \sigma_{p}$ where $\sigma_{i} := \Sigma_{0,ii}^{1/2}$ for all $i$.
Then the following holds:
\bens
\Sigma_0 & = & W \Psi_0 W \; \; \text{ and } 
\Theta_0 \; = \; W^{-1} \Omega_0 W^{-1} 
\eens
Given sample covariance matrix $\hat{S}_n$, we construct sample correlation matrix
$\hat{\Gamma}_n$ as follows. Let  $\hat{W} = \diag(\hat{S}_n)^{1/2}$ and
\ben
\label{eq::Gamma-appen}
\hat{\Gamma}_n =
 \hat{W}^{-1} (\hat{S}_n) \hat{W}^{-1},\; \text{ where } \; 
\hat{\Gamma}_{n,ij} = \frac{\hat{S}_{n,ij}}{\hat{\sigma}_i \hat{\sigma}_j} = 
\frac{\ip{X_i, X_j}}{\twonorm{X_i}\twonorm{X_j}}
\een
where $\hat{\sigma}_i^2 := \hat{S}_{n, ii}$.
Thus $\hat{\Gamma}_n$ is a matrix with diagonal entries being all $1$s
and non-diagonal entries being the sample correlation coefficients, which we 
denote by $\hat{\rho}_{ij}$.

The maximum likelihood estimate for $\Omega_0 = \Psi_0^{-1}$ minimizes over all $\Omega \succ 0$,
\ben
\label{eq::sample-risk-omega}
\hat{R}_n(\Omega) = {\rm tr}(\Omega \hat{\Gamma}_n) - \log|\Omega|
\een
To facilitate technical discussions, we need to introduce some more notation.
Let  $\PDcone$ denote the set of $p \times p$ symmetric positive definite 
matrices: 
$$\PDcone = \{\Theta \in \R^{p \times p} | \Theta \succ 0\}.$$ 
Let us define a subspace $\SubE$ corresponding to an edge set 
$E \subset \{(i,j): i, j = 1, \ldots, p, i \not=j \}$:
\ben
\label{eq::magic-set}
\SubE :=  \{\Theta \in \R^{p \times p} \;  | \; \theta_{ij} = 0 \; \forall \;  i \not= j
\; \text{ s.t.} \; (i, j) \not\in E \} \; 
\text{ and denote } \; \Set_n & = &  \PDcone \cap \SubE.
\een
Minimizing $\hat{R}_n(\Theta)$ without constraints gives
$\hat\Psi_n = \hat{\Gamma}_n$.  Subject to the constraints that 
$\Omega \in \Set_n$ as defined in~\eqref{eq::magic-set}, we write
the maximum likelihood estimate for $\Omega_0$:
\beq
\label{eq::refit-omega}
\hat{\Omega}_n(E)  := 
\arg \min_{\Omega\in \Set_n}  \hat{R}_n(\Omega) = 
\arg \min_{\Omega \in \PDcone \cap \SubE} 
\big\{ {\rm tr}(\Omega \hat{\Gamma}_n) - \log |\Omega| \big\}
\eeq
which yields the following relationships regarding  
$\hat{\Omega}_n(E)$, its inverse $\hat{\Psi}_n = (\hat{\Omega}_n(E))^{-1}$, 
and $\hat{\Gamma}_n$.  For $E$ as defined in~\eqref{eq::non-drop},
\bens
\hat{\Omega}_{n,ij} & = & 0, \; \forall (i, j) \in \drop \\
\hat{\Psi}_{n,ij} & = & \hat{\Gamma}_{n, ij} := \hat{\rho}_{ij} \;\;\; \; \forall (i, j) \in E \; \\
 \text{ and } \;\;
 \hat\Psi_{n,ii} & = & 1 \; \; \;  \forall i = 1, \ldots, p.
\eens
Given $\hat{\Omega}_n(E)$ and its inverse $\hat{\Psi}_n = (\hat{\Omega}_n(E))^{-1}$, we obtain
\bens
\hat\Sigma_n =  \hat{W} \hat\Psi_n \hat{W} \; \; \text{ and } \;\;
\hat\Theta_n = \hat{W}^{-1} \hat\Omega_n \hat{W}^{-1}
\eens
and therefore the following holds:
for $E$ as defined in~\eqref{eq::non-drop},
\bens
\hat{\Theta}_{n,ij} & = & 0, \; \forall (i, j) \in \drop \\
\hat{\Sigma}_{n,ij} & = & \hat{\sigma}_i \hat{\sigma}_j \hat{\Psi}_{n,ij}  = 
\hat{\sigma}_i \hat{\sigma}_j \hat{\Gamma}_{n,ij}  =  \hat{S}_{n, ij} \;\;\; \; \forall (i, j) \in E \; \\
 \text{ and } \;\;
 \hat\Psi_{n,ii} & = & \hat{\sigma}_i^2 = \hat{S}_{n,ii} \; \; \;  \forall i = 1, \ldots, p.
\eens
The proof of Theorem~\ref{thm::frob-missing-omega} appears 
in Section \ref{sec:append-MLE-refit-general}.
\begin{theorem}
\label{thm::frob-missing-omega}
Consider data generating random variables as in expression (\ref{data}) 
and assume that $(A1)$ and $(A2)$ hold.
Let $\sigma_{\max}^2 := \max_{i} \Sigma_{0, ii} < \infty$
and $\sigma_{\min}^2 := \min_{i} \Sigma_{0, ii} >0$.
Let $\event$ be some event such that $\prob{\event} \geq 1- d/p^2$ for 
a small constant $d$. Let $S_{0,n}$ be as defined in~\eqref{eq::cond-regr2}.
Suppose on event $\event$:
\begin{enumerate}
\item
We obtain an edge set $E$ 
such that its size $|E| = \lin(S_{0,n})$ is a linear function in $S_{0,n}$.
\item
And for $\tilde{\Theta}_0$ as in \eqref{eq::new-posi} and
for some constant $C_{\bias}$ to be specified in~\eqref{eq::bias-local}, we have
\beq
\label{eq::bias-thm-omega}
\fnorm{\Theta_{0,\drop}}  :=  \fnorm{\tilde{\Theta}_0  -\Theta_0}
\leq C_{\bias} \sqrt{2 S_{0,n}  \log (p) /n}.
\eeq
\end{enumerate}
Let $\hat\Omega_n(E)$  be as defined in~\eqref{eq::refit-omega}
Suppose the sample size satisfies for $C_3 \geq 4 \sqrt{5/3}$,
\ben
\label{eq::sample-frob-omega}
n > \frac{144 \sigma_{\max}^4} {M_{\mathrm{low}}^2}
\left(4C_3 +\frac{13 M_{\mathrm{upp}} }{12 \sigma_{\min}^2} \right)^2
\max\left\{2 |E|\log \max(n,p), \;
 C^2_{\bias} 2 S_{0,n}  \log p \right\}.
\een
Then with probability $\geq 1- (d+1)/p^2$, we have for
$M = (9 \sigma_{\max}^4/(2 \underline{k}^2)) \cdot \left(4C_3 +
13 M_{\mathrm{upp}}/(12 \sigma_{\min}^2) \right)$
\begin{equation}
\label{eq::frob-con-corr}
\fnorm{\hat \Omega_n(E) -\Omega_0} \leq 
(M + 1) \max\left\{\sqrt{{2|E|\log\max(n,p)}/{n}}, \; 
C_{\bias} \sqrt{{2 S_{0,n}  \log (p)}/{n}} \right\}.
\end{equation}
\end{theorem}

\begin{remark}
\label{remark:general-opnorm}
We note that the constants in Theorem~\ref{thm::frob-missing-omega} are 
by no means the best possible. 
From~\eqref{eq::frob-con-corr}, we can derive bounds on
$\shtwonorm{\hat \Theta_n(E) -\Theta_0}$ and 
 $\shtwonorm{\hat \Sigma_n(E) -\Sigma_0}$ to be in the same order as
 in~\eqref{eq::frob-con-corr} following the analysis in~\citet[Theorem
 2]{RBLZ08}.
The corresponding bounds on the Frobenius norms on covariance estimation 
would be in the order of 
$O_P\left(\sqrt{\frac{p+S_0}{n}}\right)$ as stated in 
Remark~\ref{remark:general-Fnorms}.
\end{remark}

\section{Numerical results}\label{sec.numerical}

We consider the empirical performance for simulated and
real data. We compare our estimation method with the GLasso, the Space
method and a simplified Gelato estimator without thresholding for inferring
the conditional independence graph. The comparison with the latter should
yield some evidence about the role of thresholding in Gelato.  
The GLasso is defined as:
\begin{equation*}
\label{logliki}
\hat{\Theta}_{\mathrm{GLasso}} = \underset{\Theta\ \succ 0}{\operatorname{argmin}}
(\mbox{tr}(\hat{\Gamma}_n \Theta) - \log |\Theta| + \rho  \sum_{i<j}|\theta_{ij}|)
\end{equation*}
where $\hat{\Gamma}_n$ is the empirical
correlation matrix and the minimization is over positive definite 
matrices. Sparse partial correlation estimation (Space) is an approach for
selecting non-zero partial correlations in the high-dimensional
framework. The method assumes an overall sparsity of the partial
correlation matrix and employs sparse regression techniques for model
fitting. For details see
\cite{PWZZ09}. We use Space with weights all equal to one, which refers to
the method type \texttt{space} in \cite{PWZZ09}. For the Space method,
estimation of $\Theta_0$ is done via maximum likelihood as in
(\ref{eq::est-mle}) based on the edge set $\hat{E}^{(Space)}_n$ from the
estimated sparse partial correlation matrix. For computation of the three
different methods, we used the R-packages \texttt{glmnet}~\citep{FHTglmnet10},
\texttt{glasso}~\citep{FHT07} and \texttt{space}~\citep{PWZZ09}.  

\subsection{Simulation study} \label{simu}

In our simulation study, we look at three different models.

\begin{itemize}

\item An AR(1)-Block model. 
In this model the covariance matrix is block-diagonal with equal-sized
AR(1)-blocks of the form $\Sigma_{Block}=\{0.9^{|i-j|}\}_{i,j}$. 
\item The random concentration matrix model considered
  in \cite{RBLZ08}. In this model, the concentration matrix is
  $\Theta=B+\delta I$ where each off-diagonal entry in $B$ is
  generated independently and equal to 0 or 0.5 with probability $1- \pi$
  or $\pi$, respectively. 
  All diagonal
  entries of $B$ are zero, and $\delta$ is chosen such that the condition
  number of $\Theta$ is $p$. 
  
\item The exponential decay model considered in \cite{FFW09}. 
In this model we consider a case where no element of the concentration
matrix is exactly zero. The elements of $\Theta_0$ are
given by $\theta_{0,ij}=\exp(-2|i-j|)$ equals essentially zero when
the difference $|i-j|$ is large.
\end{itemize}

We compare the three estimators for each model with $p=300$ and
$n=40,80,320$. For each model we sample data $X^{(1)},\dots ,X^{(n)}$ i.i.d. $\sim
\mathcal{N}(0,\Sigma_0)$. We use two different performance
measures. The Frobenius norm of the estimation error
$\|\hat{\Sigma}_n-\Sigma_0\|_F$ and $\|\hat{\Theta}_n-\Theta_0\|_F$, and
the Kullback Leibler divergence between 
${\cal N}(0,\Sigma_0)$ and ${\cal N}(0,\hat{\Sigma}_n)$ as defined in
\eqref{eq::kld}. 

For the three estimation methods we have
various tuning parameters, namely $\lambda$, $\tau$ (for Gelato), $\rho$
(for GLasso) and $\eta$ (for Space). We
denote the regularization parameter of the Space technique by $\eta$ in contrary
to \cite{PWZZ09}, in order to distinguish it from the other parameters. Due
to the computational complexity we specify the two parameters of our Gelato 
method sequentially. 
That is, we derive the optimal value of the penalty parameter $\lambda$ by
10-fold cross-validation with respect to the test set squared error for all the
nodewise regressions. 
After fixing $\lambda=\lambda_{CV}$ we obtain the optimal threshold $\tau$ 
again by 10-fold cross-validation but with respect to
the negative Gaussian log-likelihood
($\mbox{tr}(\hat{\Theta}\hat{S}^{out})- \log |\hat{\Theta}|$, where
$\hat{S}^{out}$ is the empirical covariance of the hold-out data). We could
use individual tuning parameters for each of the regressions. However, this
turned out to be sub-optimal in some simulation scenarios (and never really
better than using a single tuning parameter $\lambda$, at least in the
scenarios we considered). For the penalty parameter $\rho$ of the GLasso
estimator and the parameter $\eta$ of the Space method we also use a
10-fold cross-validation with respect to the negative Gaussian log-likelihood.
The grids of candidate values are given as follows:
\begin{eqnarray*}
&&\lambda_k= A_k\sqrt{\frac{\log{p}}{n}} \quad k=1, \dots ,10 \quad
\mbox{with}\quad \tau_k= 0.75 \cdot B_k\sqrt{\frac{\log{p}}{n}} \\
&&\rho_k= C_k\sqrt{\frac{\log{p}}{n}} \quad k=1, \dots ,10\\
&& \eta_r=1.56\sqrt{n}\Phi^{-1}\left(1-\frac{D_r}{2p^2}\right) \quad r=1, \dots ,7
\end{eqnarray*}
where $A_k, B_k, C_k \in \{0.01,0.05,0.1,0.3,0.5,1,2,4,8,16\}$ and $D_r \in
\{0.01,0.05,0.075,0.1,0.2,$\\ $0.5,1\}$.  
The two different performance measures are evaluated for the estimators
based on the sample $X^{(1)},\dots ,X^{(n)}$ with optimal CV-estimated tuning
parameters $\lambda$, $\tau$, $\rho$ and $\eta$ for each model from above. All
results are based on 50 independent simulation runs.

\subsubsection{The AR(1)-block model}

We consider two different covariance matrices. The first one is a simple
auto-regressive process of order one with trivial block size equal to
$p=300$, denoted by $\Sigma^{(1)}_{0}$. This is also known as a Toeplitz
matrix. That is, we have 
$\Sigma_{0;i,j}^{(1)} = 0.9^{|i-j|} \; \forall \; i, j \in  \{1, . . . ,p\}$.
 The second matrix
$\Sigma_{0}^{(2)}$ is a block-diagonal matrix with AR(1) blocks of equal
block size $30 \times 30$, 
and hence the block-diagonal of $\Sigma_{0}^{(2)}$ equals
$\Sigma_{Block;i,j} = 0.9^{|i-j|}$, $i,j\in \{1,\dots ,30\}$.
The simulation results for the AR(1)-block models are shown in Figure
\ref{causal1} and \ref{causal2}.

\begin{figure}
        \centerline{ 
          \subfigure[
          $n=40$]{\includegraphics[trim=0cm 0cm 0cm 1cm ,clip=TRUE ,scale=0.3]{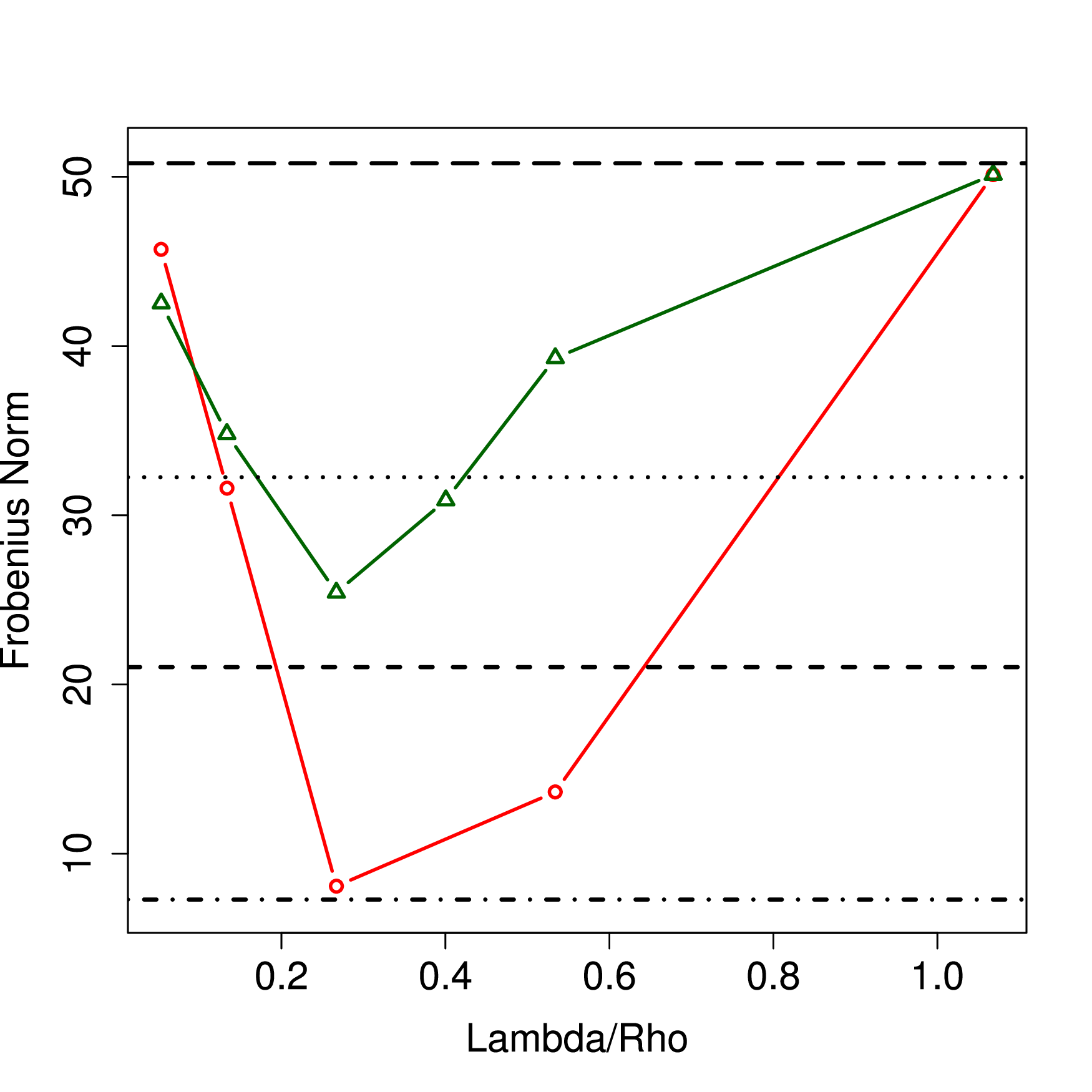}}
          \subfigure[$n=80$]{\includegraphics[trim=0cm 0cm 0cm 1cm ,clip=TRUE ,scale=0.3]{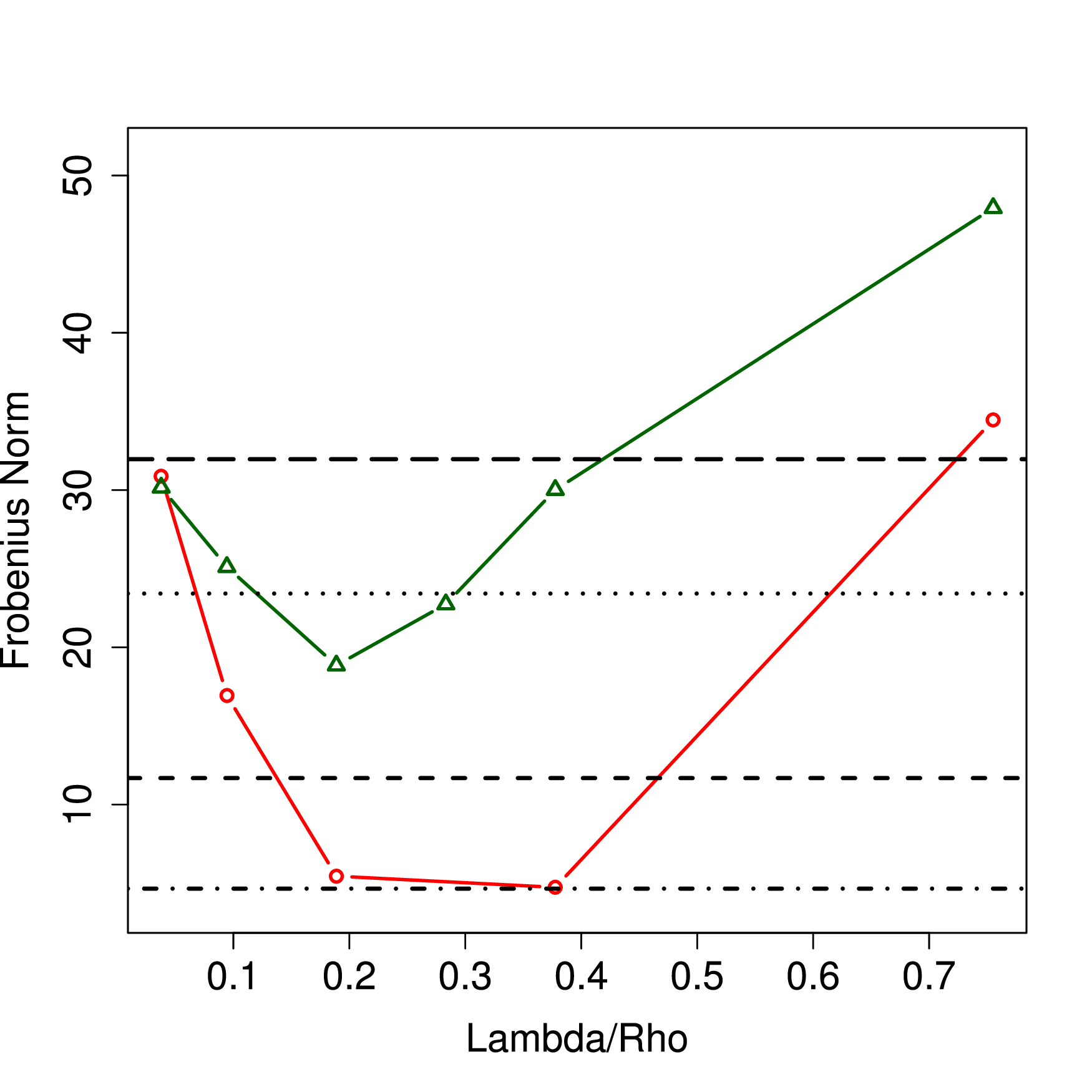}}
          \subfigure[$n=320$]{\includegraphics[trim=0cm 0cm 0cm 1cm ,clip=TRUE ,scale=0.3]{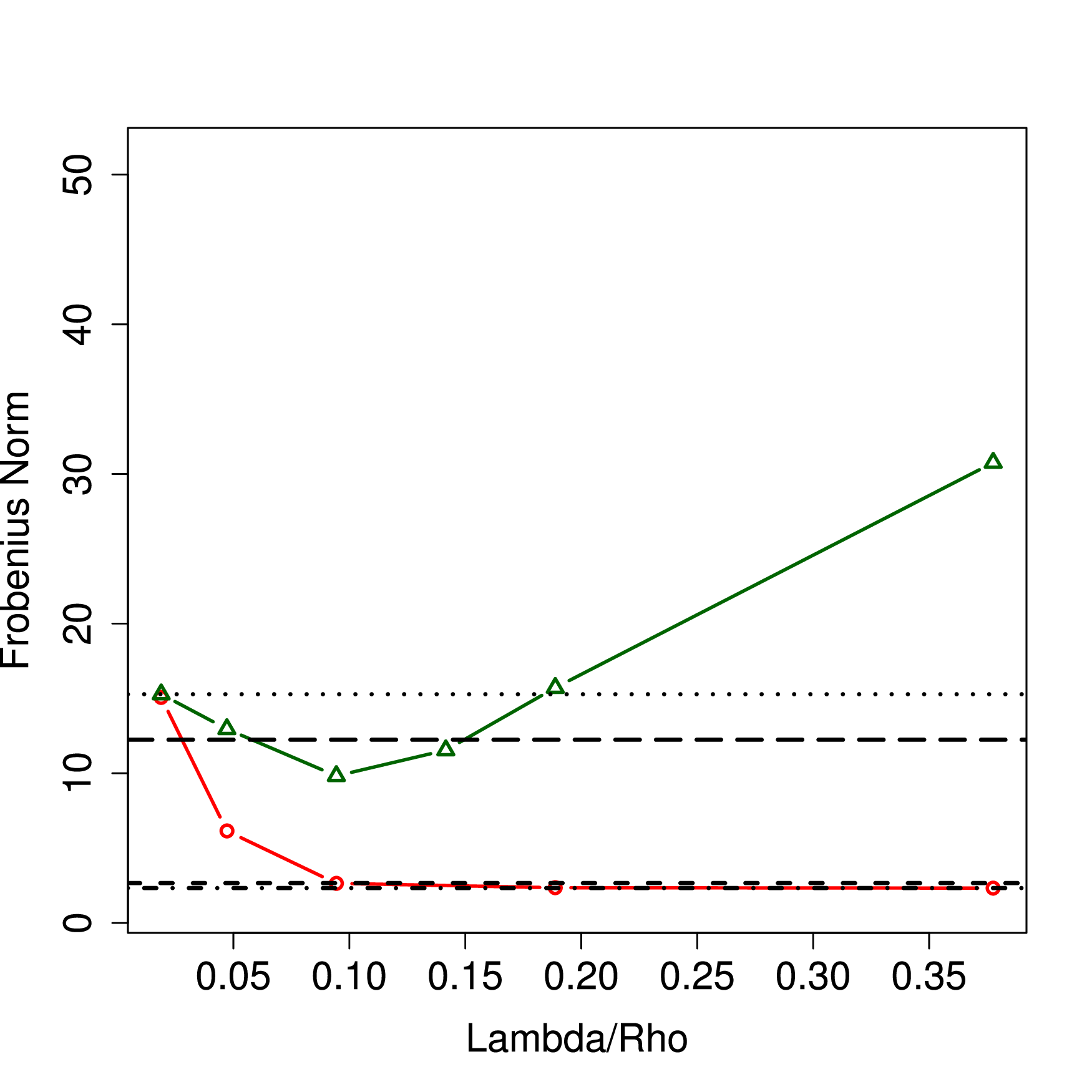}}
}

        \centerline{ 
          \subfigure[$n=40$]{\includegraphics[trim=0cm 0cm 0cm 1cm ,clip=TRUE ,scale=0.3]{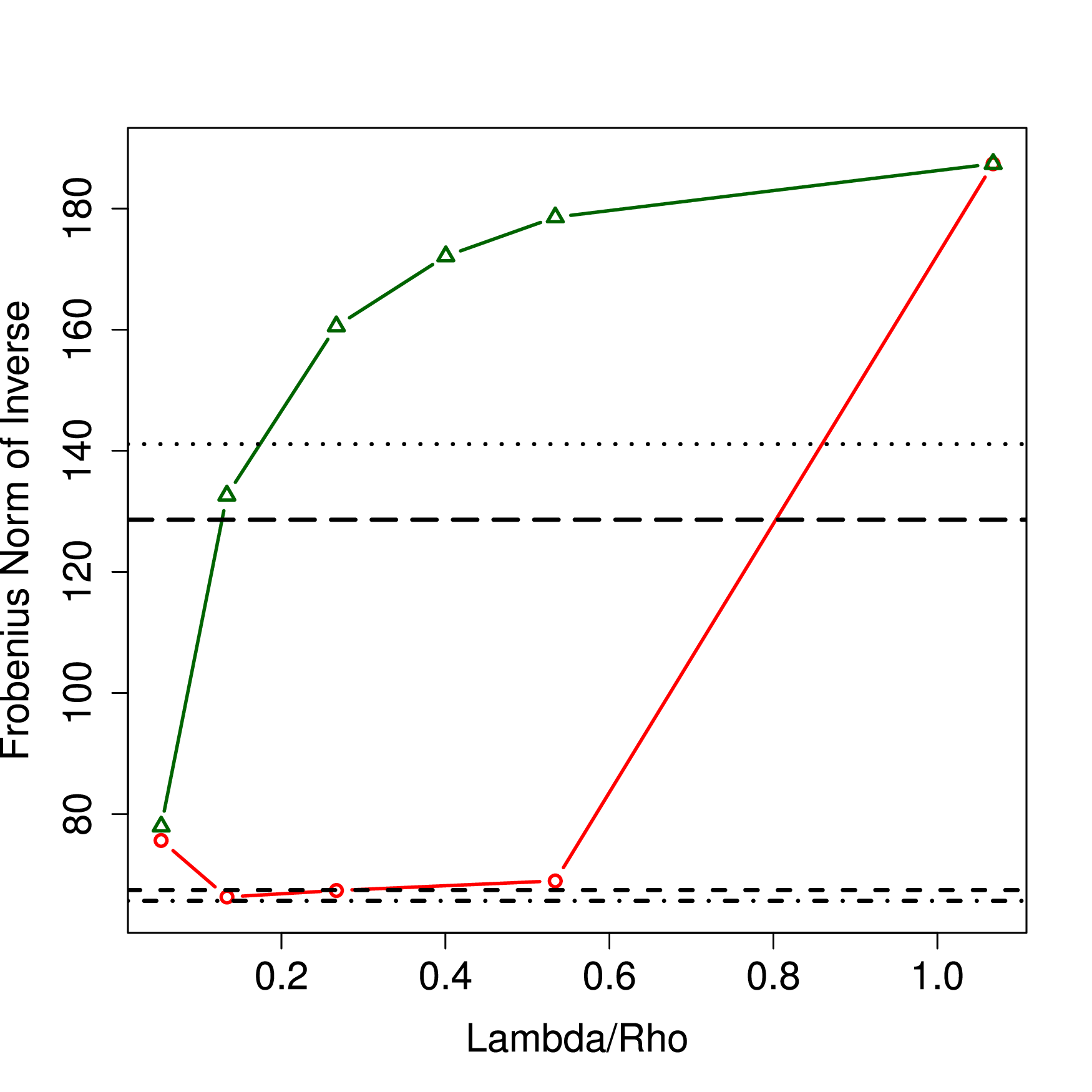}}
          \subfigure[$n=80$]{\includegraphics[trim=0cm 0cm 0cm 1cm ,clip=TRUE ,scale=0.3]{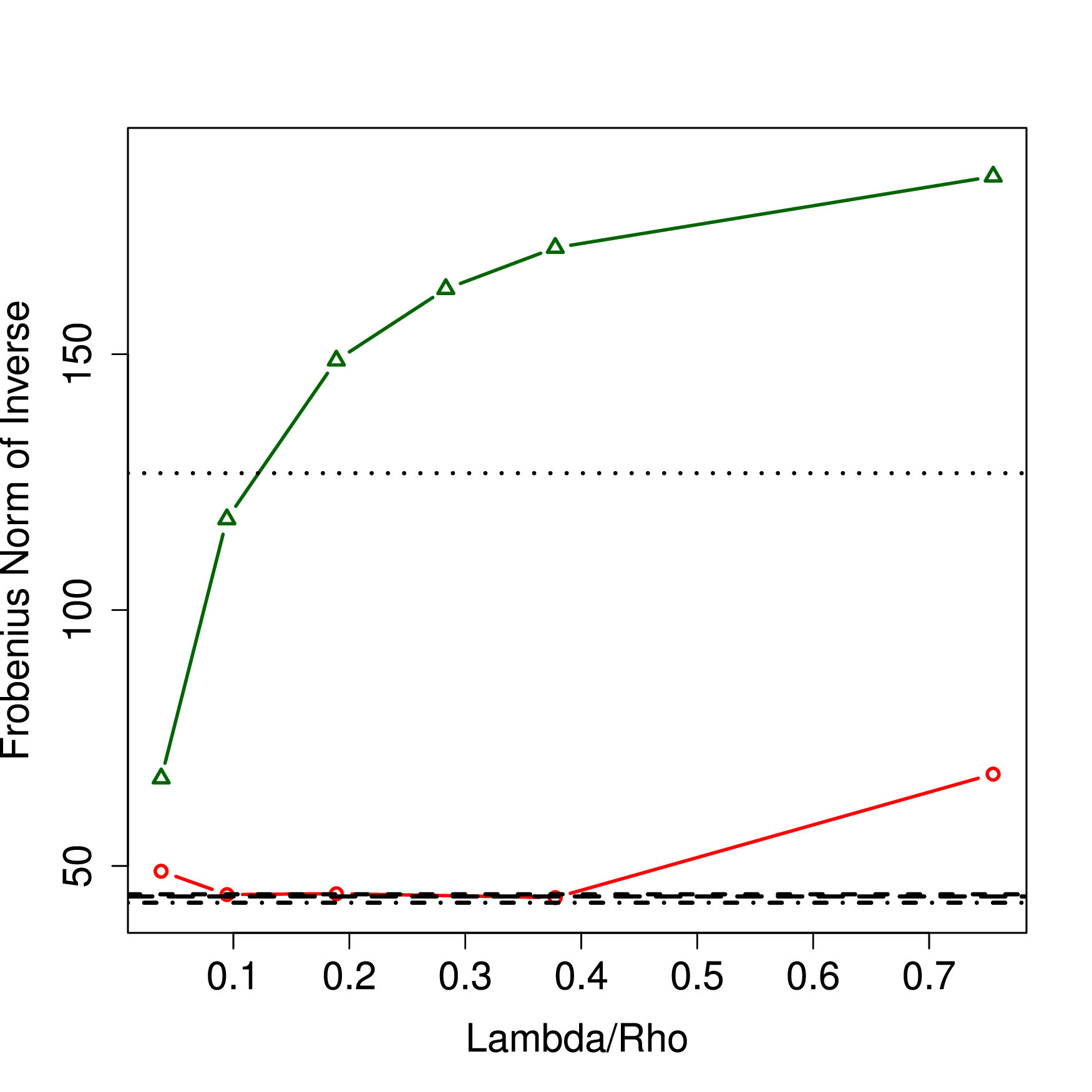}}
          \subfigure[$n=320$]{\includegraphics[trim=0cm 0cm 0cm 1cm ,clip=TRUE ,scale=0.3]{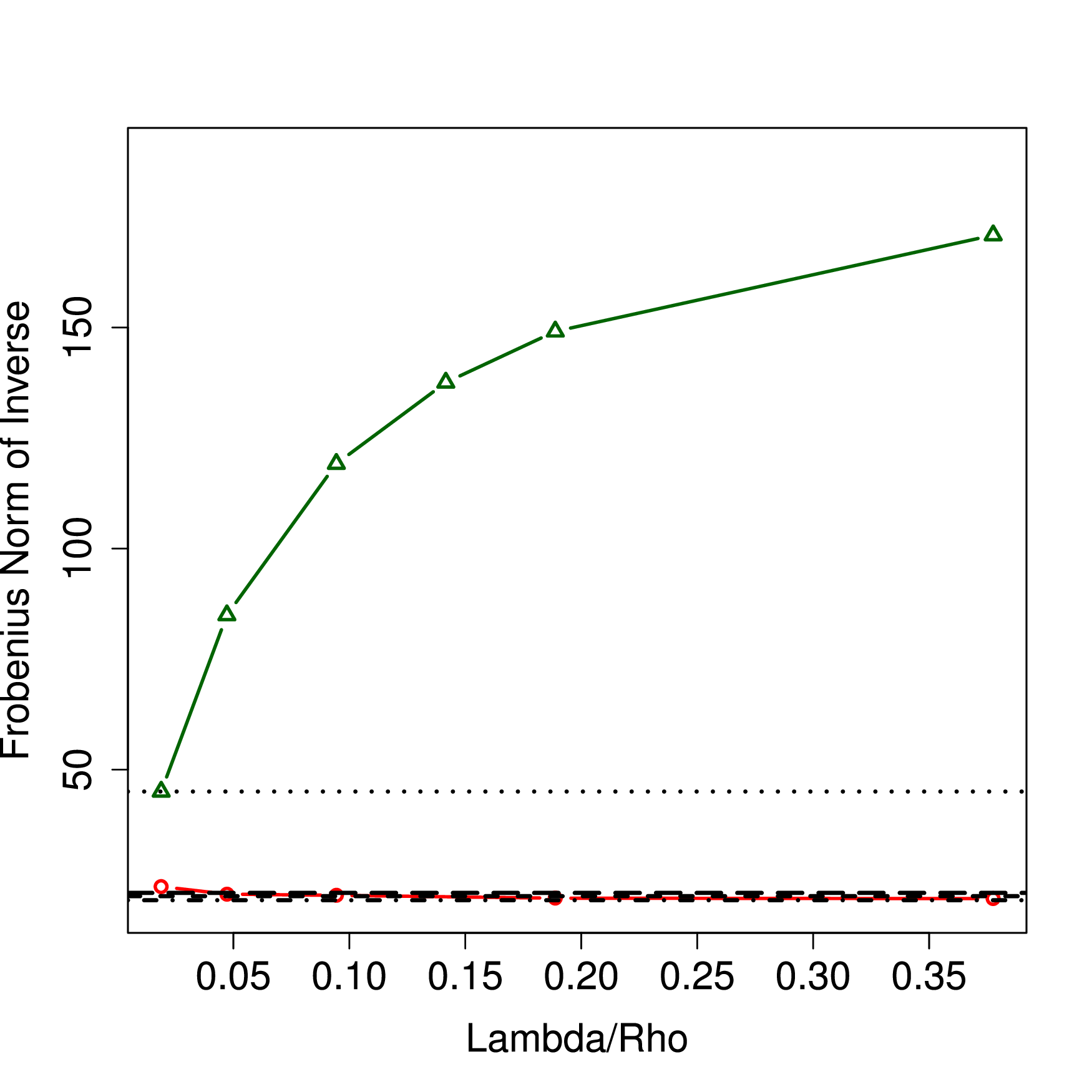}}
}

         \centerline{ 
          \subfigure[$n=40$]{\includegraphics[trim=0cm 0cm 0cm 1cm ,clip=TRUE ,scale=0.3]{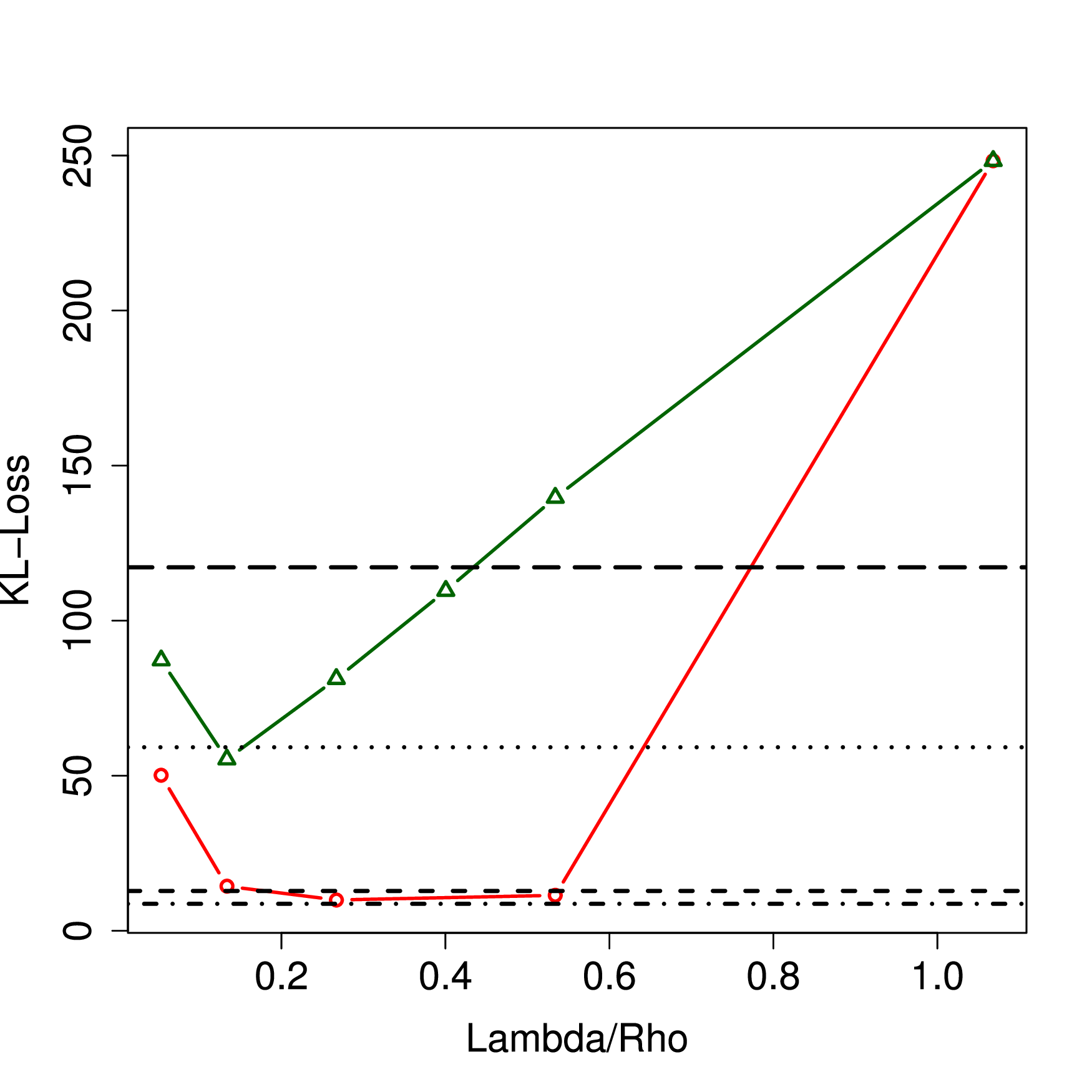}}
          \subfigure[$n=80$]{\includegraphics[trim=0cm 0cm 0cm 1cm ,clip=TRUE ,scale=0.3]{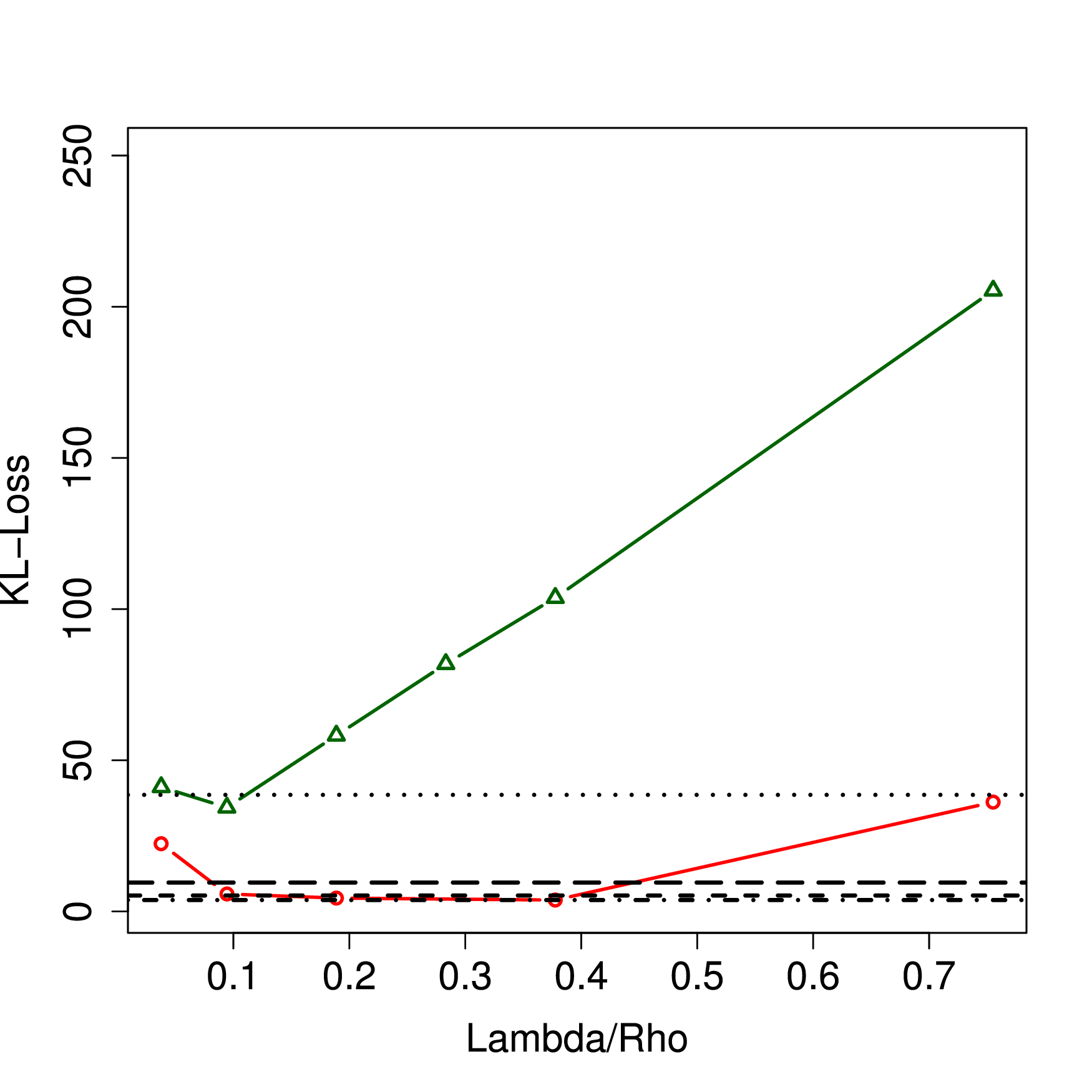}}
          \subfigure[$n=320$]{\includegraphics[trim=0cm 0cm 0cm 1cm ,clip=TRUE ,scale=0.3]{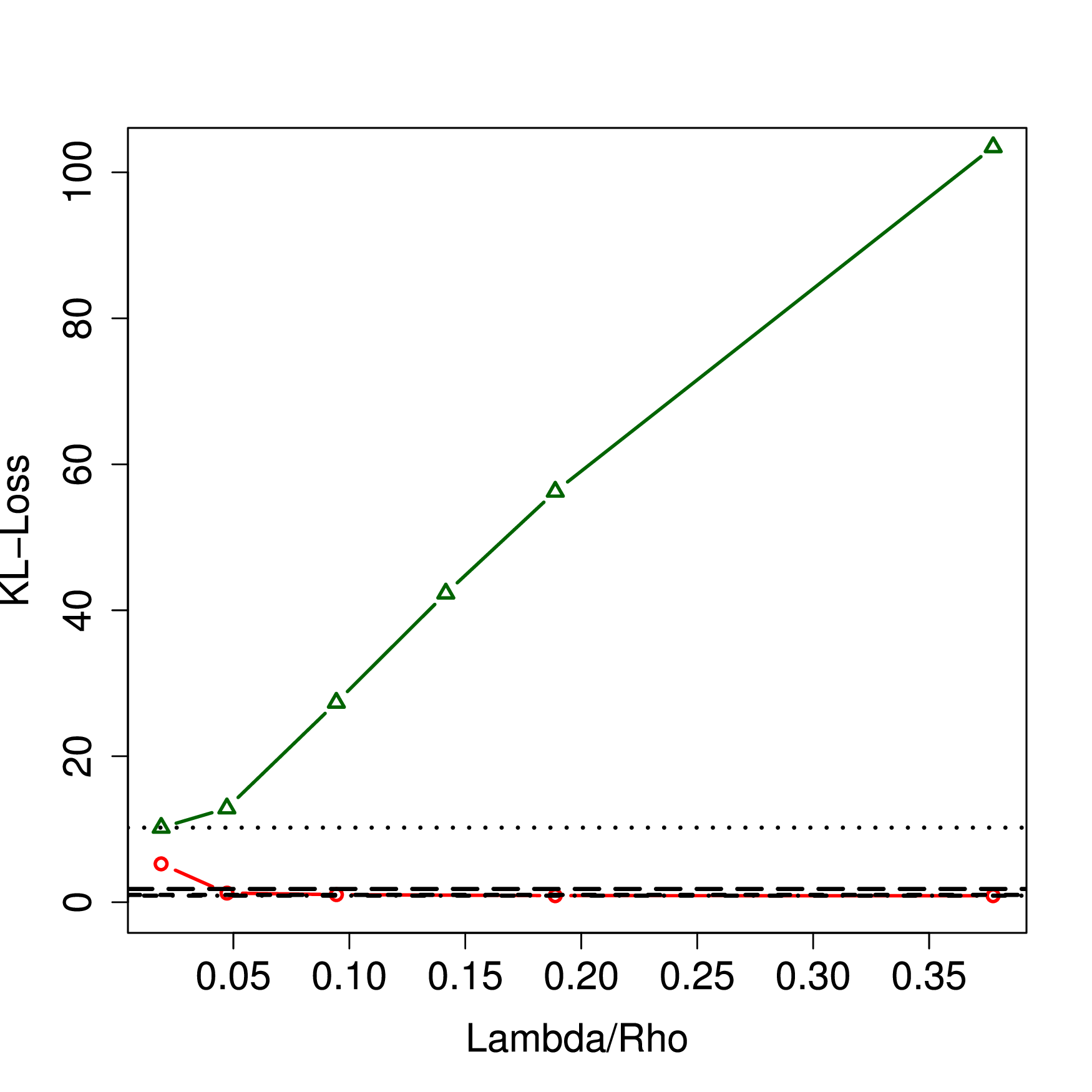}}
}
        \caption{Plots for model $\Sigma^{(1)}_{0}$. The triangles (green)
          stand for the 
          GLasso and the circles (red) for our Gelato method with a
          reasonable value of $\tau$. The horizontal lines
          show the performances of the three techniques for cross-validated
          tuning parameters $\lambda$, $\tau$, $\rho$ and $\eta$. The
          dashed line stands for our Gelato method, the dotted one for the
          GLasso and the dash-dotted line for the Space technique. The
          additional dashed line with the longer dashes stands for the
          Gelato without thresholding. Lambda/Rho stands for $\lambda$ or
          $\rho$, respectively.} 
        \label{causal1}
\end{figure}

\begin{figure}
        \centerline{ 
          \subfigure[$\Sigma^{(2)}_{AR}$ with $n=40$]{\includegraphics[trim=0cm 0cm 0cm 1cm ,clip=TRUE ,scale=0.3]{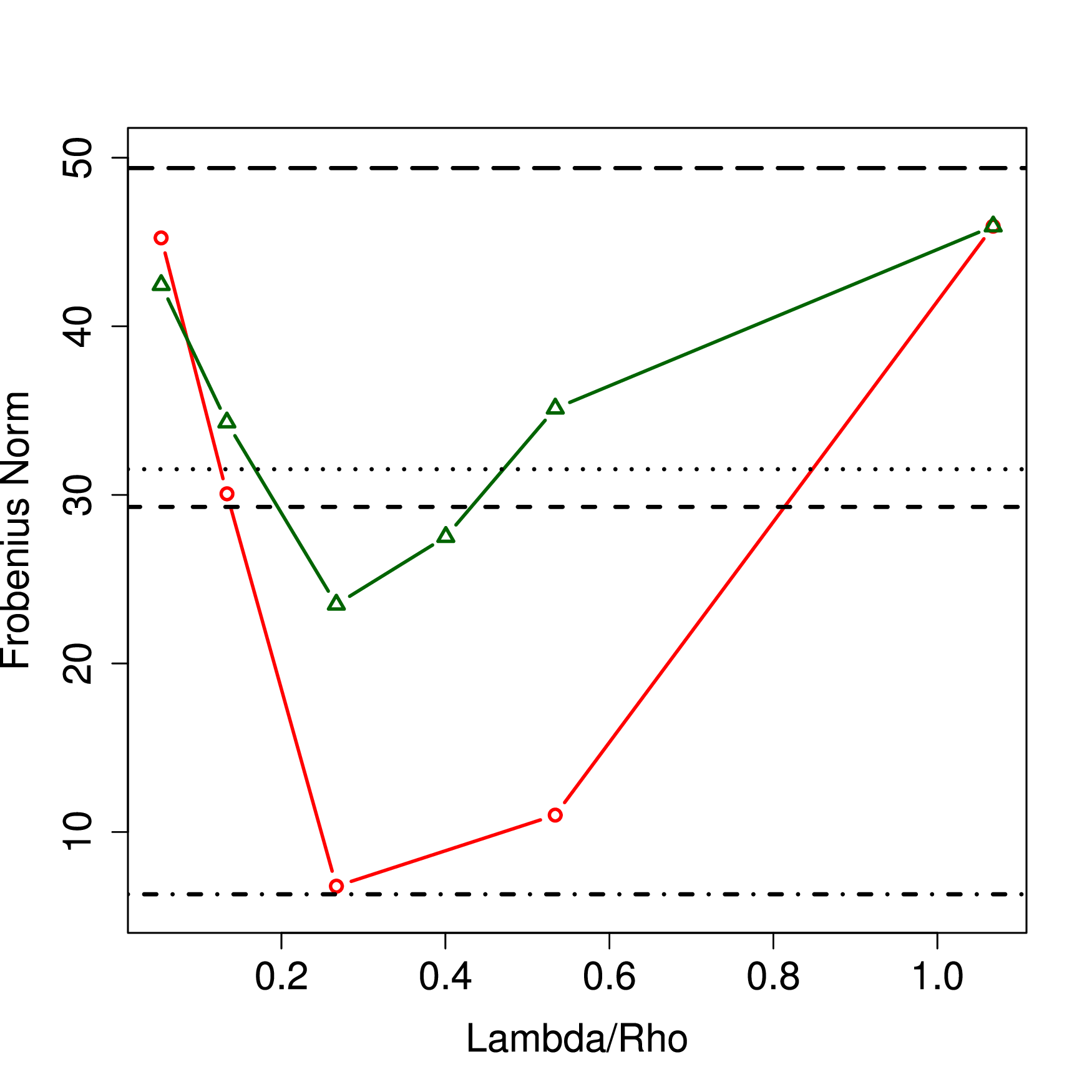}}
          \subfigure[$\Sigma^{(2)}_{AR}$ with $n=80$]{\includegraphics[trim=0cm 0cm 0cm 1cm ,clip=TRUE ,scale=0.3]{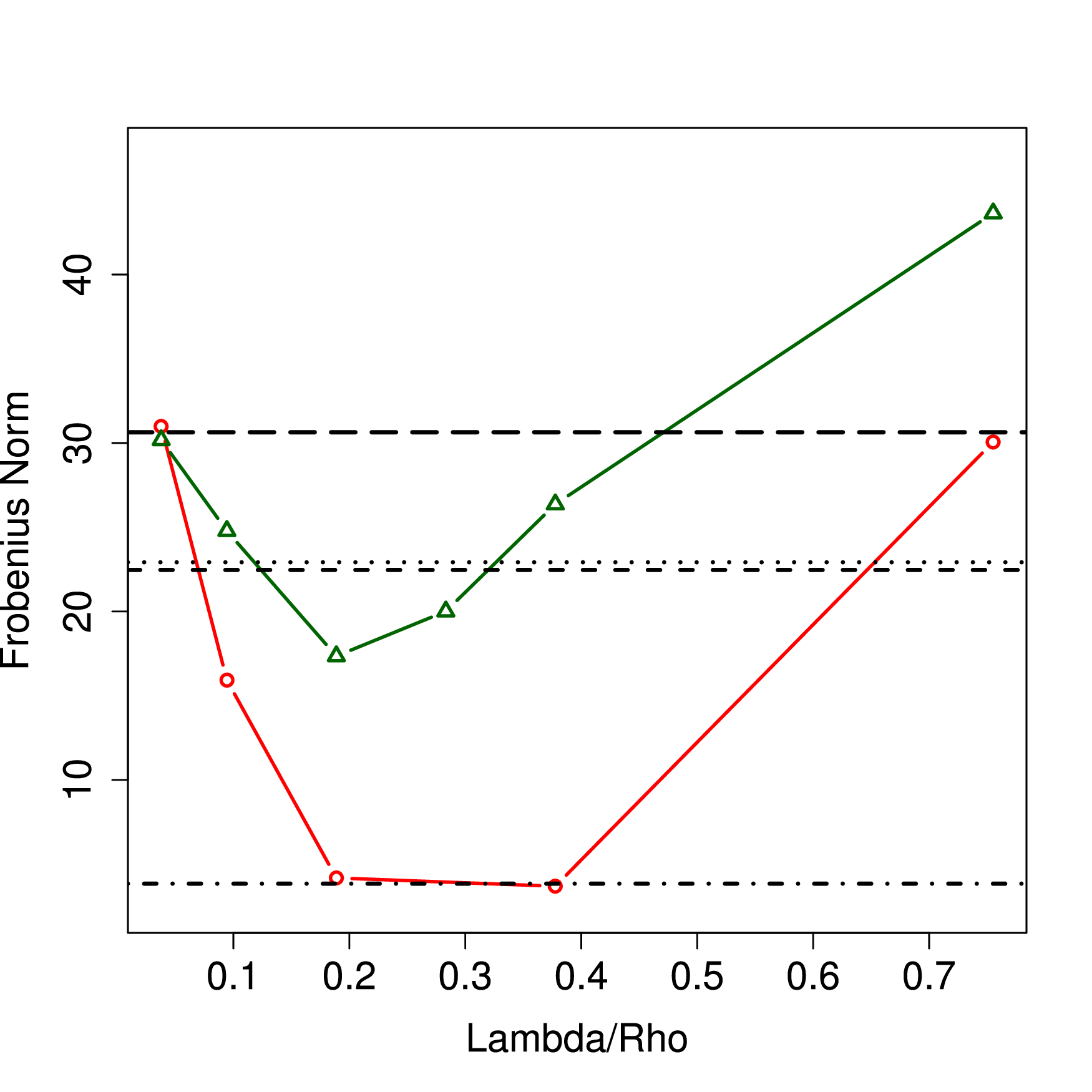}}
          \subfigure[$\Sigma^{(2)}_{AR}$ with $n=320$]{\includegraphics[trim=0cm 0cm 0cm 1cm ,clip=TRUE ,scale=0.3]{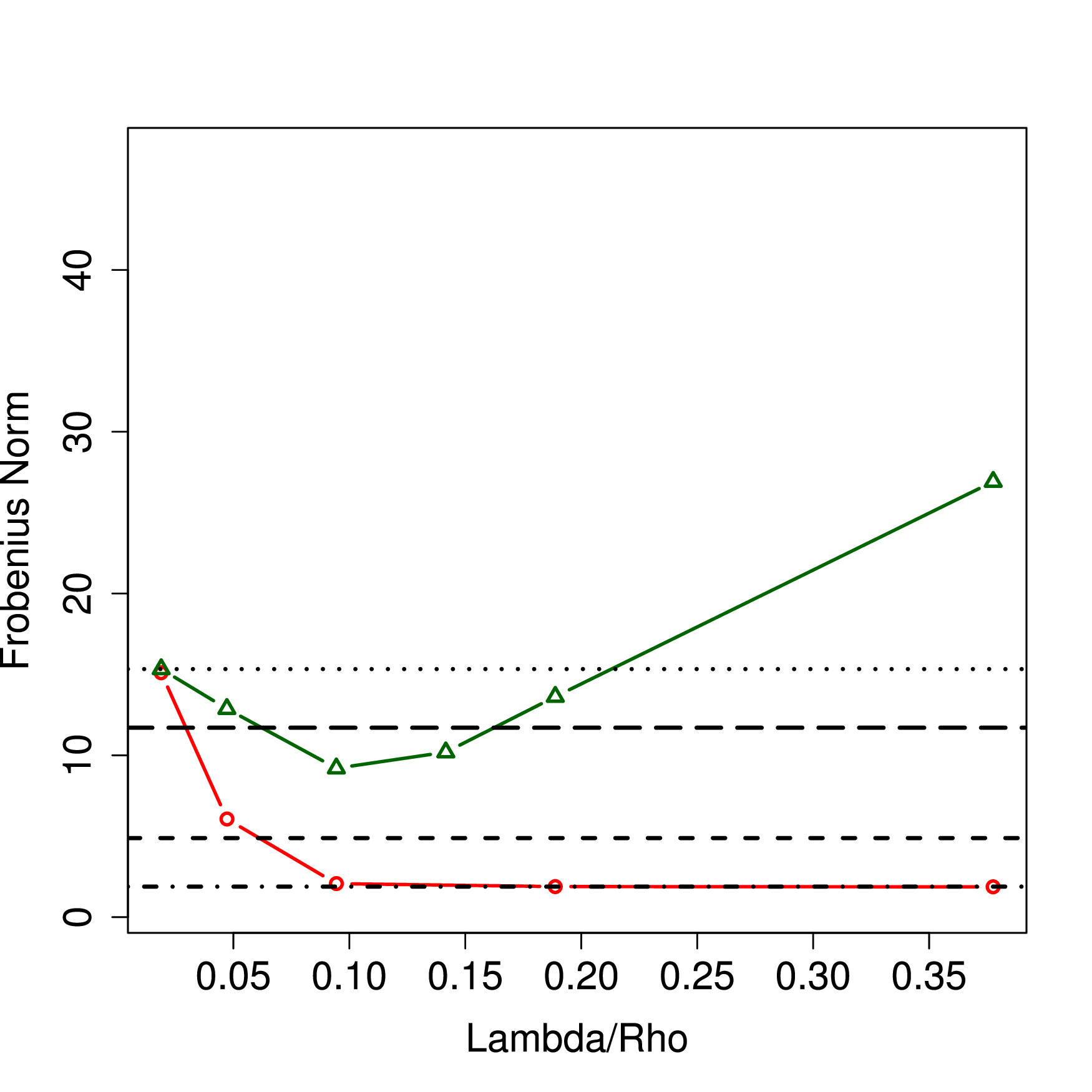}}
}

        \centerline{ 
          \subfigure[$\Sigma^{(2)}_{AR}$ with $n=40$]{\includegraphics[trim=0cm 0cm 0cm 1cm ,clip=TRUE ,scale=0.3]{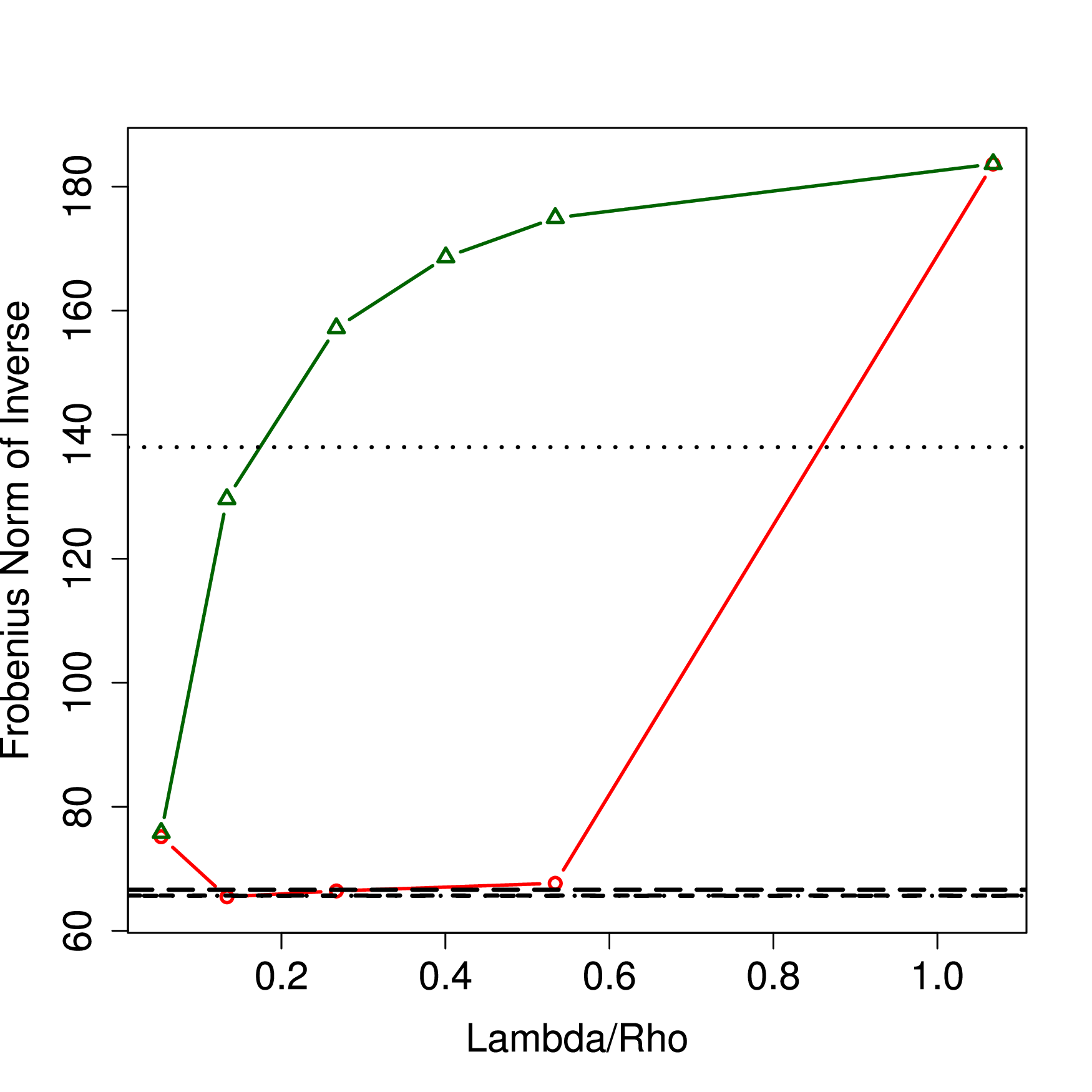}}
          \subfigure[$\Sigma^{(2)}_{AR}$ with $n=80$]{\includegraphics[trim=0cm 0cm 0cm 1cm ,clip=TRUE ,scale=0.3]{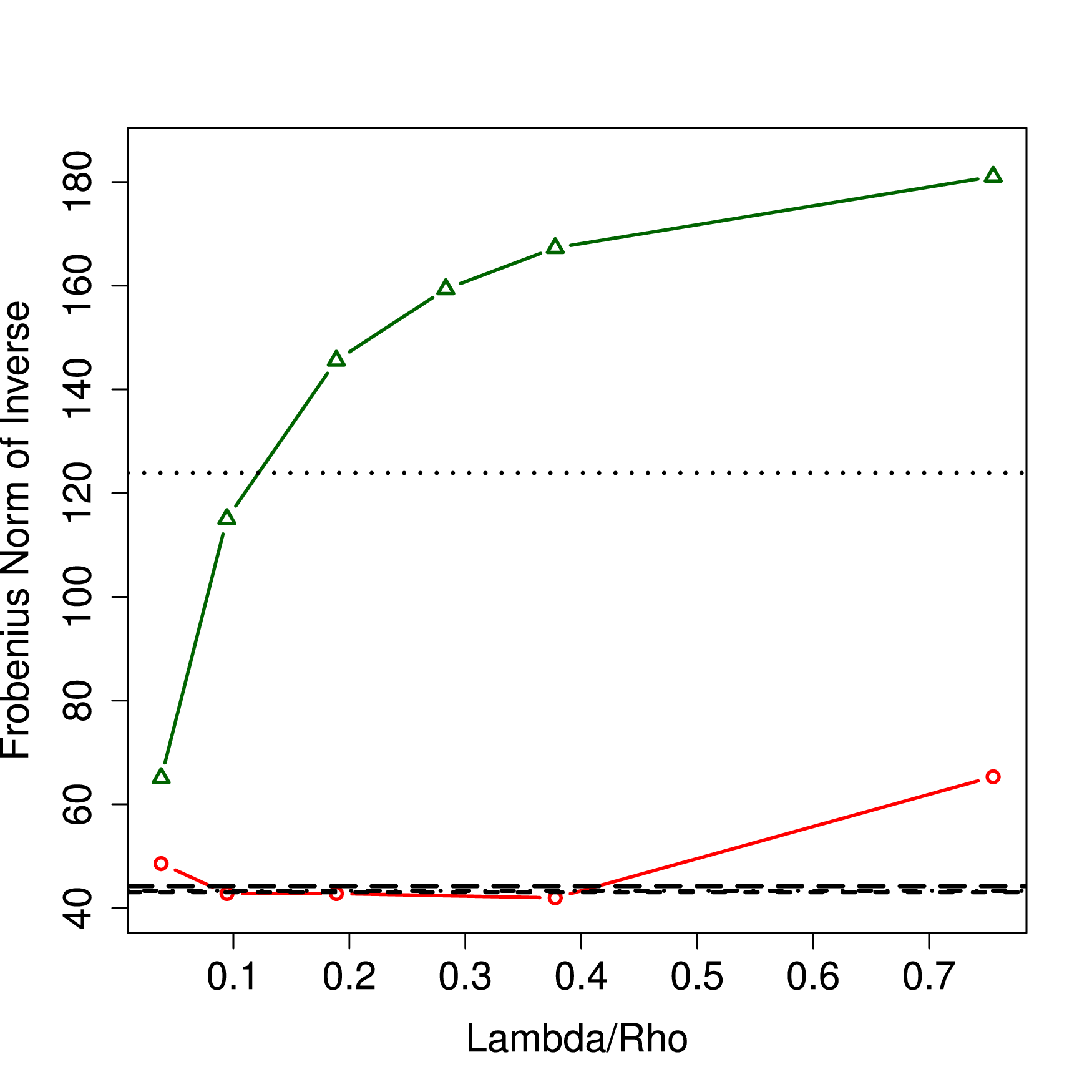}}
          \subfigure[$\Sigma^{(2)}_{AR}$ with $n=320$]{\includegraphics[trim=0cm 0cm 0cm 1cm ,clip=TRUE ,scale=0.3]{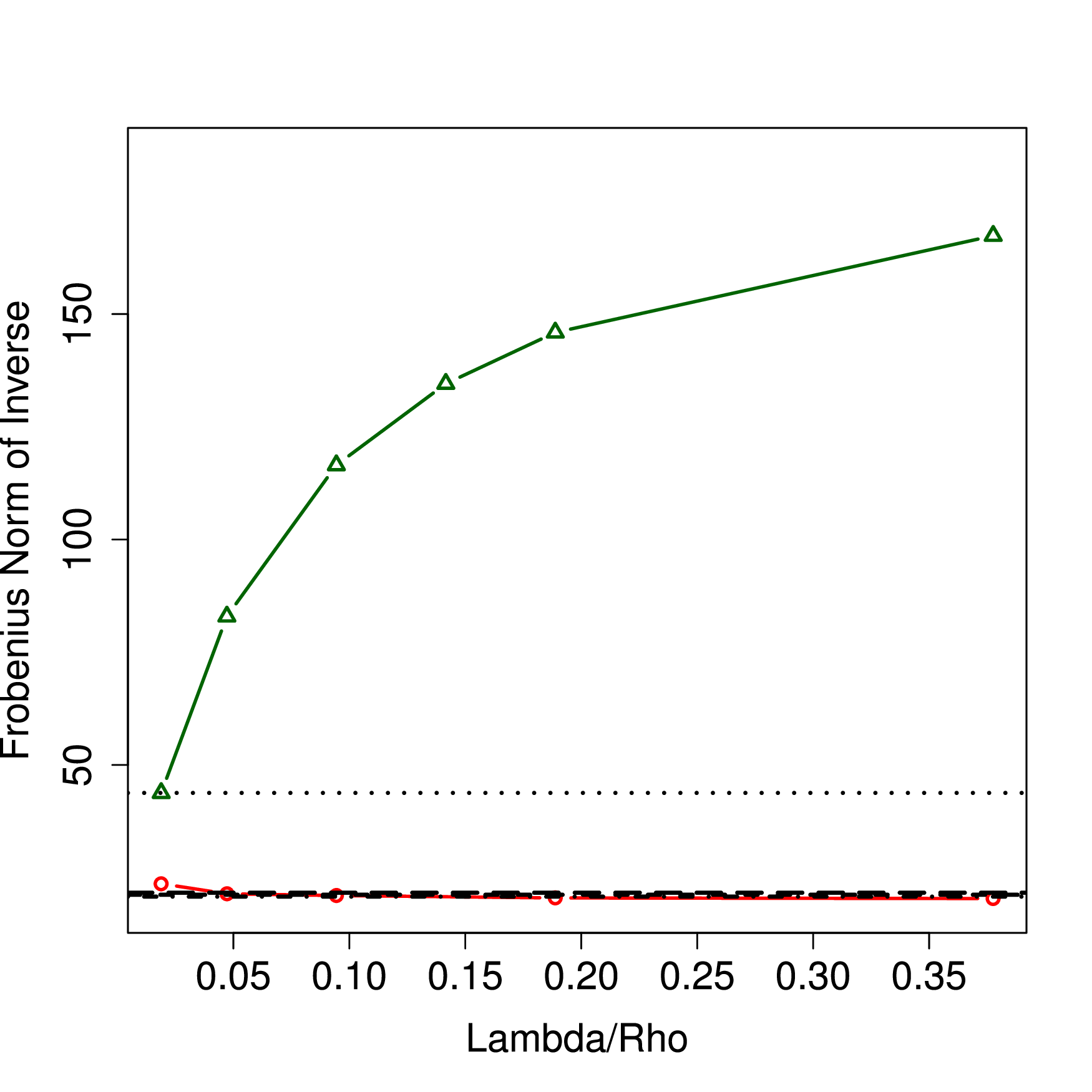}}
}

         \centerline{ 
          \subfigure[$\Sigma^{(2)}_{AR}$ with $n=40$]{\includegraphics[trim=0cm 0cm 0cm 1cm ,clip=TRUE ,scale=0.3]{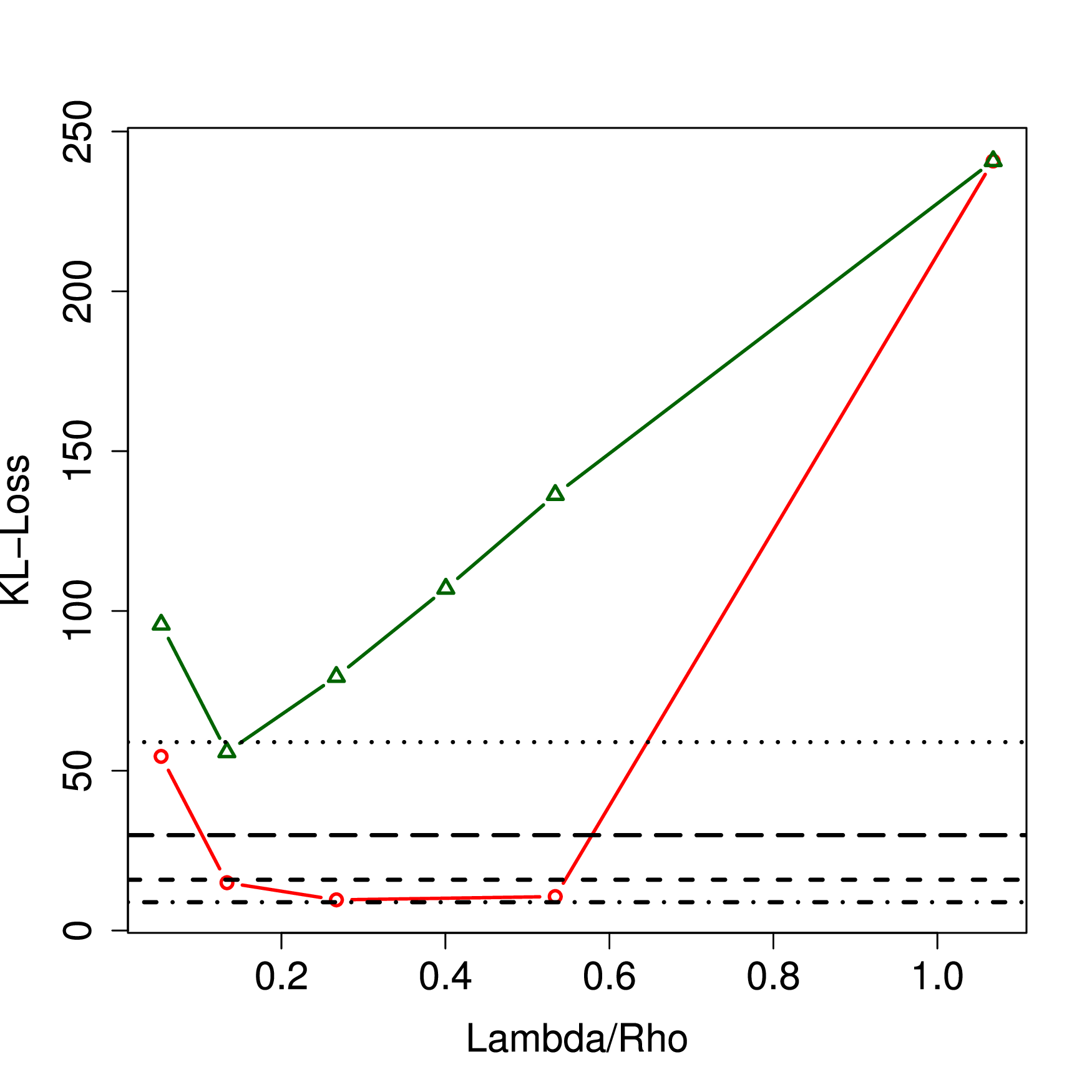}}
          \subfigure[$\Sigma^{(2)}_{AR}$ with $n=80$]{\includegraphics[trim=0cm 0cm 0cm 1cm ,clip=TRUE ,scale=0.3]{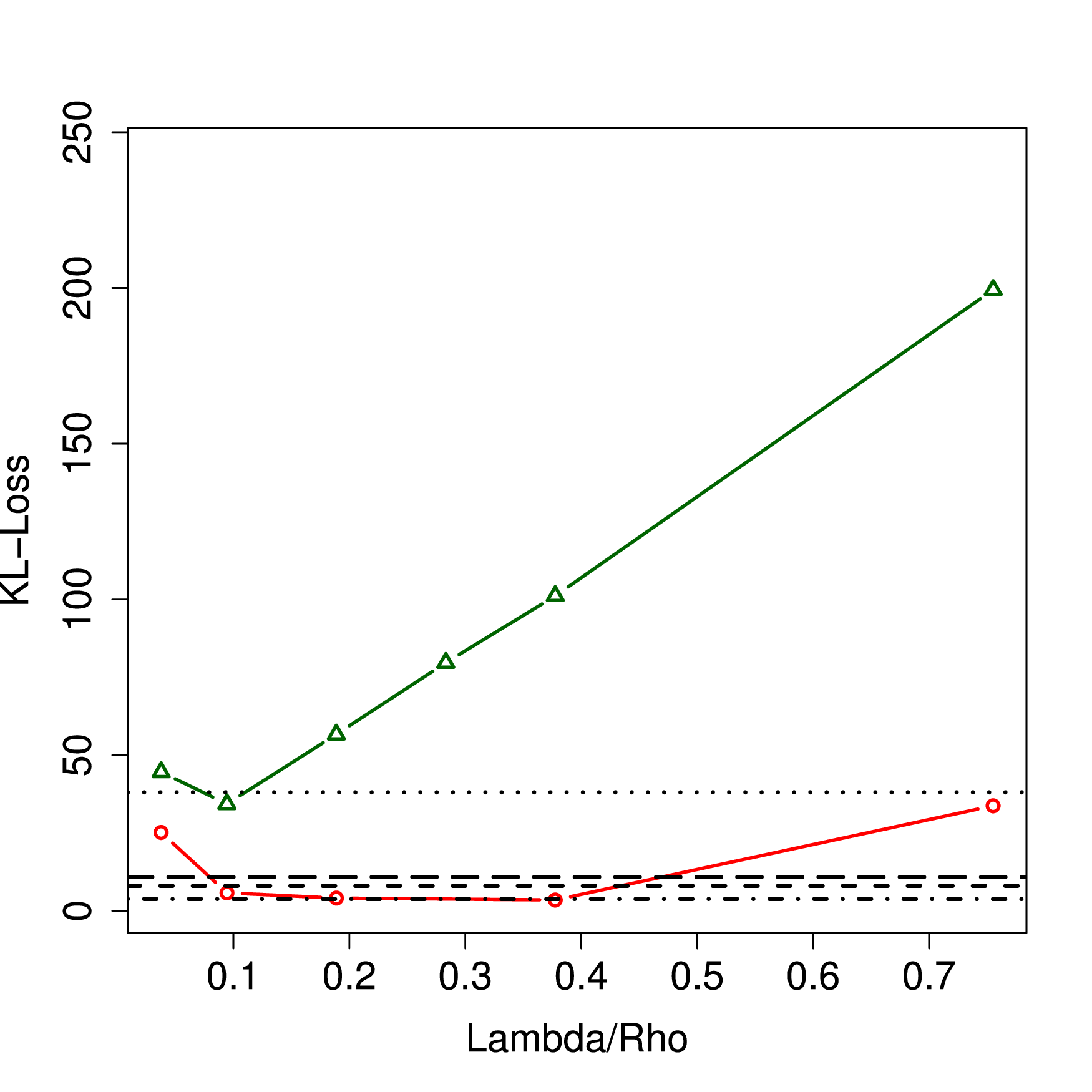}}
          \subfigure[$\Sigma^{(2)}_{AR}$ with $n=320$]{\includegraphics[trim=0cm 0cm 0cm 1cm ,clip=TRUE ,scale=0.3]{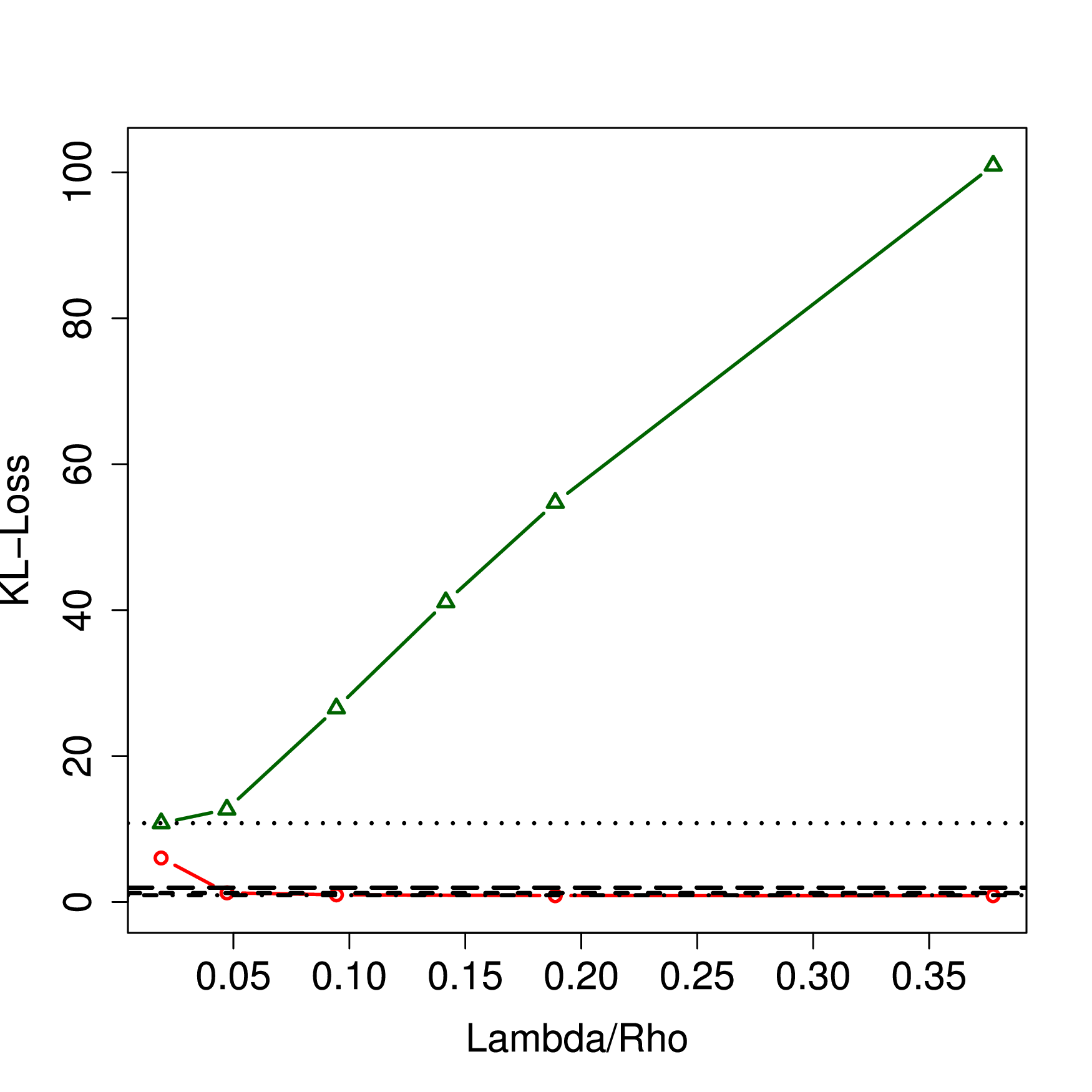}}
}
        \caption{Plots for model $\Sigma^{(2)}_{0}$. The triangles (green)
          stand for the 
          GLasso and the circles (red) for our Gelato method with a
          reasonable value of $\tau$. The horizontal lines
          show the performances of the three techniques for cross-validated
          tuning parameters $\lambda$, $\tau$, $\rho$ and $\eta$. The
          dashed line stands for our Gelato method, the dotted one for the
          GLasso and the dash-dotted line for the Space technique. The
          additional dashed line with the longer dashes stands for the
          Gelato without thresholding. Lambda/Rho stands for $\lambda$ or
          $\rho$, respectively.}
        \label{causal2}
\end{figure}

The figures show a substantial performance gain of our method
compared to the GLasso in both considered covariance models. This result
speaks for our method, especially because AR(1)-block models are very
simple. The Space method performs about as well as Gelato, except for the
Frobenius norm of $\hat{\Sigma}_n-\Sigma_0$. There we see an performance
advantage of the Space method compared to Gelato. We also exploit
the clear advantage of thresholding in Gelato for a small sample size.

\subsubsection{The random precision matrix model}

\begin{figure}
        \centerline{ 
          \subfigure[$n=40$]{\includegraphics[trim=0cm 0cm 0cm 1cm ,clip=TRUE ,scale=0.3]{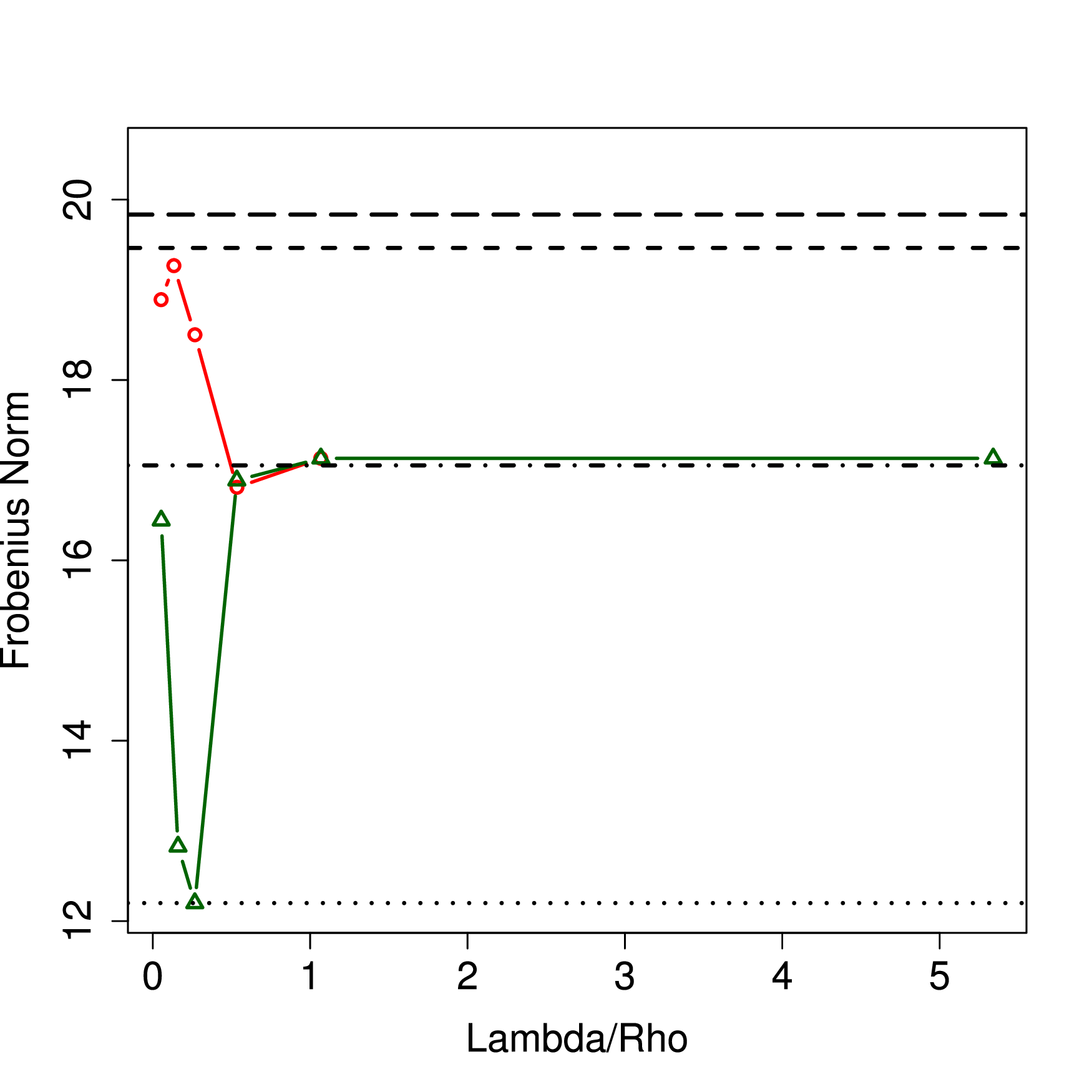}}
          \subfigure[$n=80$]{\includegraphics[trim=0cm 0cm 0cm 1cm ,clip=TRUE ,scale=0.3]{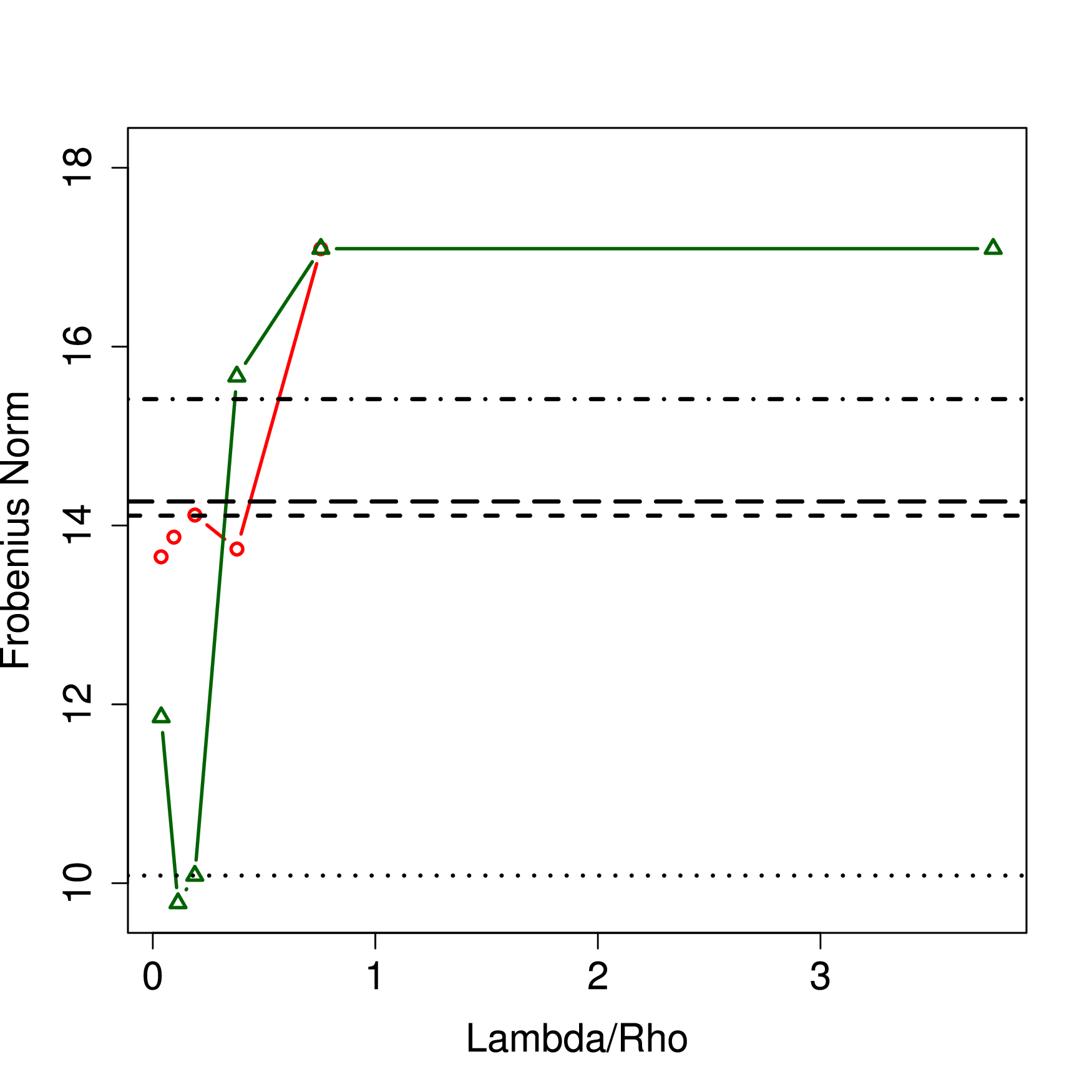}}
          \subfigure[$n=320$]{\includegraphics[trim=0cm 0cm 0cm 1cm ,clip=TRUE ,scale=0.3]{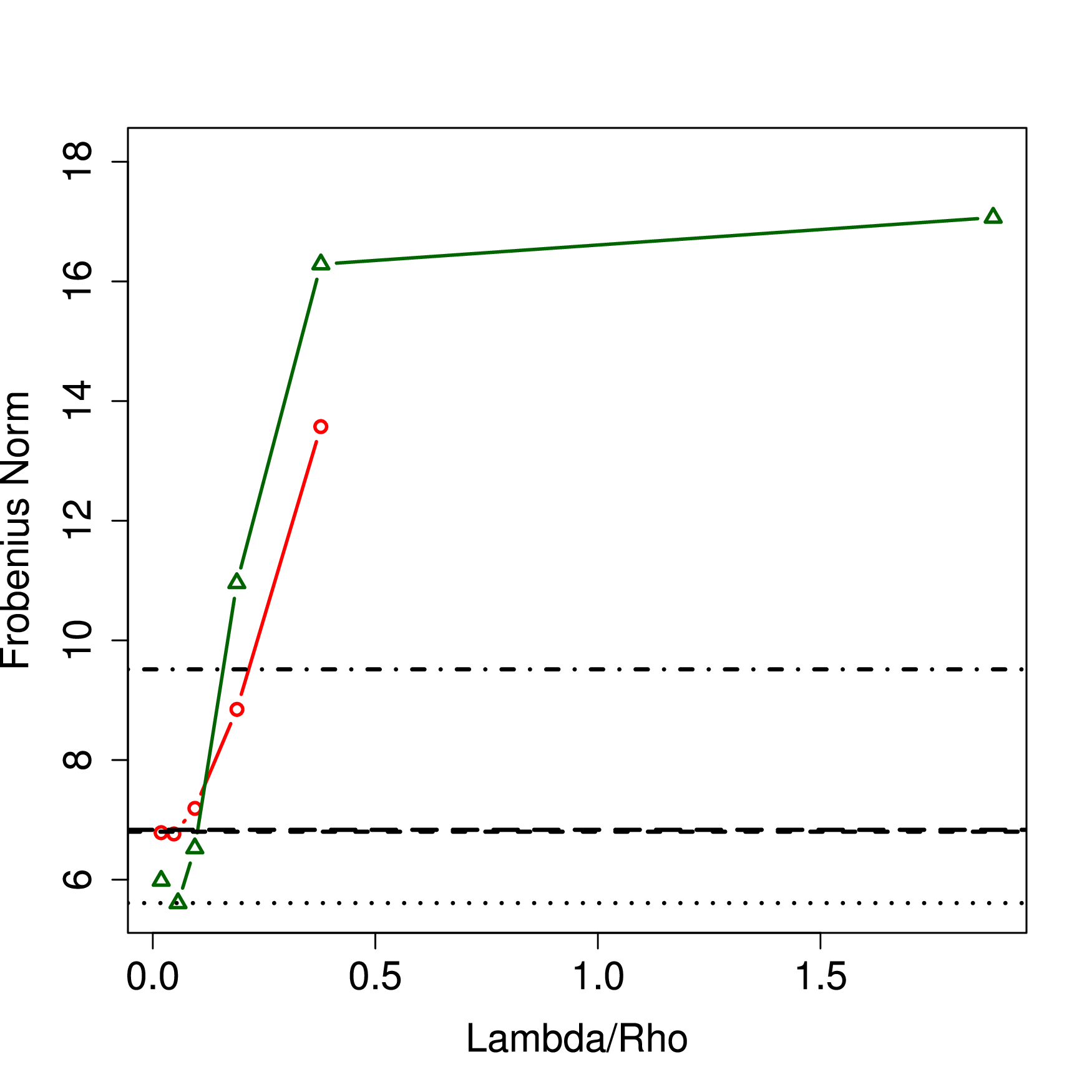}}
}

        \centerline{ 
          \subfigure[$n=40$]{\includegraphics[trim=0cm 0cm 0cm 1cm ,clip=TRUE ,scale=0.3]{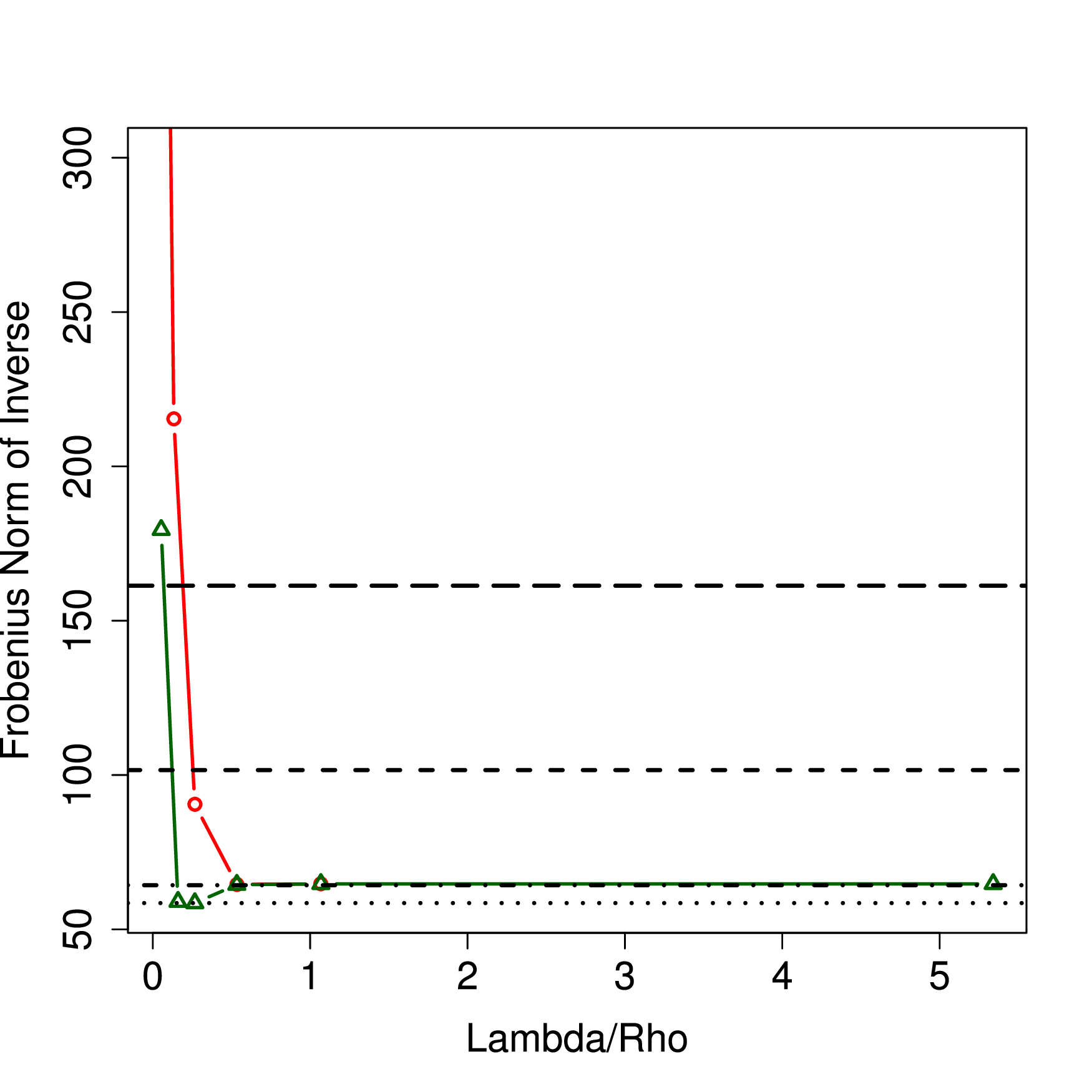}}
          \subfigure[$n=80$]{\includegraphics[trim=0cm 0cm 0cm 1cm ,clip=TRUE ,scale=0.3]{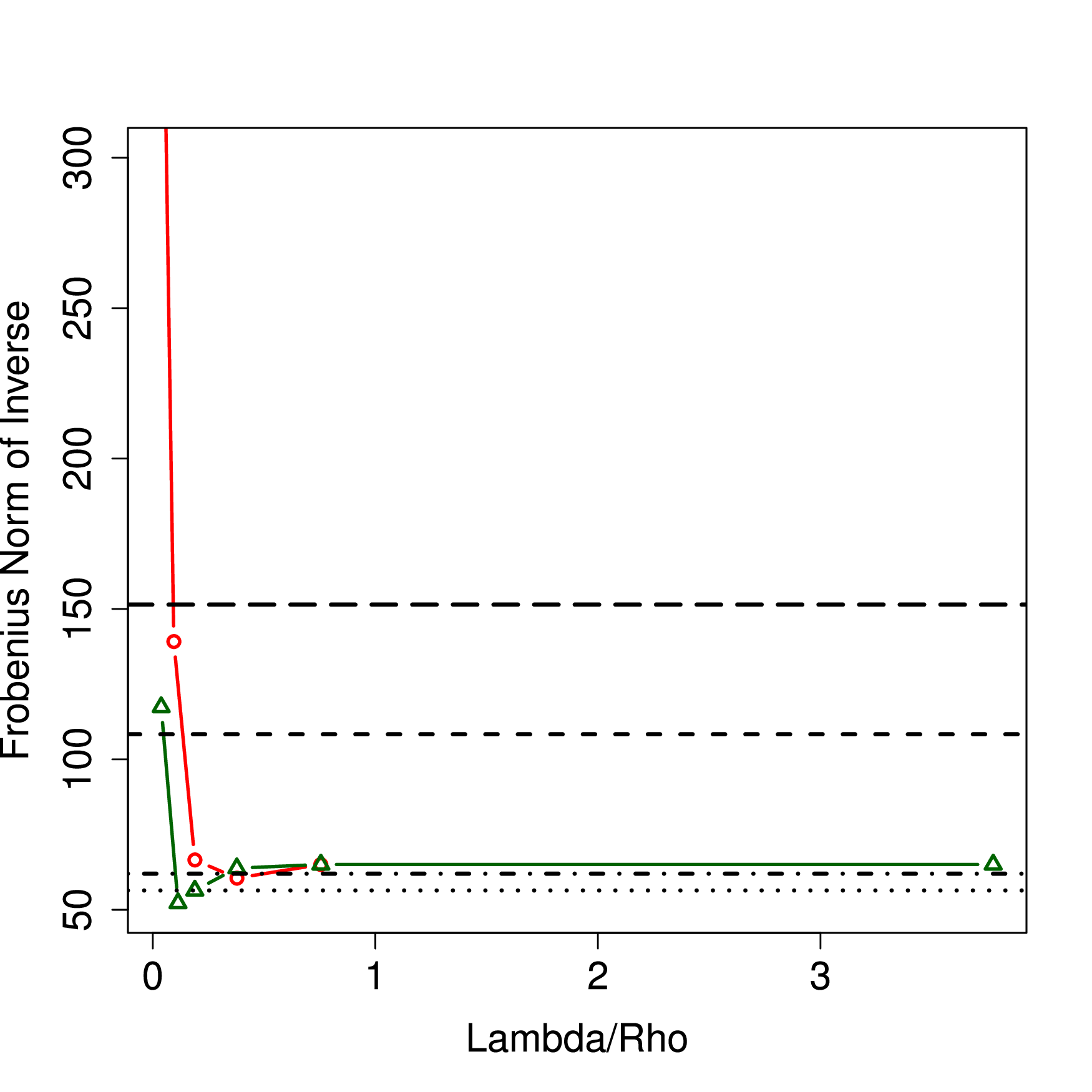}}
          \subfigure[$n=320$]{\includegraphics[trim=0cm 0cm 0cm 1cm ,clip=TRUE ,scale=0.3]{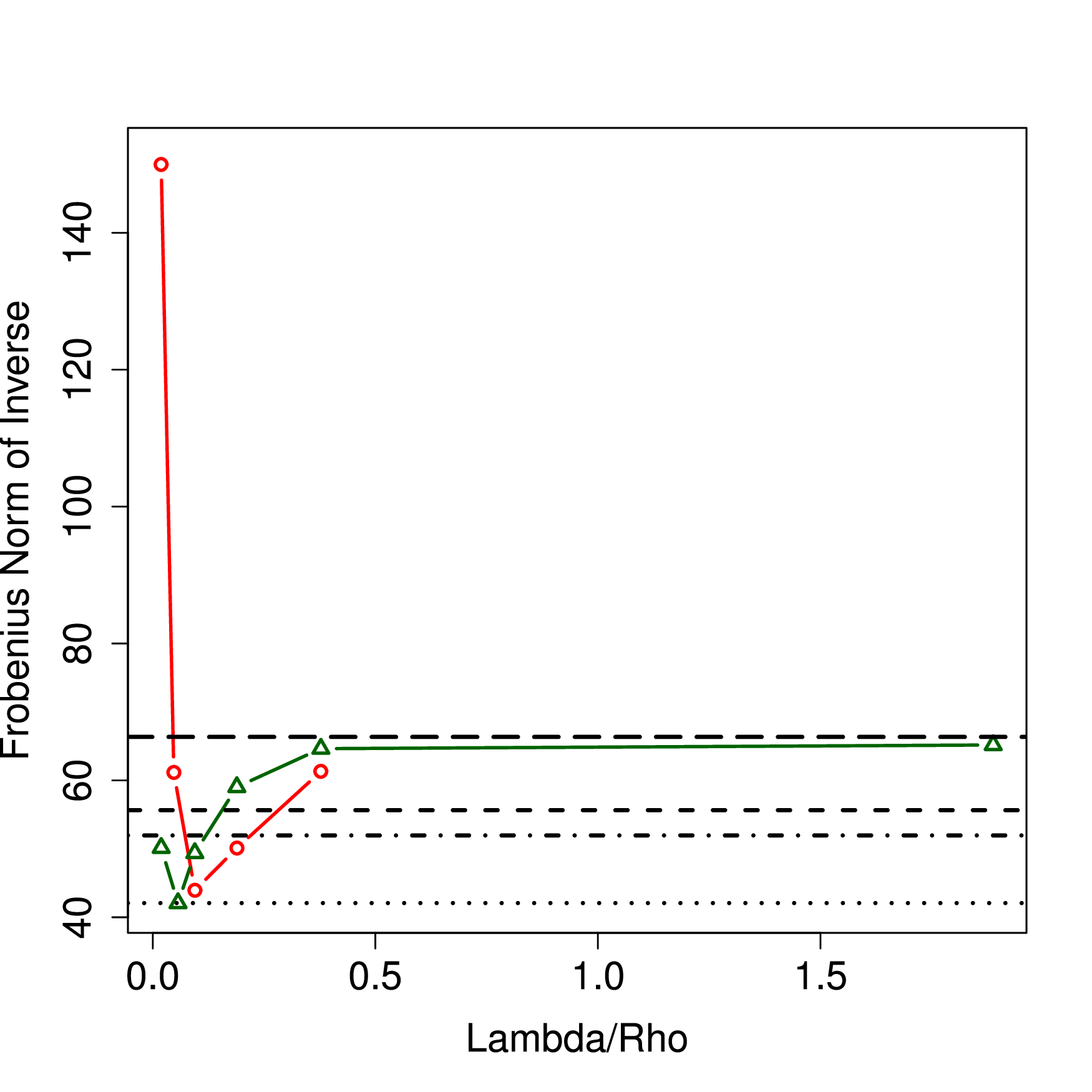}}
}

         \centerline{ 
          \subfigure[$n=40$]{\includegraphics[trim=0cm 0cm 0cm 1cm ,clip=TRUE ,scale=0.3]{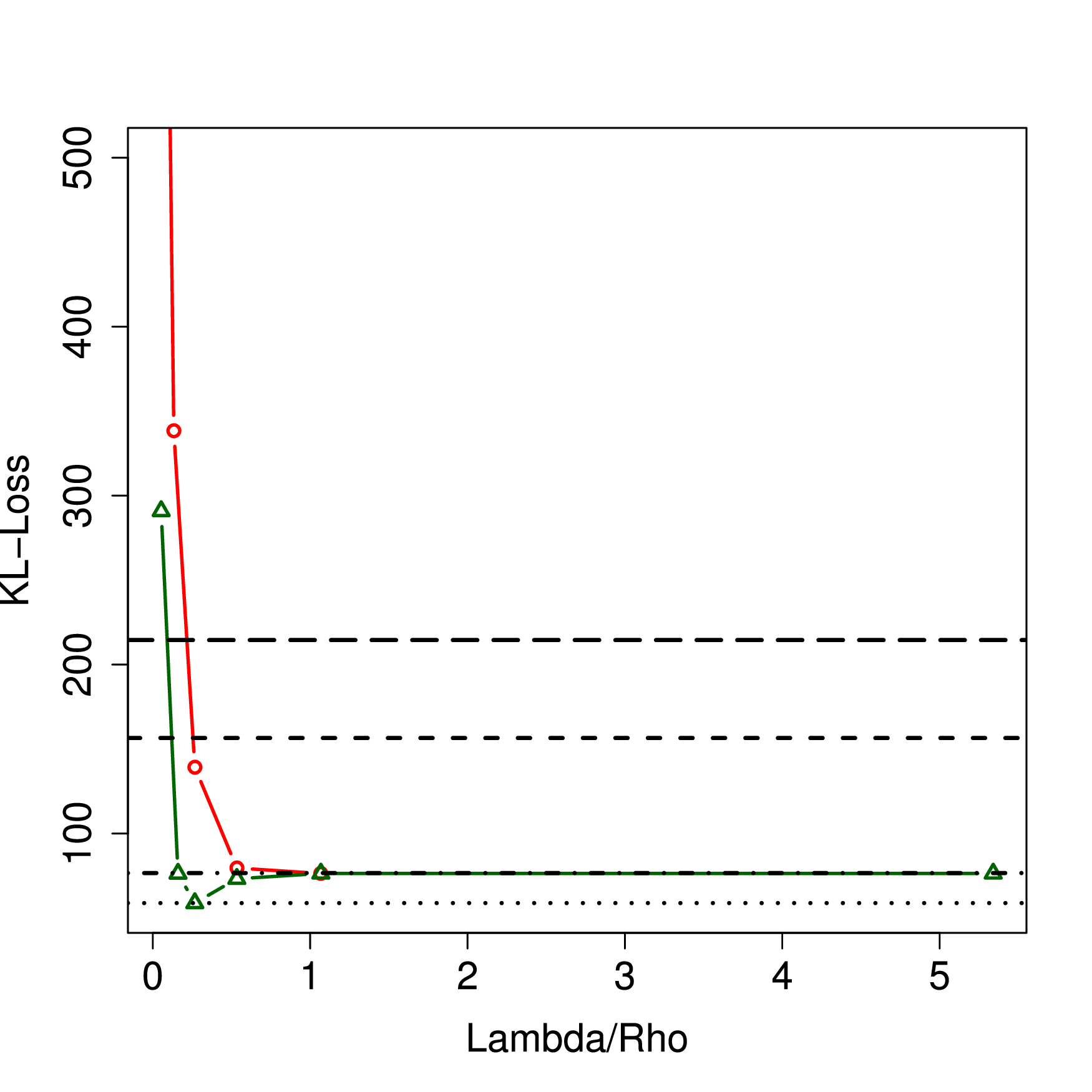}}
          \subfigure[$n=80$]{\includegraphics[trim=0cm 0cm 0cm 1cm ,clip=TRUE ,scale=0.3]{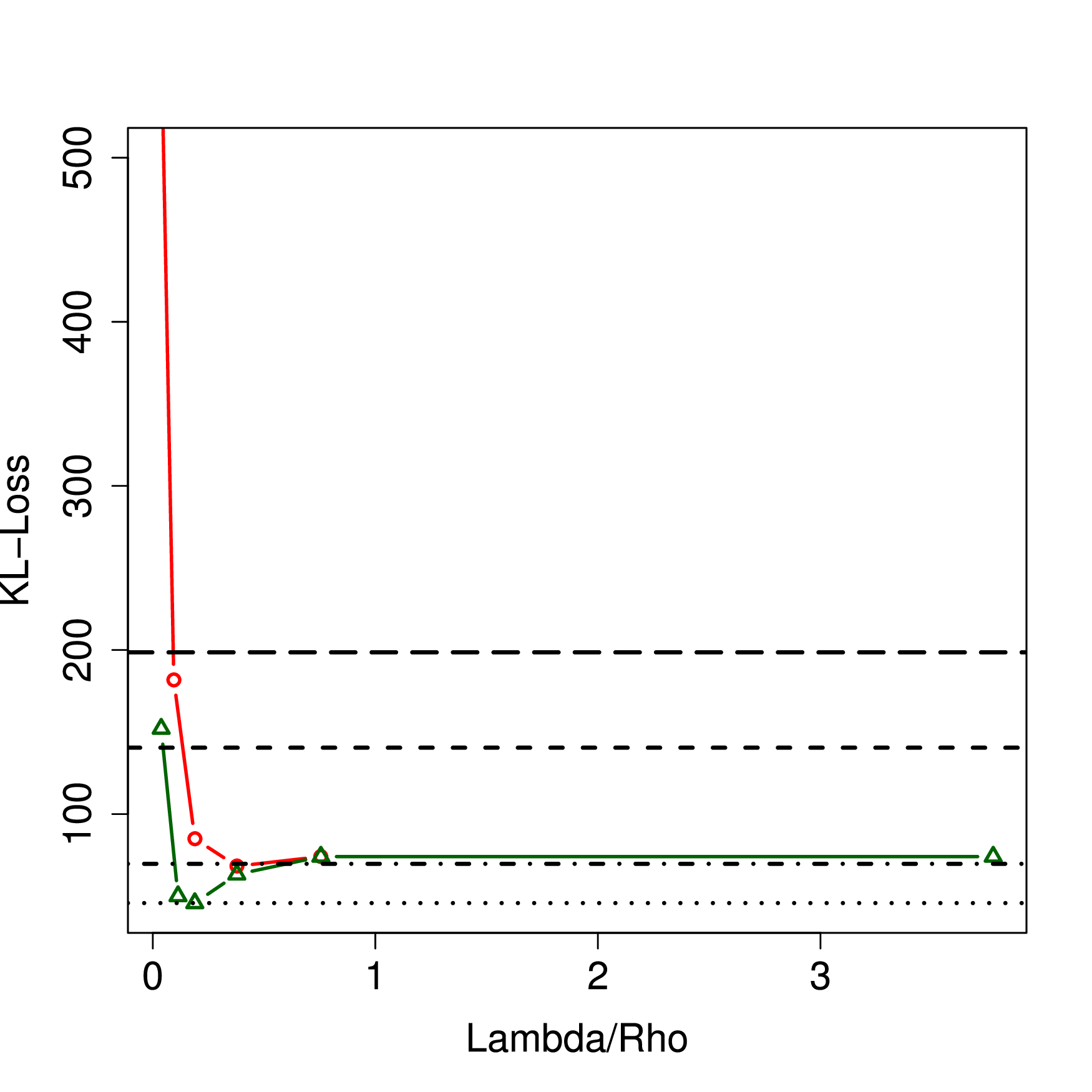}}
          \subfigure[$n=320$]{\includegraphics[trim=0cm 0cm 0cm 1cm ,clip=TRUE ,scale=0.3]{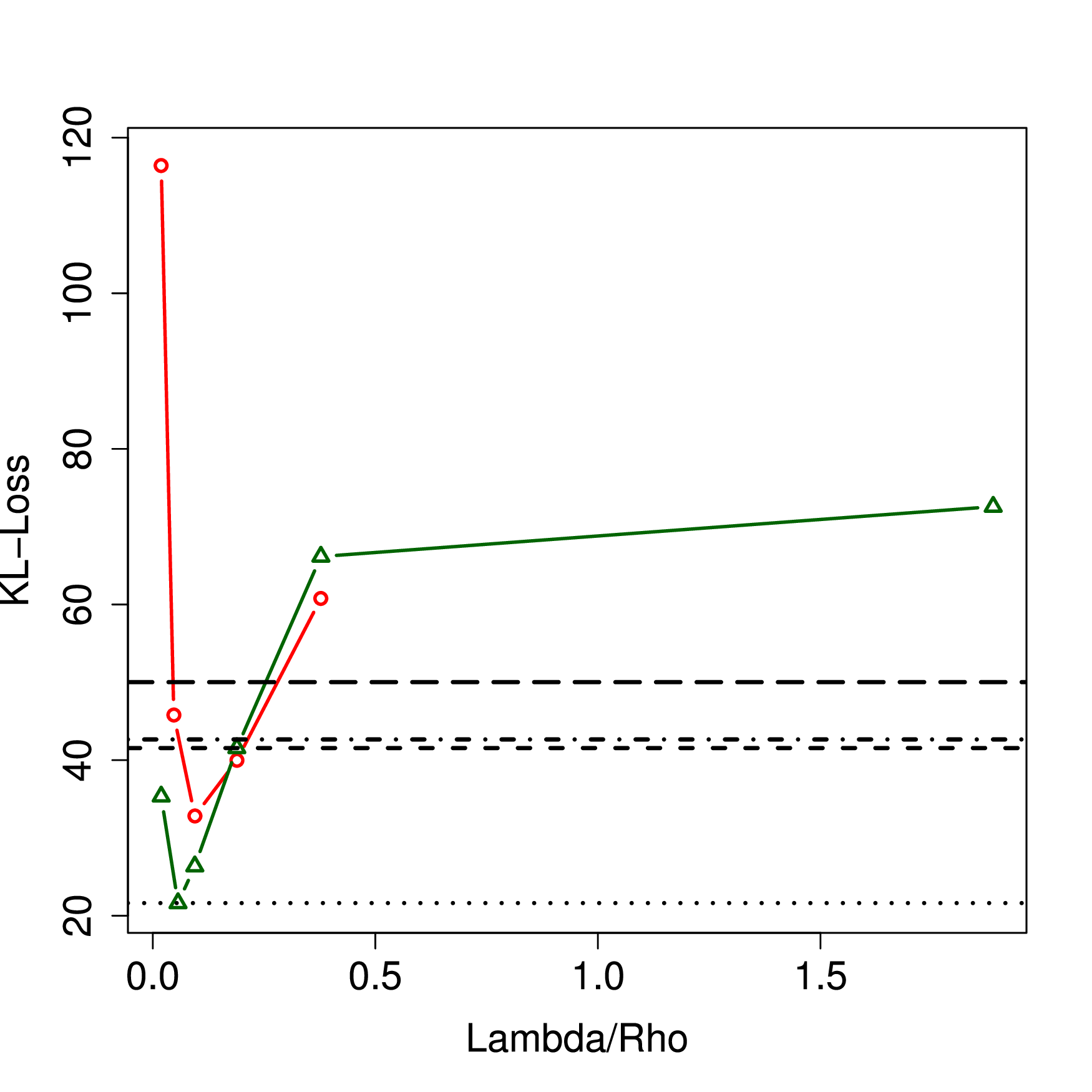}}
}
        \caption{Plots for model $\Theta^{(3)}_0$. The triangles (green)
          stand for the 
          GLasso and the circles (red) for our Gelato method with a
          reasonable value of $\tau$. The horizontal lines
          show the performances of the three techniques for cross-validated
          tuning parameters $\lambda$, $\tau$, $\rho$ and $\eta$. The
          dashed line stands for our Gelato method, the dotted one for the
          GLasso and the dash-dotted line for the Space technique. The
          additional dashed line with the longer dashes stands for the
          Gelato without thresholding. Lambda/Rho stands for $\lambda$ or
          $\rho$, respectively.}
        \label{Random1}
\end{figure}

\begin{figure}
        \centerline{ 
          \subfigure[$n=40$]{\includegraphics[trim=0cm 0cm 0cm 1cm ,clip=TRUE ,scale=0.3]{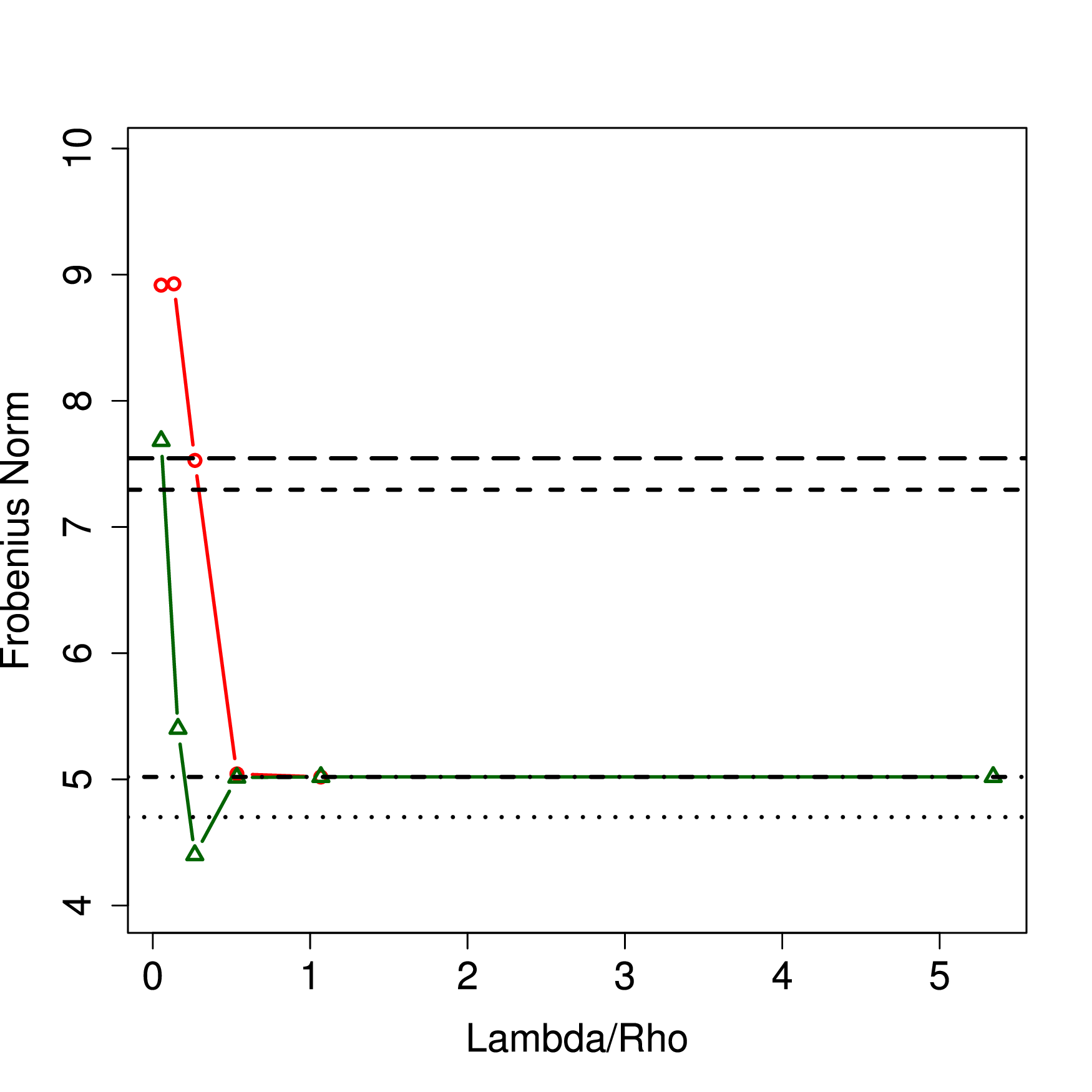}}
          \subfigure[$n=80$]{\includegraphics[trim=0cm 0cm 0cm 1cm ,clip=TRUE ,scale=0.3]{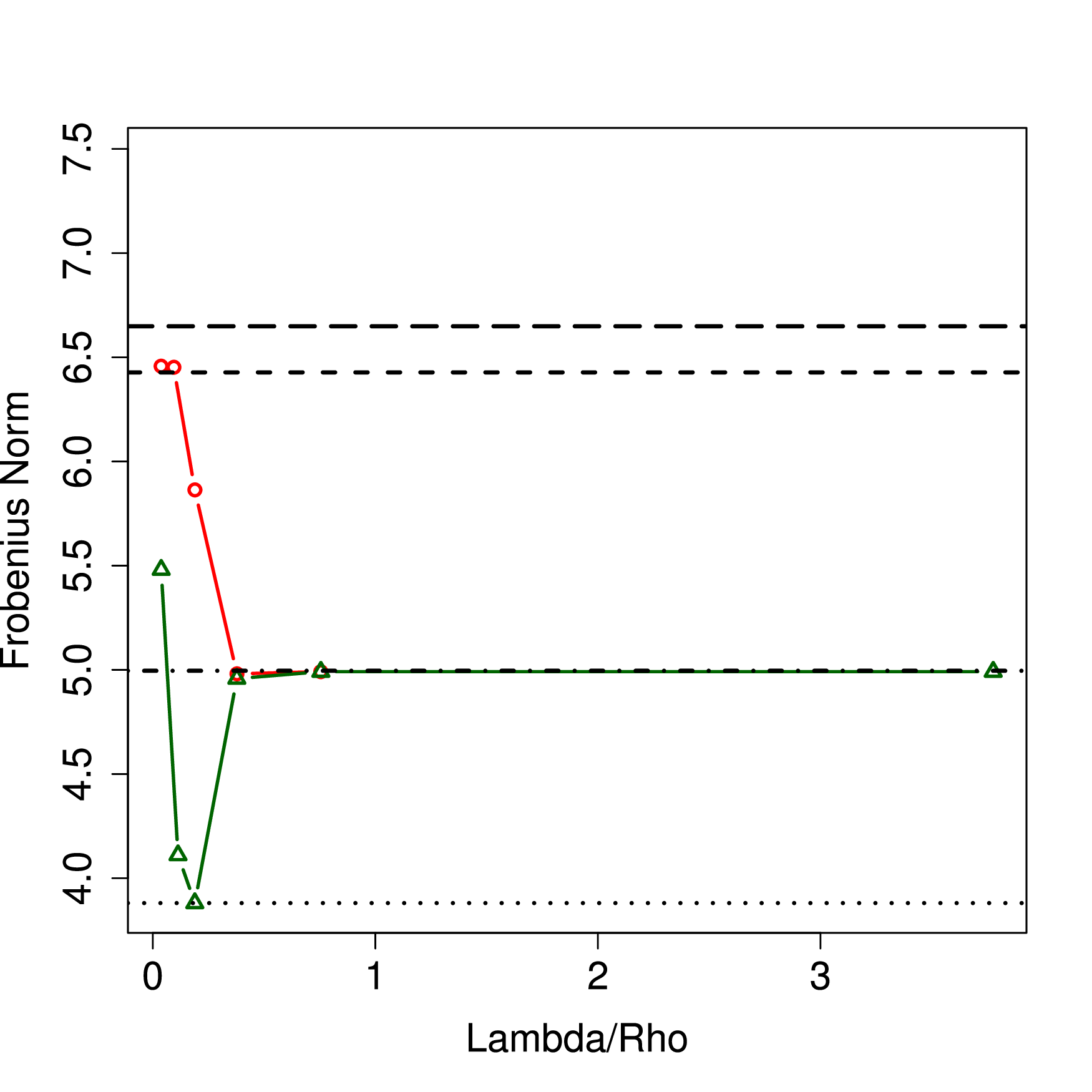}}
          \subfigure[$n=320$]{\includegraphics[trim=0cm 0cm 0cm 1cm ,clip=TRUE ,scale=0.3]{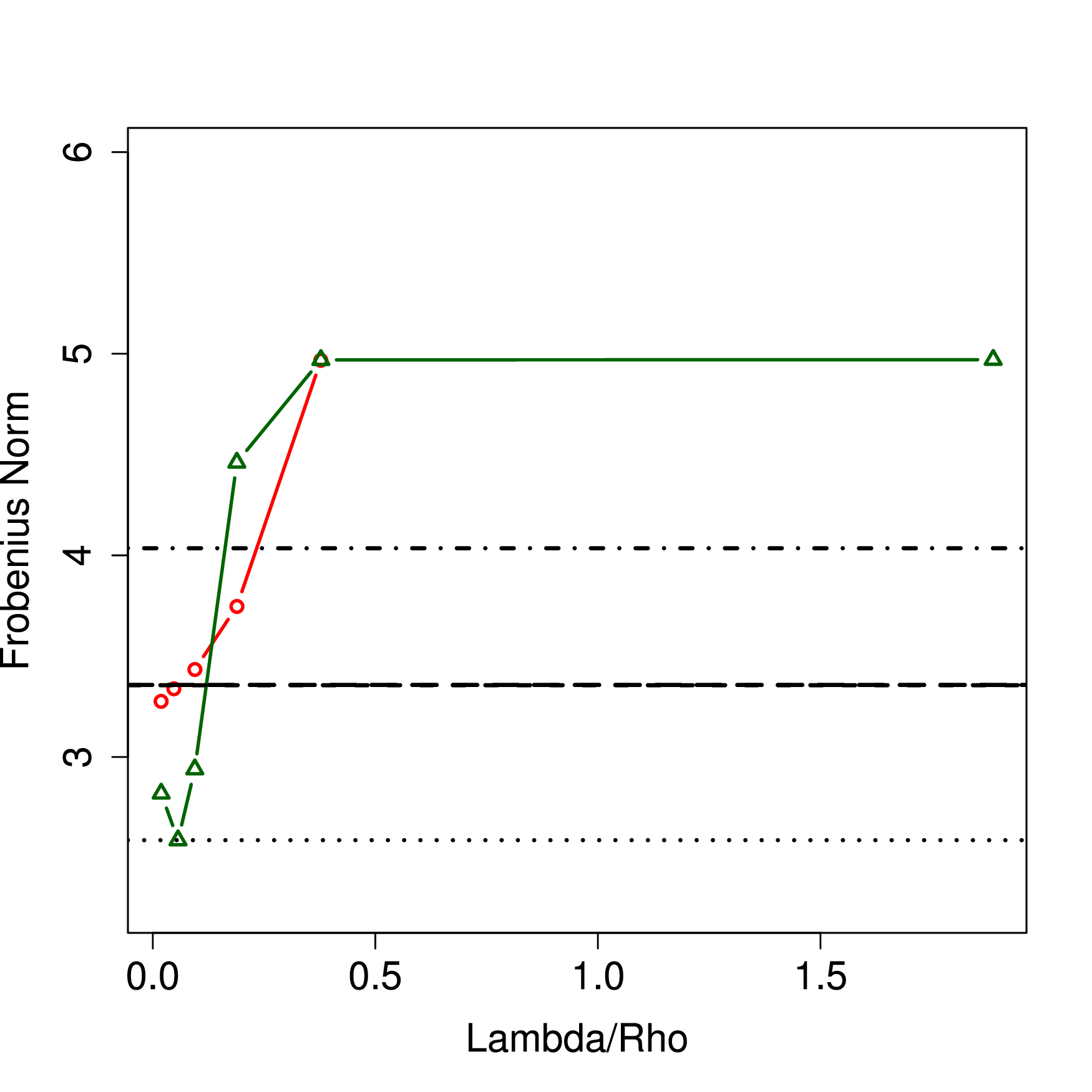}}
}

        \centerline{ 
          \subfigure[$n=40$]{\includegraphics[trim=0cm 0cm 0cm 1cm ,clip=TRUE ,scale=0.3]{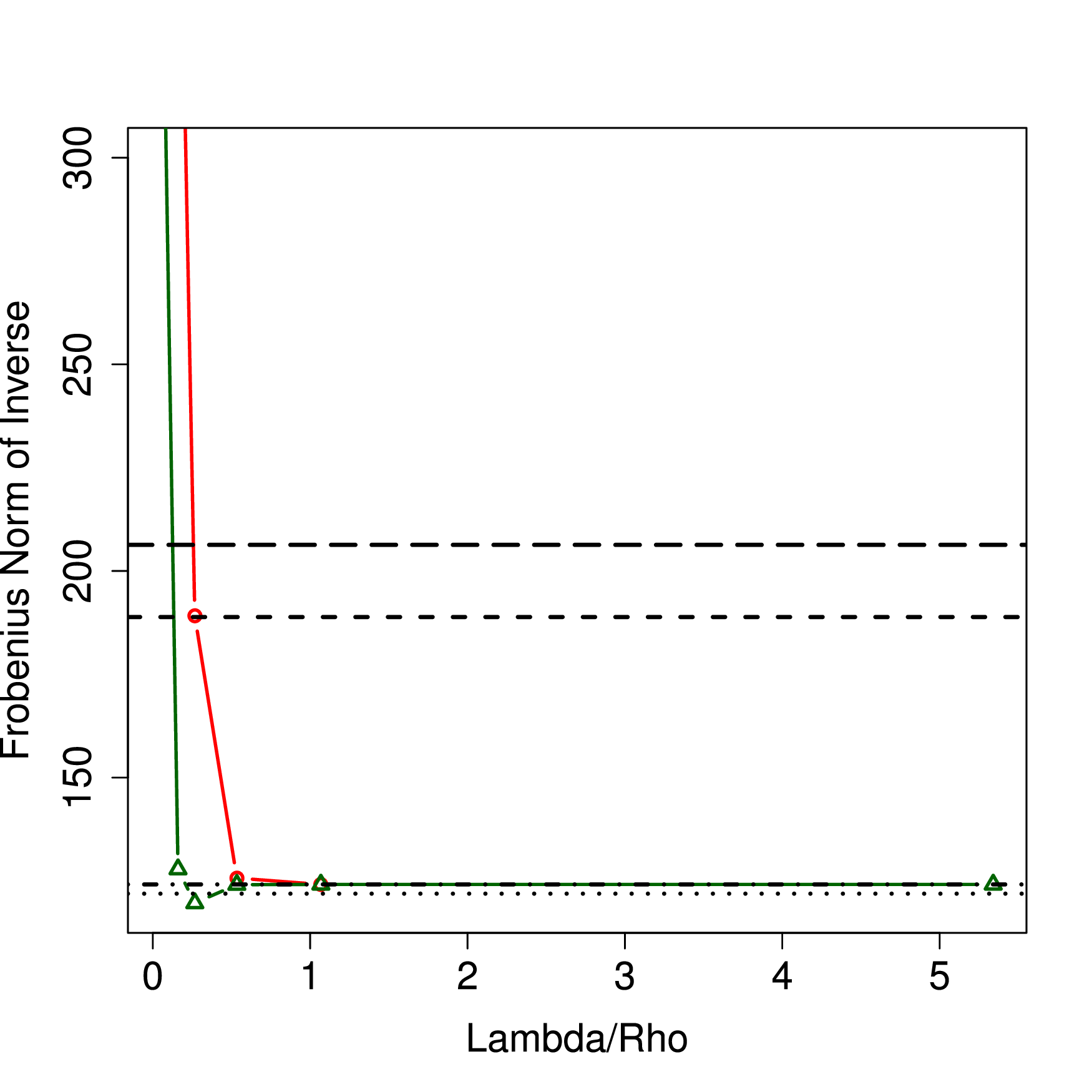}}
          \subfigure[$n=80$]{\includegraphics[trim=0cm 0cm 0cm 1cm ,clip=TRUE ,scale=0.3]{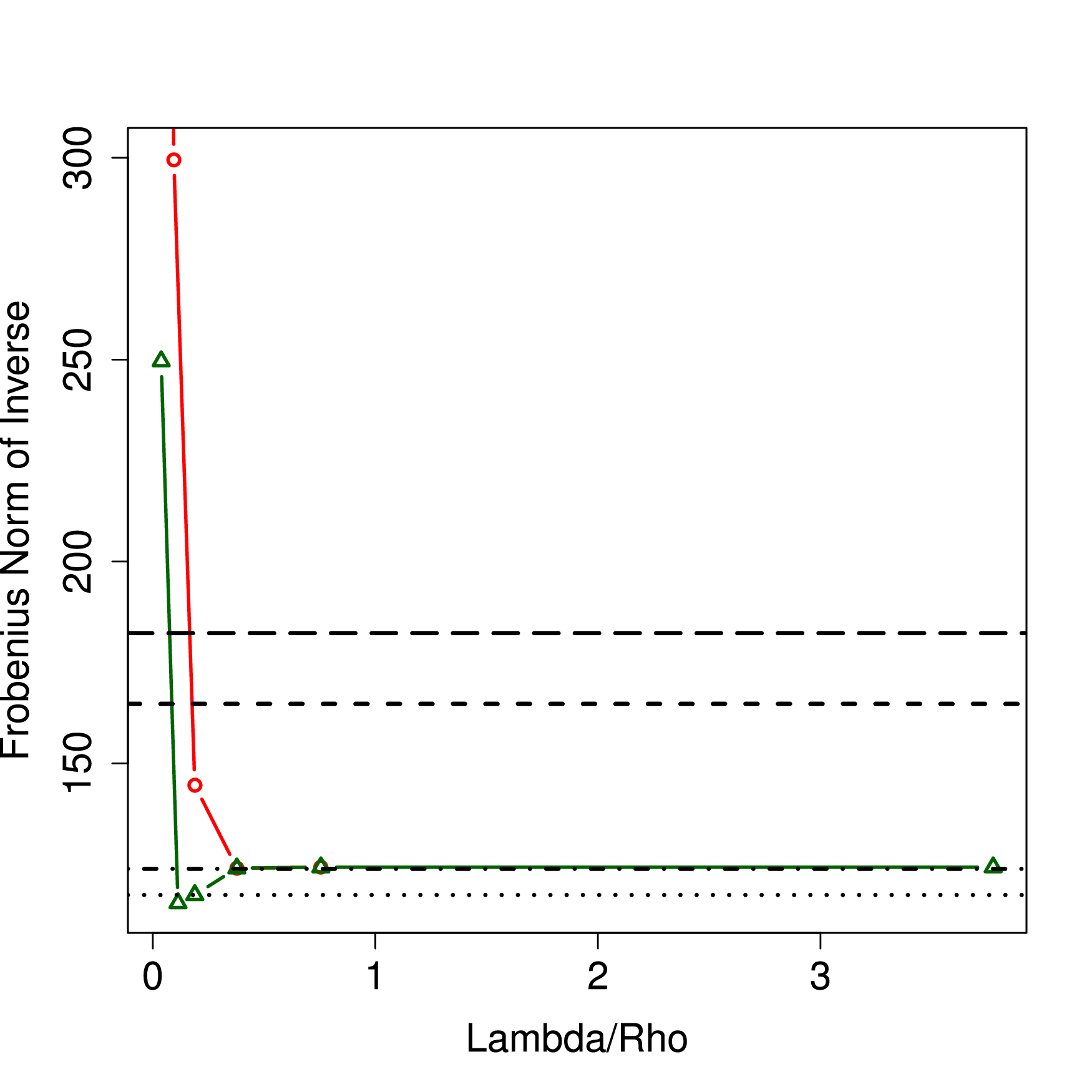}}
          \subfigure[$n=320$]{\includegraphics[trim=0cm 0cm 0cm 1cm ,clip=TRUE ,scale=0.3]{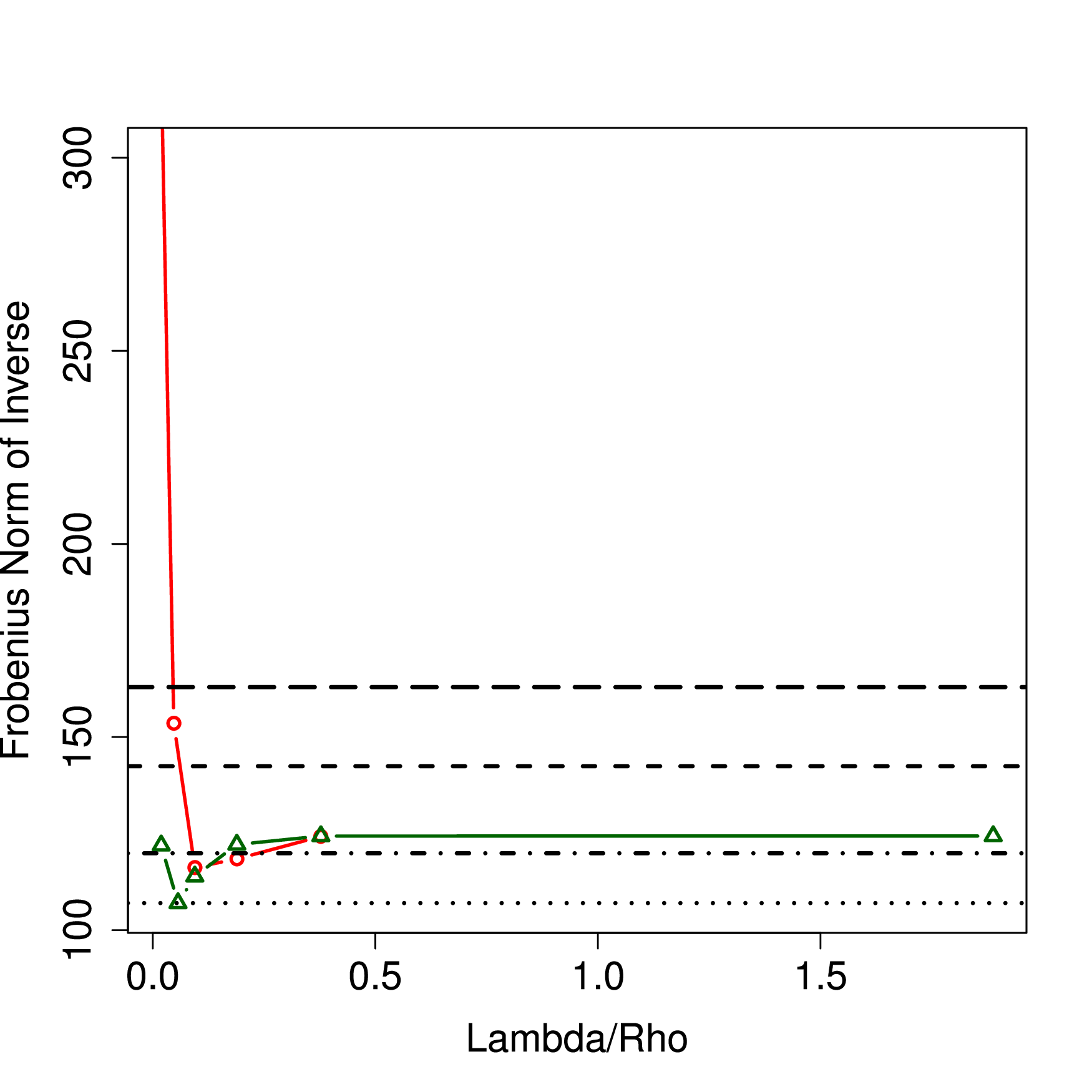}}
}

         \centerline{ 
          \subfigure[$n=40$]{\includegraphics[trim=0cm 0cm 0cm 1cm ,clip=TRUE ,scale=0.3]{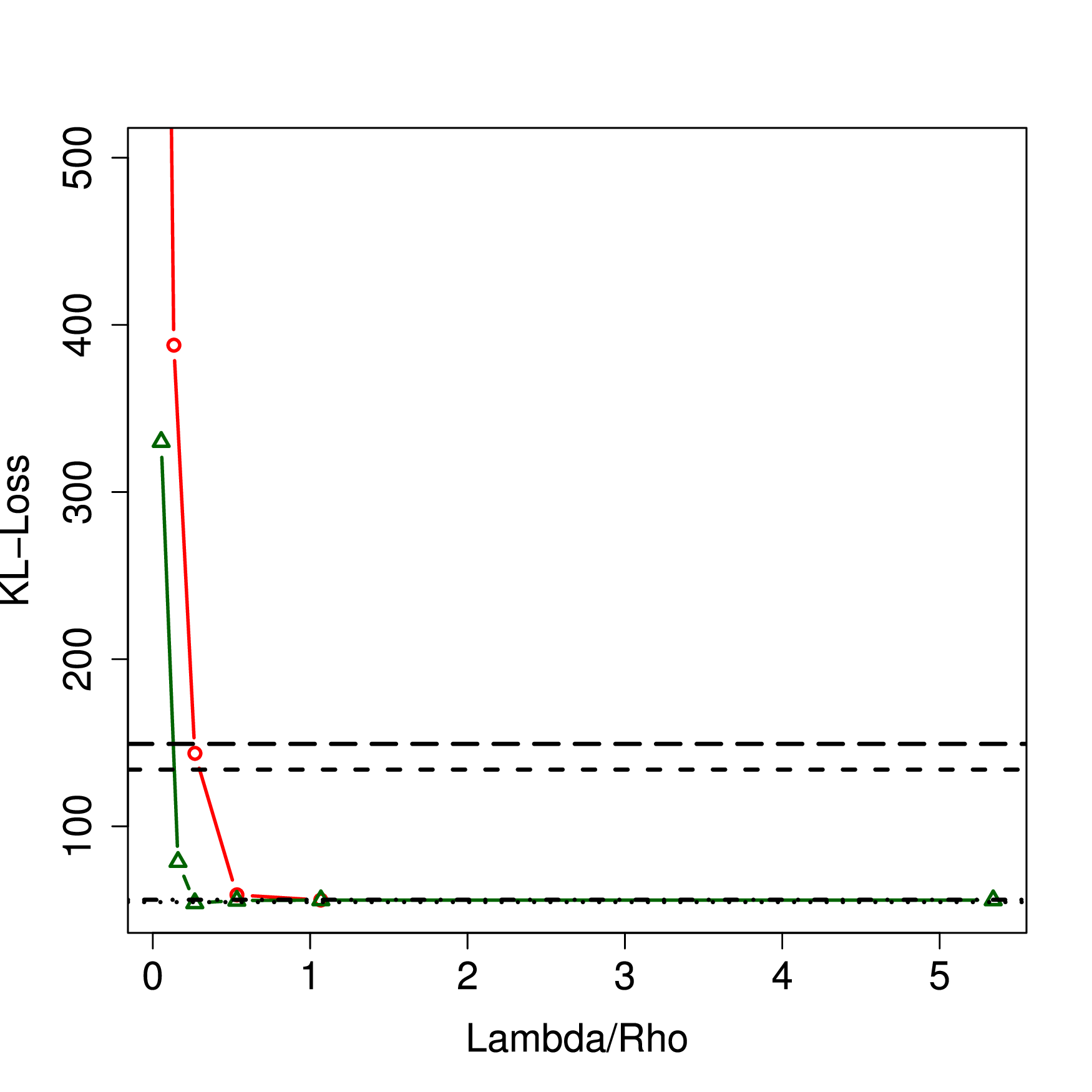}}
          \subfigure[$n=80$]{\includegraphics[trim=0cm 0cm 0cm 1cm ,clip=TRUE ,scale=0.3]{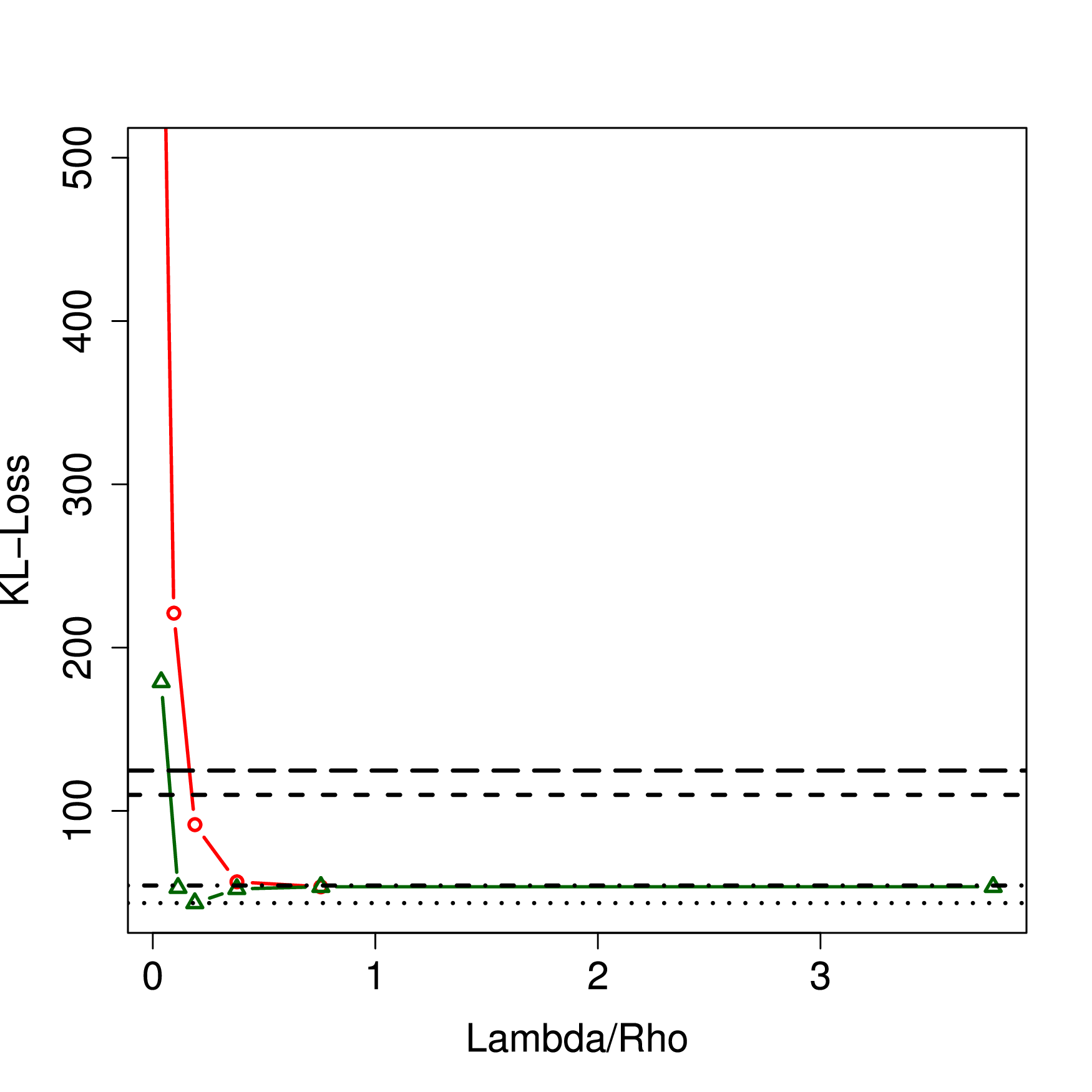}}
          \subfigure[$n=320$]{\includegraphics[trim=0cm 0cm 0cm 1cm ,clip=TRUE ,scale=0.3]{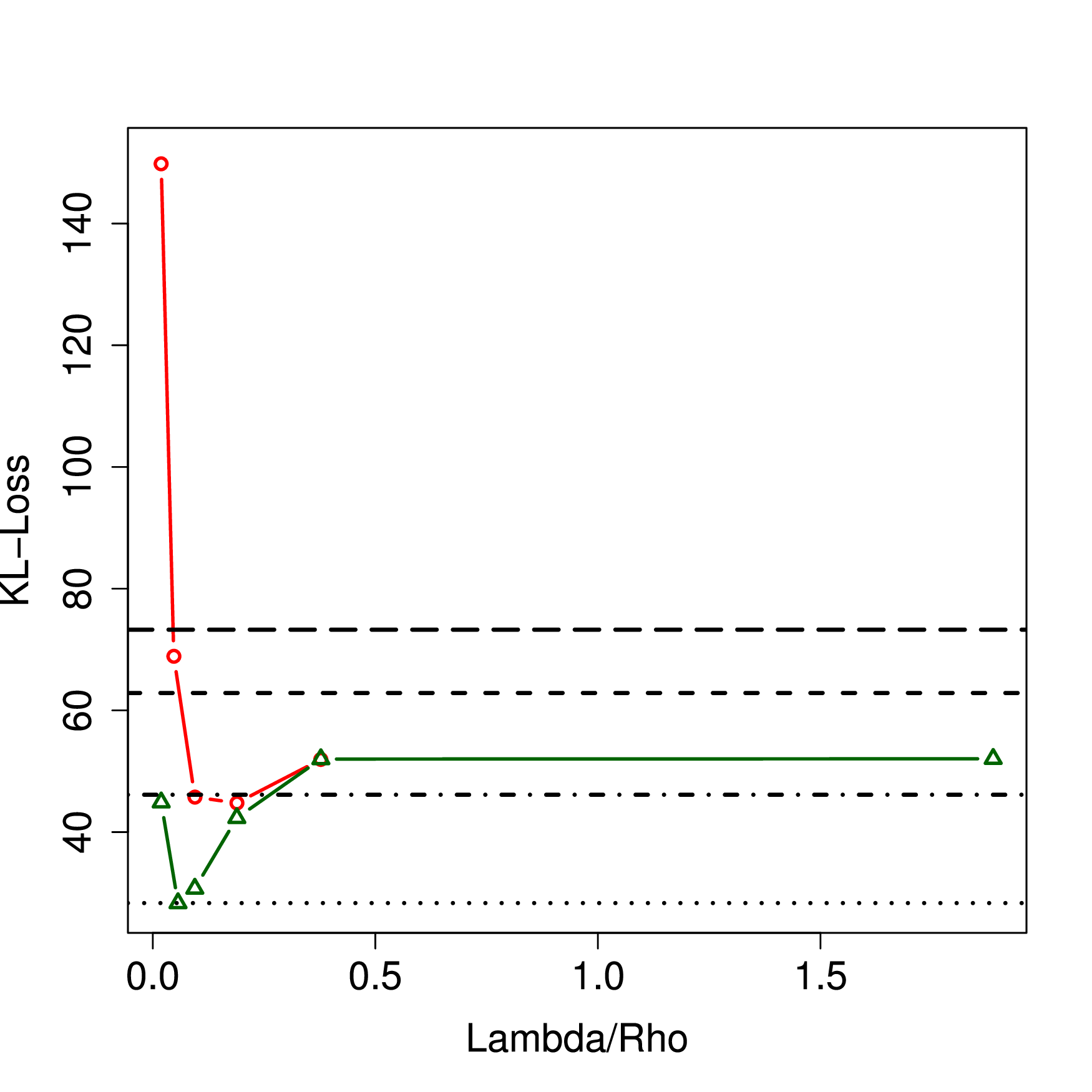}}
}
        \caption{Plots for model $\Theta^{(4)}_0$. The triangles (green)
          stand for the 
          GLasso and the circles (red) for our Gelato method with a
          reasonable value of $\tau$. The horizontal lines
          show the performances of the three techniques for cross-validated
          tuning parameters $\lambda$, $\tau$, $\rho$ and $\eta$. The
          dashed line stands for our Gelato method, the dotted one for the
          GLasso and the dash-dotted line for the Space technique. The
          additional dashed line with the longer dashes stands for the
          Gelato without thresholding. Lambda/Rho stands for $\lambda$ or
          $\rho$, respectively.}
        \label{Random2}
\end{figure} 
For this model we also consider two different matrices, which differ in
sparsity. For the sparser matrix $\Theta^{(3)}_0$ we set the probability
$\pi$ to $0.1$. That is, we have an off diagonal entry in
$\Theta^{(3)}$ of 0.5 with probability $\pi=0.1$ and an entry of 0
with probability $0.9$. In the case of the second matrix $\Theta^{(4)}_0$ we
set $\pi$ to $0.5$ which provides us with a denser concentration matrix. The
simulation results for the two performance measures are given in Figure
\ref{Random1} and \ref{Random2}.

From Figures \ref{Random1} and \ref{Random2} we see that
GLasso performs better than Gelato with respect to
$\|\hat{\Theta}_n-\Theta_0\|_F$ and the Kullback Leibler divergence in both
the sparse and the dense simulation setting. If we consider
$\|\hat{\Sigma}_n-\Sigma_0\|_F$, Gelato seems to keep up with GLasso to some
degree. For the Space method we have a similar situation to the one with
GLasso. The Space method outperforms Gelato for
$\|\hat{\Theta}_n-\Theta_0\|_F$ and $D_\text{KL}(\Sigma_0 \|
\hat{\Sigma}_n)$ but for $\|\hat{\Sigma}_n-\Sigma_0\|_F$, Gelato somewhat
keeps up with Space.

\subsubsection{The exponential decay model}

In this simulation setting we only have one version of the concentration
matrix $\Theta^{(5)}_0$. The entries of $\Theta^{(5)}_0$ are generated by
$\theta^{(5)}_{0,ij}=\exp(-2|i-j|)$. Thus, $\Sigma_0$ is a banded and sparse matrix.

Figure \ref{EXP}  shows the results of the simulation.
We find that all three
methods show equal performances in both the Frobenius norm and the 
Kullback Leibler divergence. This is interesting because even with a sparse
approximation of $\Theta_0$ (with GLasso or Gelato), we obtain competitive
performance for (inverse) covariance estimation.
\begin{figure}
        \centerline{ 
          \subfigure[$n=40$]{\includegraphics[trim=0cm 0cm 0cm 1cm ,clip=TRUE ,scale=0.3]{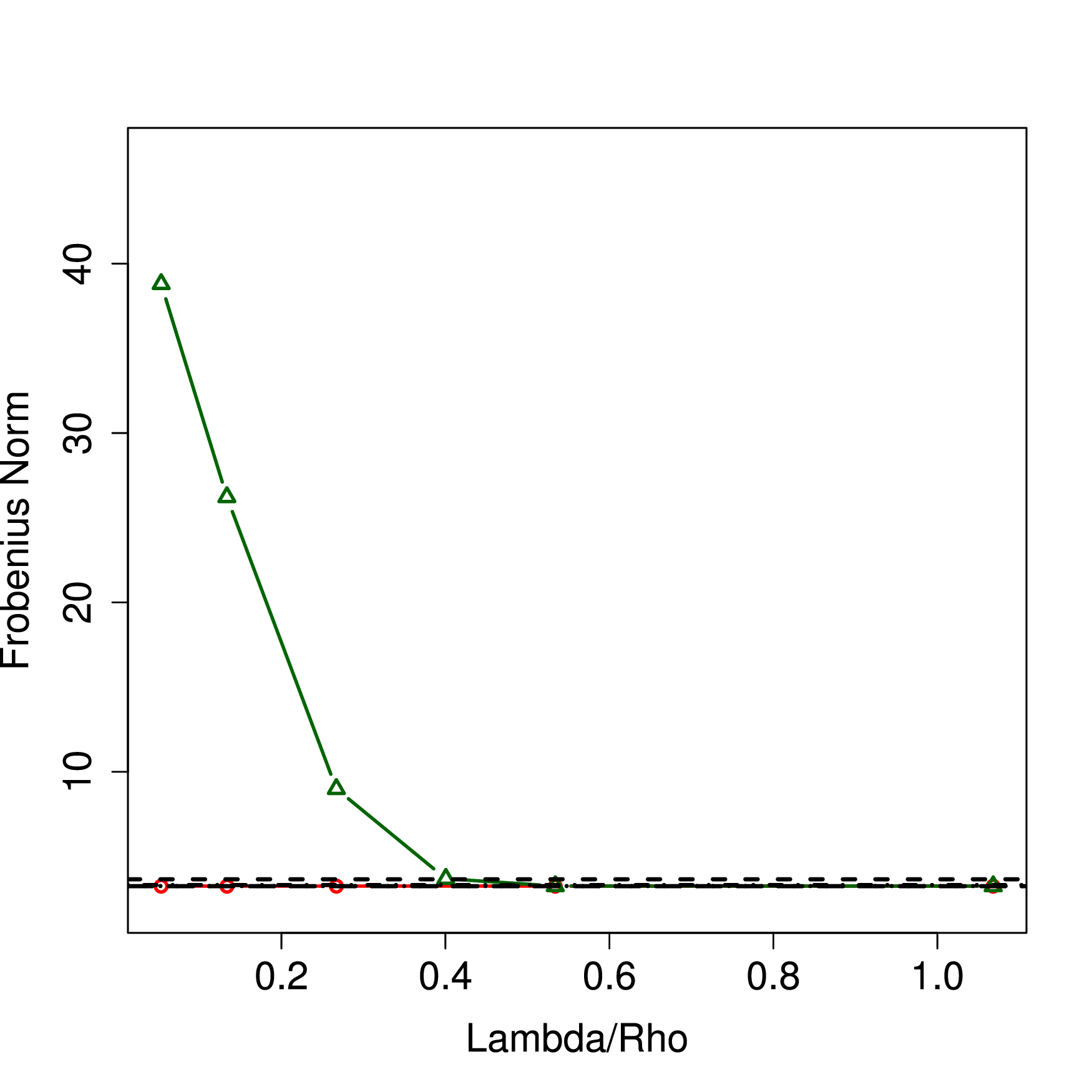}}
          \subfigure[$n=80$]{\includegraphics[trim=0cm 0cm 0cm 1cm ,clip=TRUE ,scale=0.3]{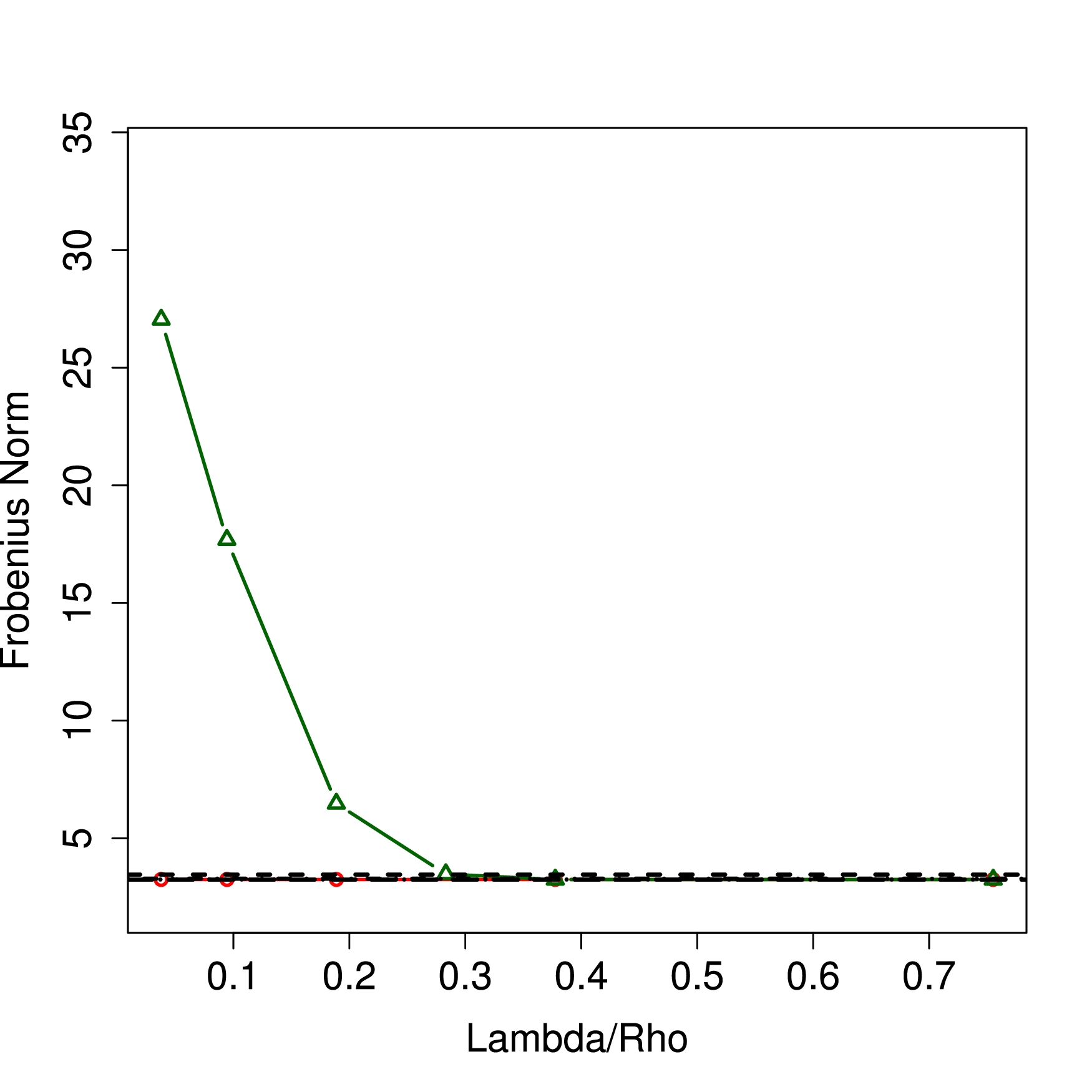}}
          \subfigure[$n=320$]{\includegraphics[trim=0cm 0cm 0cm 1cm ,clip=TRUE ,scale=0.3]{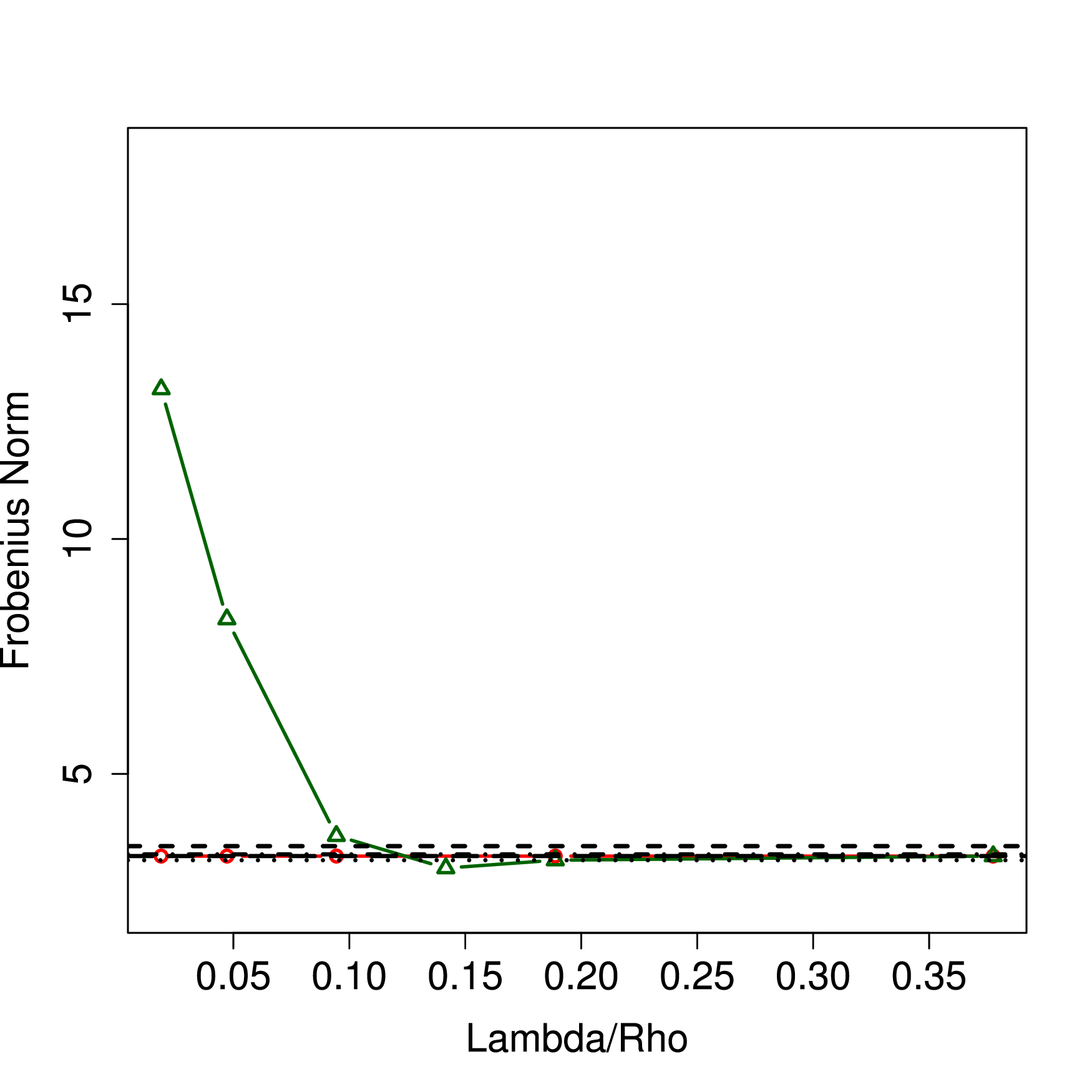}}
}

        \centerline{ 
          \subfigure[$n=40$]{\includegraphics[trim=0cm 0cm 0cm 1cm ,clip=TRUE ,scale=0.3]{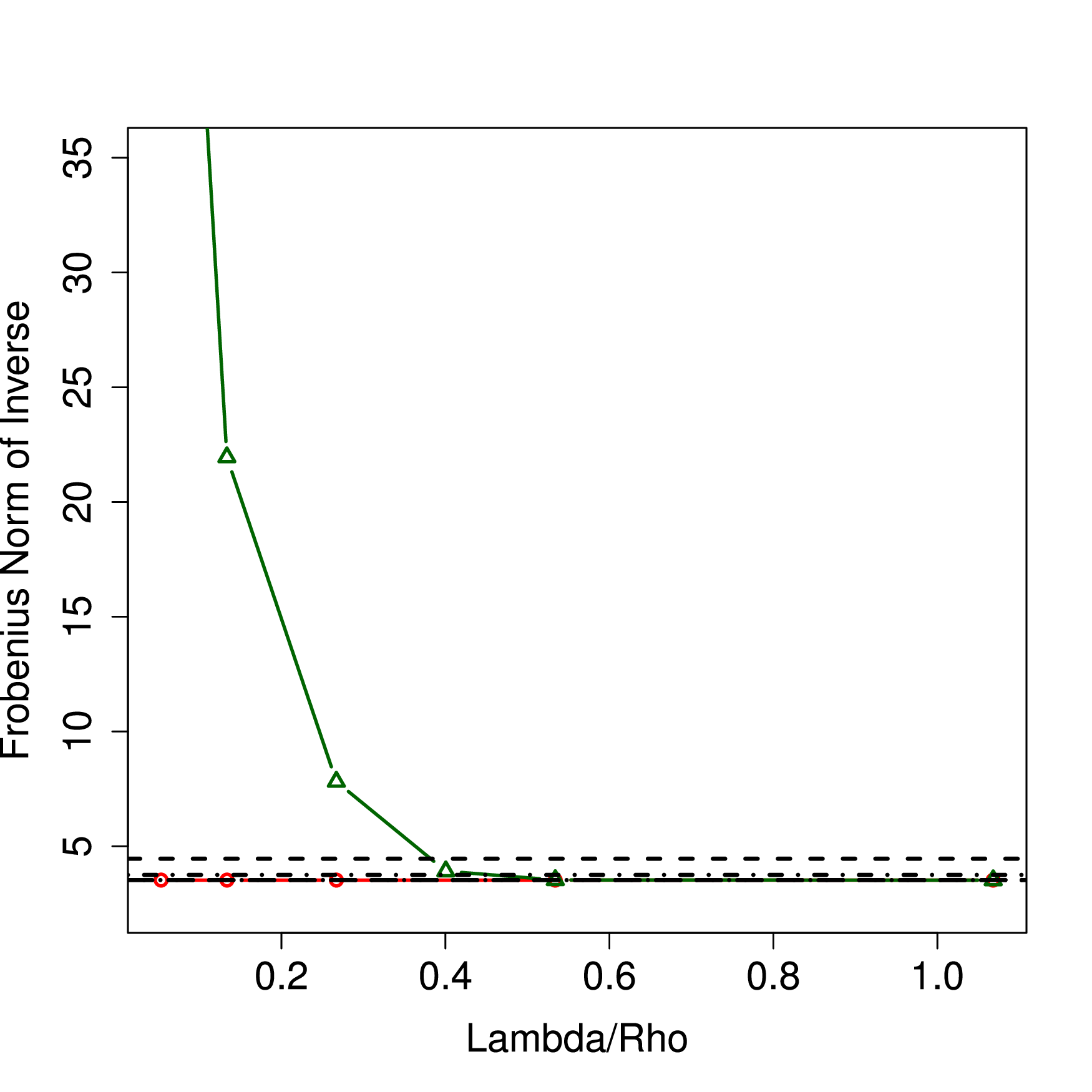}}
          \subfigure[$n=80$]{\includegraphics[trim=0cm 0cm 0cm 1cm ,clip=TRUE ,scale=0.3]{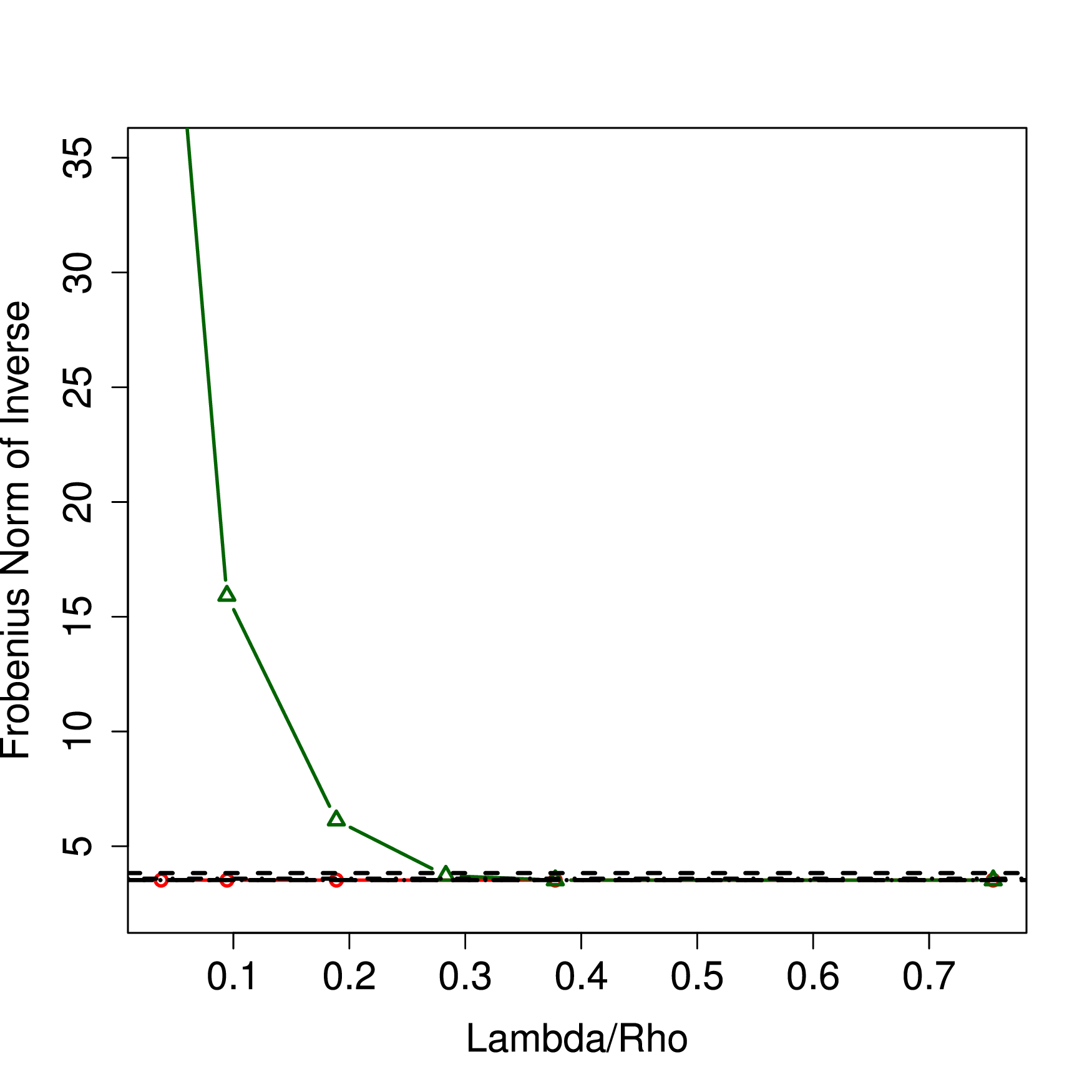}}
          \subfigure[$n=320$]{\includegraphics[trim=0cm 0cm 0cm 1cm ,clip=TRUE ,scale=0.3]{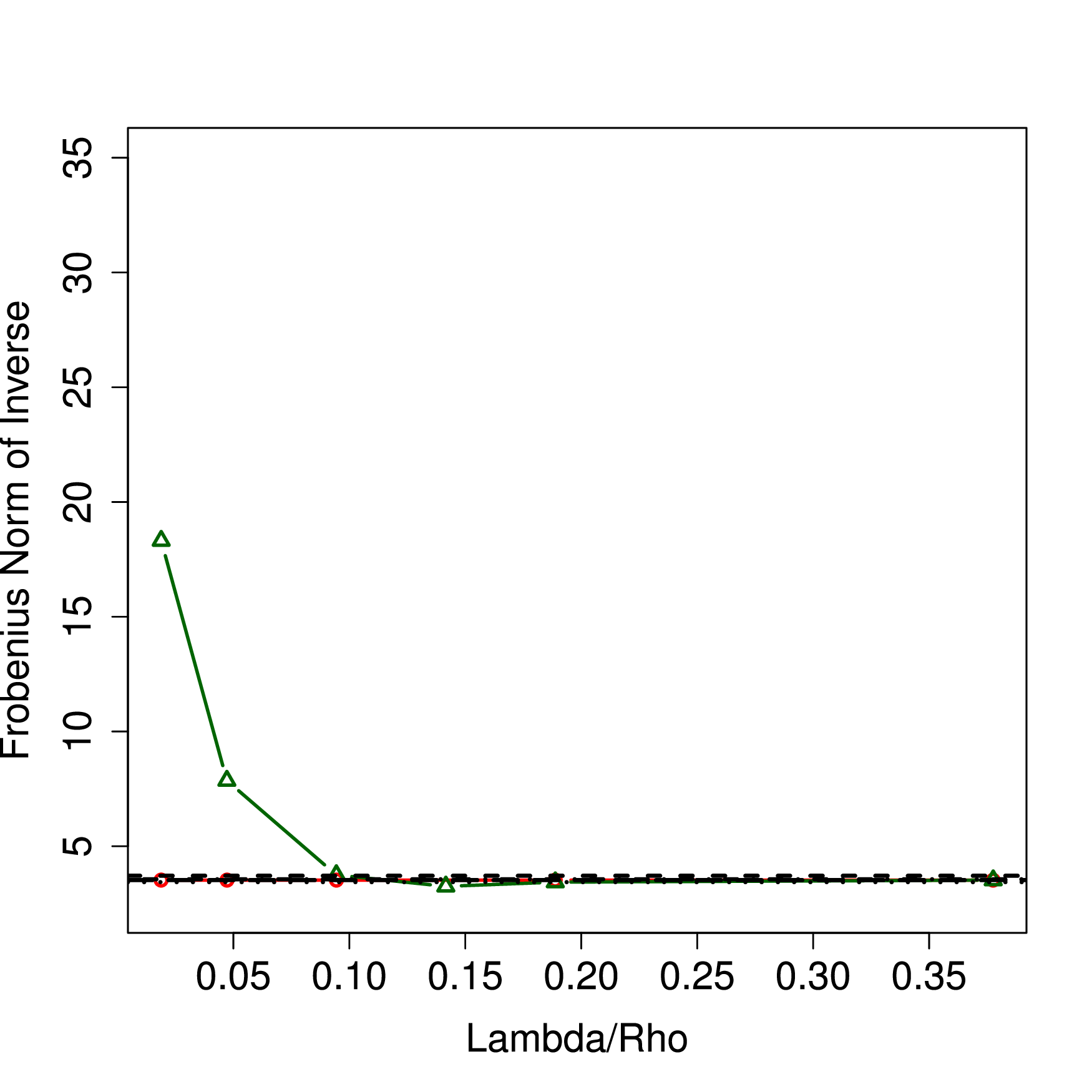}}
}

         \centerline{ 
          \subfigure[$n=40$]{\includegraphics[trim=0cm 0cm 0cm 1cm ,clip=TRUE ,scale=0.3]{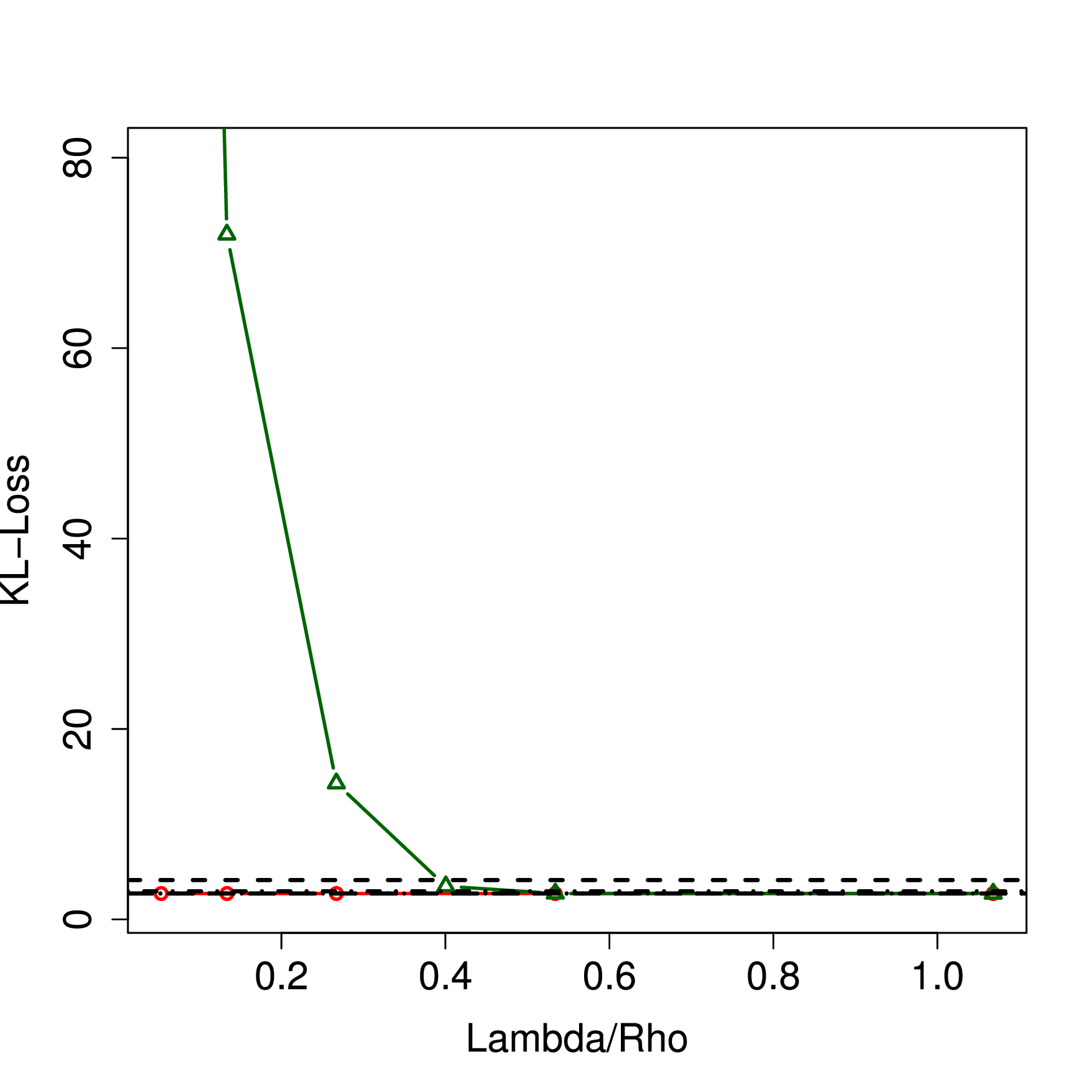}}
          \subfigure[$n=80$]{\includegraphics[trim=0cm 0cm 0cm 1cm ,clip=TRUE ,scale=0.3]{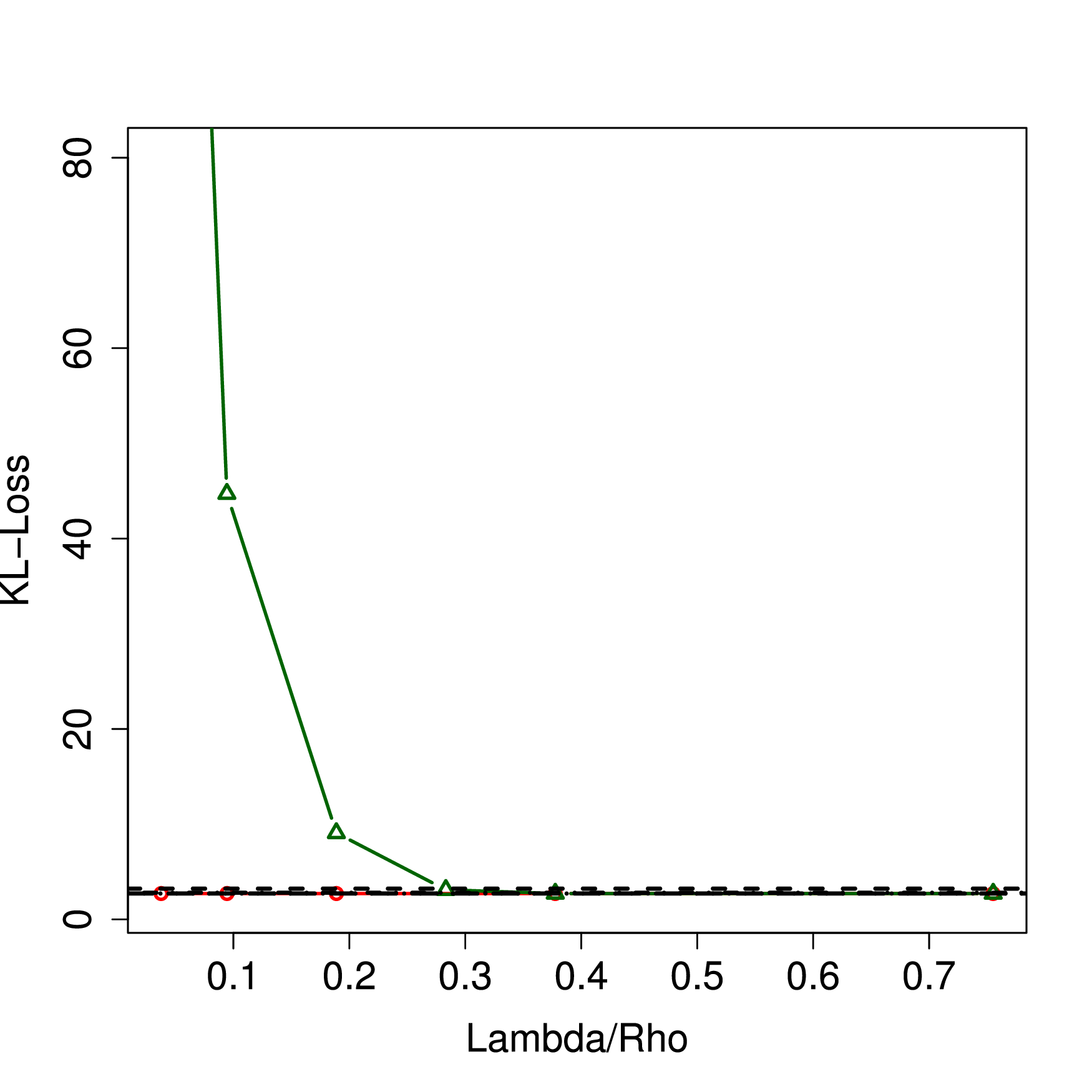}}
          \subfigure[$n=320$]{\includegraphics[trim=0cm 0cm 0cm 1cm ,clip=TRUE ,scale=0.3]{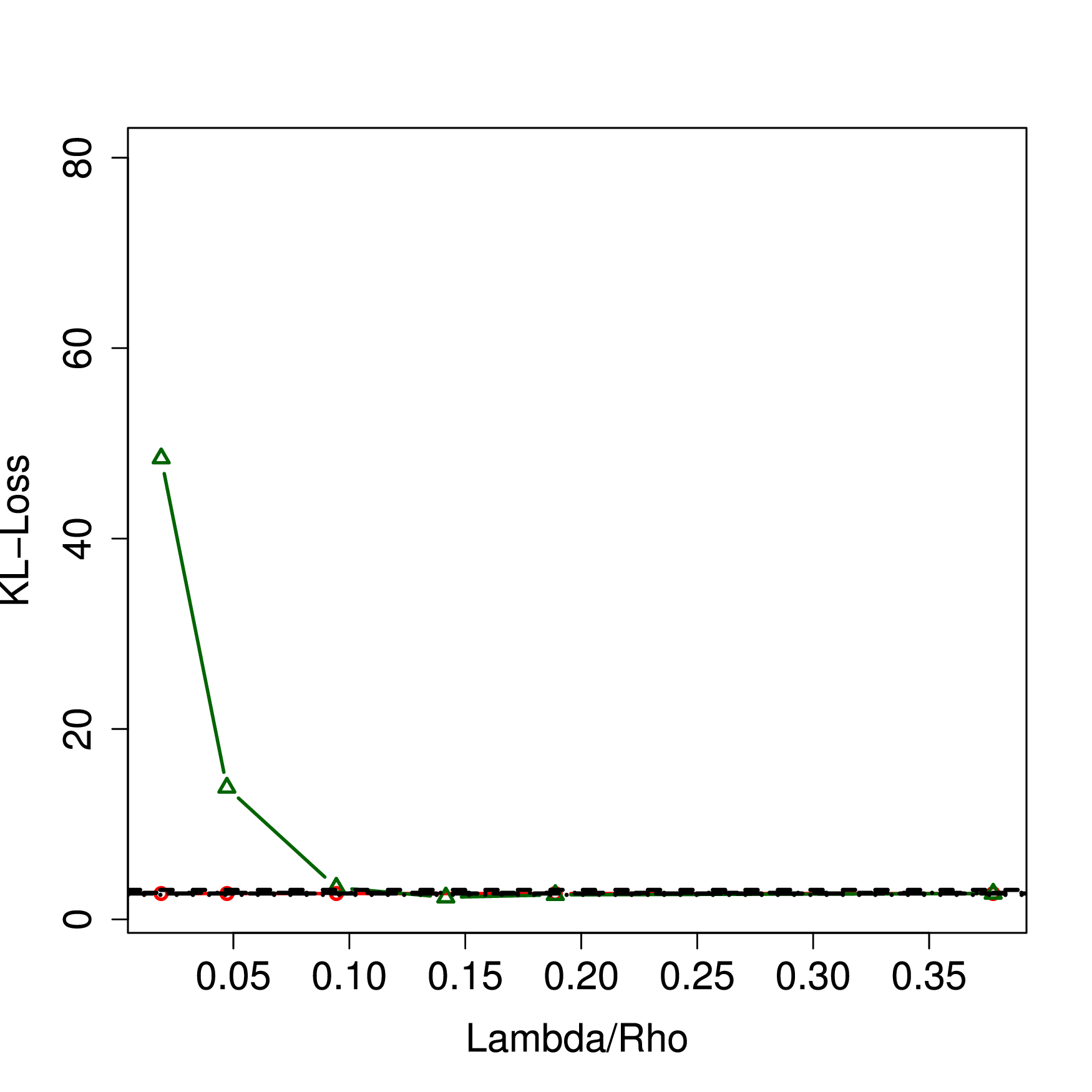}}
}
        \caption{Plots for model $\Theta^{(5)}_0$. The triangles (green)
          stand for the 
          GLasso and the circles (red) for our Gelato method with a
          reasonable value of $\tau$. The horizontal lines
          show the performances of the three techniques for cross-validated
          tuning parameters $\lambda$, $\tau$, $\rho$ and $\eta$. The
          dashed line stands for our Gelato method, the dotted one for the
          GLasso and the dash-dotted line for the Space technique. The
          additional dashed line with the longer dashes stands for the
          Gelato without thresholding. Lambda/Rho stands for $\lambda$ or
          $\rho$, respectively.}
        \label{EXP}
\end{figure}

\subsubsection{Summary}

Overall we can say that the performance of the methods depend on the
model. For the models $\Sigma^{(1)}_{0}$ and $\Sigma^{(2)}_{0}$ the Gelato
method performs best. In case of the models $\Theta^{(3)}_{0}$ and
$\Theta^{(4)}_{0}$, Gelato gets outperformed by GLasso and the Space
method and for the model $\Theta^{(5)}_{0}$ none of the three methods has a
clear advantage. 
In Figures \ref{causal1} to \ref{Random2}, we see the advantage of Gelato with 
thresholding over the one without thresholding,
in particular, for the simulation settings
$\Sigma^{(1)}_{0}$, $\Sigma^{(2)}_{0}$ and $\Theta^{(3)}_{0}$. 
Thus thresholding is a useful feature of Gelato. 
 
\subsection{Application to real data}

\subsubsection{Isoprenoid gene pathway in Arabidobsis thaliana}

In this example we compare the two estimators on the isoprenoid
biosynthesis pathway data given in \cite{AWPB04}. 
Isoprenoids 
play various roles in plant and animal physiological processes and
as intermediates in the biological synthesis of other important molecules.
In plants they
serve numerous biochemical functions in processes such as photosynthesis,
regulation of growth and development.\\  
The data set  consists of $p=39$ isoprenoid genes for which we have $n=118$
gene expression patterns under various experimental conditions. In order to
compare the two techniques we compute the negative log-likelihood via
10-fold cross-validation for different values of $\lambda$, $\tau$ and

$\rho$. 
\begin{figure}
        \centerline{ 
          \subfigure[isoprenoid data]{\includegraphics[scale=0.4]{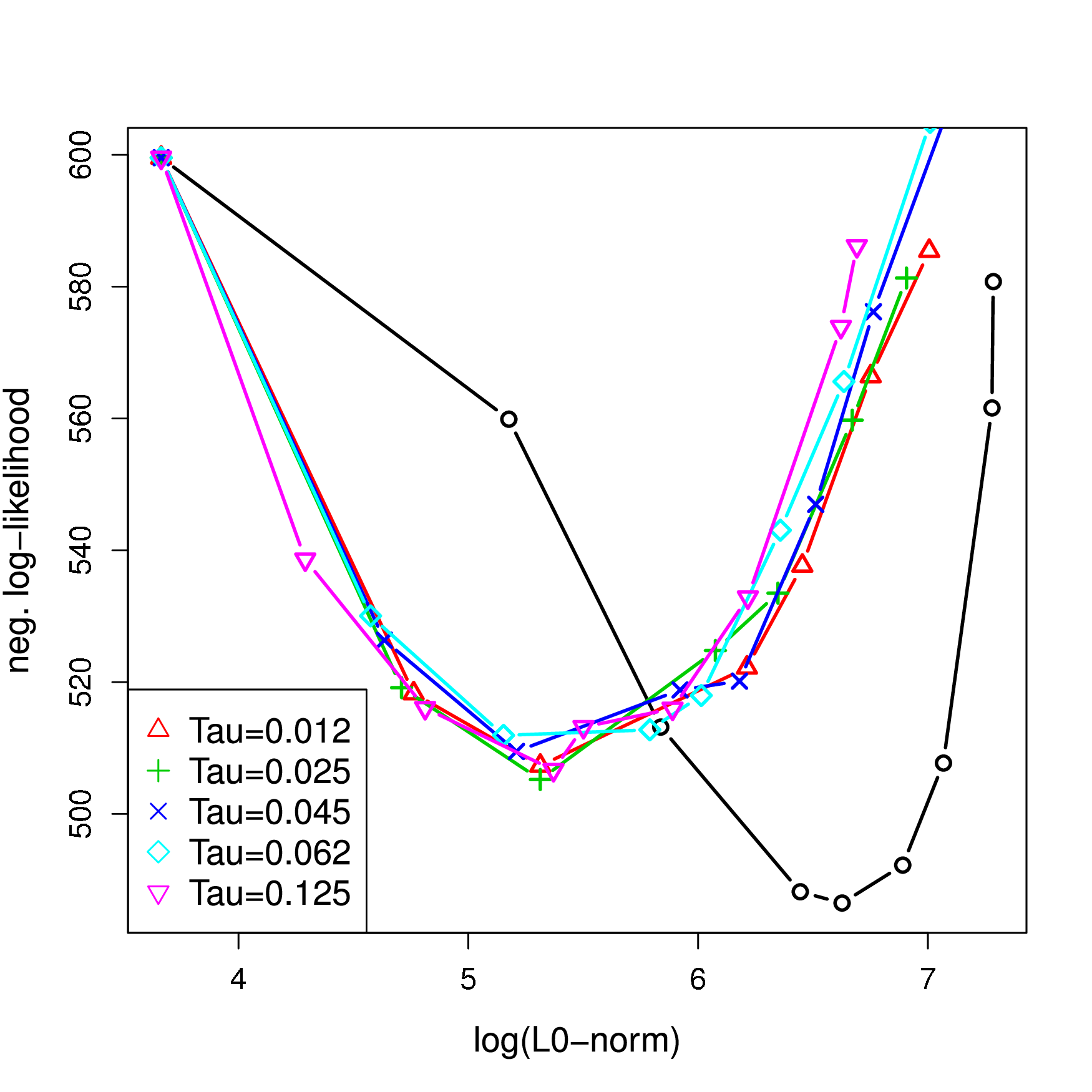}} 
          \subfigure[breast cancer data]{\includegraphics[scale=0.4]{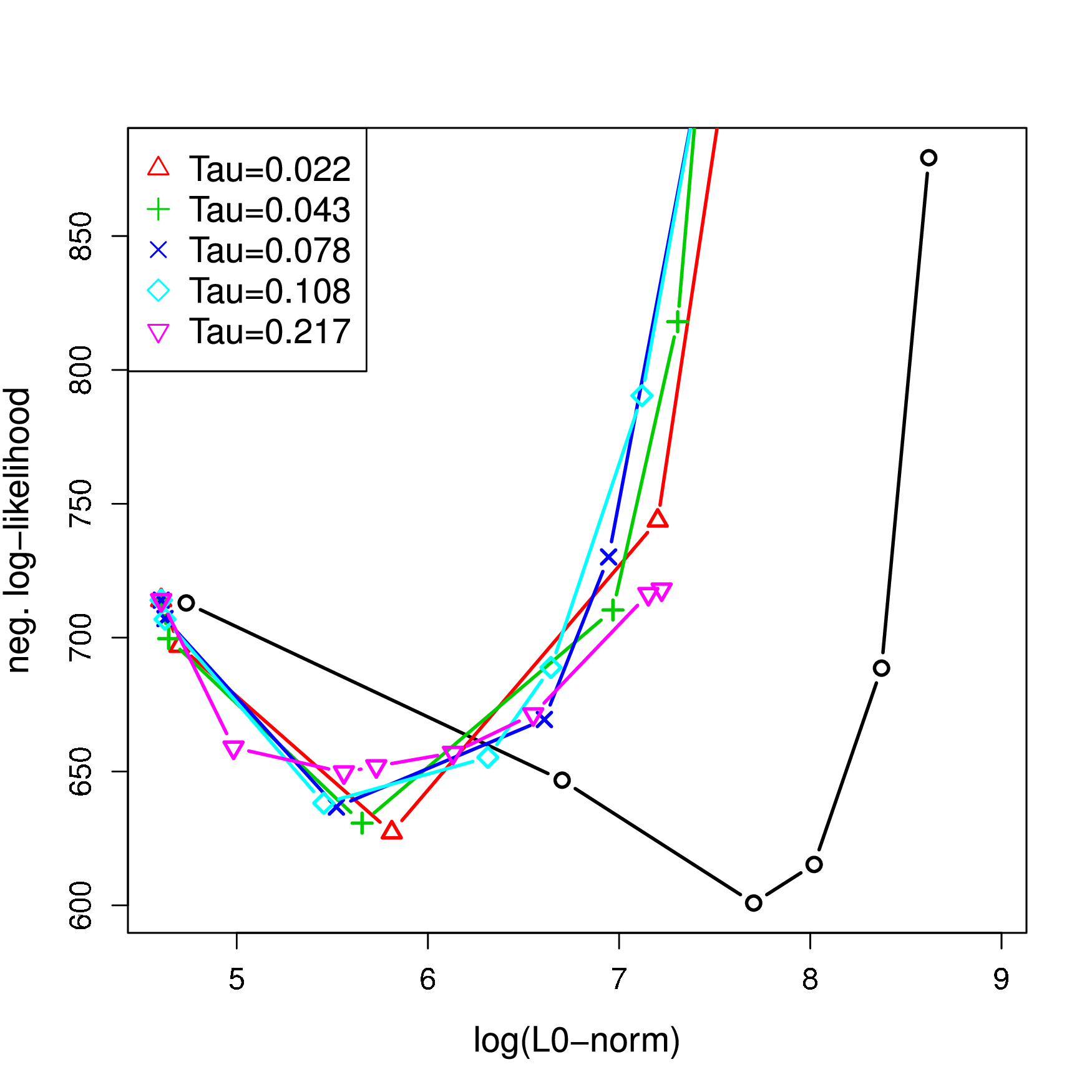}}}
        \caption{
Plots for the isoprenoid data from arabidopsis thaliana (a) and the human
breast cancer data (b). 10-fold cross-validation of negative log-likelihood against the logarithm
of the average number of non-zero entries of the estimated concentration
matrix $\hat{\Theta}_n$. The circles stand for the GLasso and the Gelato is
displayed for various values of $\tau$.
}
\label{willebreast}
\end{figure}
In Figure \ref{willebreast} we plot the cross-validated
negative log-likelihood against the logarithm of the average number of 
non-zero entries (logarithm of the $\ell_0$-norm) of the estimated 
concentration matrix
$\hat{\Theta}_n$. The logarithm of the $\ell_0$-norm reflects the sparsity of
the matrix $\hat{\Theta}_n$ and therefore the figures show the performance of the
estimators for different levels of sparsity. The plots do not allow for a clear 
conclusion. The GLasso performs slightly
better when allowing for a rather dense fit. On the other hand, when
requiring a sparse fit, the Gelato performs better. 

\subsubsection{Clinical status of human breast cancer}

As a second example, we compare the two methods on the breast cancer
dataset from \cite{WEST01}. The tumor samples were selected from the Duke Breast
Cancer SPORE tissue bank. The data consists of $p=7129$ genes with $n=49$
breast tumor samples. For the analysis we use the 100 variables with the
largest sample variance. As before, we compute the negative log-likelihood
via 10-fold cross-validation. Figure \ref{willebreast} shows the results. 
In this real data example the interpretation of the plots is similar as for
the arabidopsis dataset. For dense fits, GLasso is better while Gelato has
an advantage when requiring a sparse fit.

\section{Conclusions}
We propose and analyze the Gelato estimator. Its advantage is that it
automatically yields a positive definite covariance matrix $\hat{\Sigma}_n$,
it enjoys fast convergence rate with respect to the operator and
Frobenius norm of $\hat{\Sigma}_n-\Sigma_0$ and $\hat{\Theta}_n-\Theta_0$. For
estimation of $\Theta_0$, Gelato has in some settings a better rate of
convergence than the GLasso or SCAD type estimators.
From a theoretical point of view, 
our method is clearly aimed for bounding the operator and Frobenius 
norm of the inverse covariance matrix. 
We also derive bounds on the convergence rate for the estimated 
covariance matrix and on the Kullback Leibler divergence.
From a non-asymptotic point of view, our method has a clear advantage 
when the sample size is small relative to the sparsity $S = |E_0|$: 
for a given sample size $n$, we bound the variance in our re-estimation 
stage by excluding edges of $E_0$ with small weights from the 
selected edge set $\hat{E}_n$ while ensuring that we  do not introduce 
too much bias.
Our Gelato method also addresses the bias problem inherent in the GLasso
estimator since we no longer shrink the entries in the covariance matrix
corresponding to the selected edge set $\hat{E}_n$ in
the maximum likelihood estimate, as shown in Section 3.3.

Our experimental results show that Gelato performs better than GLasso or
the Space method for AR-models while the situation is reversed for some
random precision matrix models; in case of an exponential decay model for
the precision matrix, all methods exhibit the same performance. For Gelato,
we demonstrate that thresholding is a valuable feature. We also show
experimentally how one can use cross-validation for choosing the tuning
parameters in regression and thresholding. Deriving theoretical results on
cross-validation is not within the scope of this paper.

\section{Acknowledgments}

Min Xu's research was supported by NSF grant CCF-0625879 and
AFOSR contract FA9550-09-1-0373.
Shuheng Zhou thanks Bin Yu warmly for hosting her visit at UC Berkeley 
while she was conducting this research in Spring 2010.
SZ's research was supported in part by the Swiss National Science 
Foundation (SNF) Grant 20PA21-120050/1.
We thank delightful conversations with Larry Wasserman,
which inspired the name of G{\bf e}lat{\bf o}, and Liza Levina.

\begin{appendix}

\section{Theoretical analysis and proofs}
\label{sec:append-pre}
In this section, we specify some preliminary definitions. 
First, note that when we discuss estimating the parameters 
$\Sigma_0$ and $\Theta_0 = \Sigma_0^{-1}$, we always assume that
\ben
\label{eq::non-singular}
& & \varphi_{\max}(\Sigma_0) := 1/\varphi_{\min}(\Theta_0) \leq 1/{\ul{c}} < \infty \;
\text{ and } 1/{\varphi_{\max}(\Theta_0)} 
=  \varphi_{\min}(\Sigma_0) \geq \ul{k} >0,
 \\
\label{eq::non-singular-2}
& & 
\text{ where we assume} \; \;
\ul{k}, \ul{c}  \leq  1 \; \; \; \text{ so that } \ul{c} \leq 1 \leq 1/{\ul{k}}.
\een
It is clear that these conditions are exactly that of (A2)
in Section~\ref{sec:theory} with
\bens
\label{eq::corres}
M_{\mathrm{upp}} := 1/{\ul{c}} \; \;\text{ and } \; \; M_{\mathrm{low}} := \ul{k},
\eens
where it is clear that for $\Sigma_{0,ii} = 1, i = 1, \ldots, p$, we have 
the sum of $p$ eigenvalues of $\Sigma_0$,
$\sum_{i=1}^p \varphi_{i}(\Sigma_0) = {\rm tr}(\Sigma_0) = p$.
Hence it will make sense to assume that~\eqref{eq::non-singular-2} holds, 
since otherwise,~\eqref{eq::non-singular} implies 
that $\varphi_{\min}(\Sigma_0) = \varphi_{\max}(\Sigma_0) = 1$ which is 
unnecessarily restrictive. 

We now define parameters relating to the key notion of {\em essential sparsity} 
$s_0$ as explored in~\cite{CT07,Zhou09th,Zhou10} for regression.
Denote the number of non-zero non-diagonal entries in each row of 
$\Theta_0$ by $s^i$. Let $s = \max_{i = 1, \ldots, p} s^i$
denote the highest node degree in $G = (V, E_0)$.
Consider nodewise regressions as in~\eqref{eq::regr}, where 
we are given vectors of parameters $\{\beta^i_j, j = 1, \ldots, p, j \not= i\}$
for $i=1, \ldots, p$.
With respect to the degree of node $i$ for each $i$,
we  define $s^i_0 \leq s^i \leq s$ as the smallest integer such that 
\ben
\label{eq::define-s0}
\sum_{j=1, j\not=i}^p \min((\beta^i_j)^2, \lambda^2 \Var(V_i)) 
\leq s^i_0 \lambda^2 \Var(V_i), \; \text{where } \lambda = \sqrt{2\log p/n},
\een
where  $s^i_0$ denotes $s^{i}_{0,n}$ as defined
in~\eqref{eq::cond-regr1}.
\begin{definition}
\textnormal{\bf (Bounded degree parameters.)}
\label{as::residual-var}
The size of the node degree $s^i$ for each node $i$
is upper bounded by an integer $s < p$. 
For $s_0^i$ as in \eqref{eq::define-s0}, define
\ben
\label{eq::omni-s0}
s_0 & := & \max_{i=1, \ldots, p} s^i_0 \leq s \text{ and } \; 
 \; S_{0,n} \; := \; \sum_{i=1, \ldots, p} s^i_0
\een 
where $S_{0,n}$ is exactly the same as
in~\eqref{eq::cond-regr2}, although we now drop subscript $n$ from
$s^{i}_{0,n}$ in order to simplify our notation.
\end{definition}
We now define the following parameters related to $\Sigma_0$.
For an integer $m \leq p$, we define the smallest and largest
 {\bf m-sparse eigenvalues} of $\Sigma_0$ as follows:
\bens
\label{eq::eigen-Sigma}
\sqrt{\rho_{\min}(m)} & := & \min_{t \not=0; m-\text{sparse}} \; \; \frac{\twonorm{\Sigma_0^{1/2} t}}{\twonorm{t}}, \; \;
\sqrt{\rho_{\max}(m)} \; := \; \max_{t \not=0;m-\text{sparse}} 
\; \; \frac{\twonorm{\Sigma_0^{1/2} t}}{\twonorm{t}}.
\eens
\begin{definition}
\textnormal{\bf (Restricted eigenvalue condition $RE(s_0, k_0, \Sigma_0)$).}
\label{def:memory}
For some integer $1\leq s_0 < p$ and a positive number $k_0$, 
the following condition holds for all $\upsilon \not=0$,
\beq
\label{eq::admissible-random}
\inv{K(s_0, k_0, \Sigma_0)} := \min_{\stackrel{J \subseteq \{1, \ldots,
    p\},}{|J| \leq s_0}} 
\min_{\norm{\upsilon_{J^c}}_1 \leq k_0 \norm{\upsilon_{J}}_1} 
\; \;  \frac{\norm{\Sigma_0^{1/2} \upsilon}_2}{\norm{\upsilon_{J}}_2} > 0,
\eeq
where $\upsilon_{J}$ represents the subvector of $\upsilon \in \R^p$ 
confined to a subset $J$ of $\{1, \ldots, p\}$.
\end{definition}
When $s_0$ and $k_0$ become smaller, this condition
is easier to satisfy.
When we only aim to estimate the graphical structure  $E_0$ itself,
the global conditions~\eqref{eq::non-singular} need not hold in general.
Hence up till Section~\ref{sec:append-frob-missing}, we only need to 
assume that $\Sigma_0$ satisfies~\eqref{eq::admissible-random} for $s_0$ as 
in~\eqref{eq::define-s0}, and the sparse eigenvalue $\rho_{\min}(s) >0$.
In order of estimate the covariance matrix $\Sigma_0$, we do assume 
that~\eqref{eq::non-singular} holds, which guarantees that the $RE$ 
condition always holds on $\Sigma_0$, and $\rho_{\max}(m),
\rho_{\min}(m)$ are  upper and lower bounded by some constants for all $m \leq p$.
We continue to adopt parameters such as $K$, $\rho_{\max}(s)$,
and $\rho_{\max}(3s_0)$ for the purpose of defining constants that are 
reasonable tight under condition~\eqref{eq::non-singular}. 
In general, one can think of
$$\rho_{\max}(\max(3s_0,s)) \ll 1/\ul{c} < \infty \; \; \text{and} \; \; 
K^2(s_0, k_0, \Sigma_0) \ll 1/{\ul{k}} < \infty,$$
for $\ul{c}, \ul{k}$ as in~\eqref{eq::non-singular} when $s_0$ is small.

Roughly speaking, for two variables $X_i, X_j$ as in~\eqref{eq::rand-des} 
such that their corresponding entry in $\Theta_0 =(\theta_{0, ij})$
satisfies: $\theta_{0,ij} < \lambda \sqrt{\theta_{0,ii}}$,
where $\lambda = \sqrt{2 \log (p) /n}$, we can not  guarantee that 
$(i, j) \in \hat{E}_n$ when we aim to keep 
$\asymp s_0^i$ edges for node $i, i=1, \ldots, p$.
For a given $\Theta_0$, as the sample size $n$ increases, 
we are able to select edges with smaller coefficient $\theta_{0,ij}$.
In fact it holds that
\ben
\label{eq::tail-s0}
|\theta_{0,ij}| < \lambda \sqrt{\theta_{0,ii}} \text{ which is equivalent to }  
|\beta^i_{j}| < \lambda \sigma_{V_i}, \; \text{ for all } \;
j \geq s^i_0 + 1 + \mathbb{I}_{i \leq s^i_0 + 1},
\een
where $\mathbb{I}_{\{\cdot\}}$ is the indicator function,
if we order the regression coefficients as follows:
\bens
|\beta^i_1| \geq |\beta^i_2| ...\geq |\beta^i_{i-1}| \geq |\beta^i_{i+1}|.... \geq |\beta^i_{p}|,
\eens
in view of~\eqref{eq::regr},
which is the same as if we order for row $i$ of $\Theta_0$,
\ben
\label{eq::beta-order}
|\theta_{0, i1} | \geq |\theta_{0, i,2}| ...\geq |\theta_{0, i, i-1}| \geq |\theta_{0, i,i+1}|.... \geq |\theta_{0, i,p}|.
\een
This has been show in~\citep{CT07}; See also~\cite{Zhou10}.
\silent{
For more detail, we have  by definition of $s^i_0$, the fact 
$0 \leq s^i_0 \leq s^i$, 
we have for $s < p$, and $V := V_i$,
\ben
\nonumber
(s^i_0 -1) \lambda^2 \theta_{0,ii}  & \leq &
\sum_{j=1, \not=i}^p \min\left(\theta_{0,ij}^2, \lambda^2 \theta_{0,ii} \right) 
\; \leq \; s^i_0 \lambda^2 \theta_{0,ii} \text{ where }  \\
\label{eq::s0-lower-bound}
 (s^i_0 + 1) \min\left(\theta^2_{0,i,t_i} \lambda^2 \theta_{0,ii} \right) & \leq &
\sum_{j=1 \not=i}^{t_i} \min\left(\theta^2_{0,ij}, \lambda^2 \theta_{0,ii} \right)
\; \leq \;  s^i_0 \lambda^2 \theta_{0,ii}
\een
where $t_i = s^i_0 + 1 + \mathbb{I}_{i \leq s^i_0 + 1}$.
In the last equation, it is understood that we sum over the first $s^i_0 +1$ 
items according to the ordering in~\eqref{eq::beta-order}.
Now ~\eqref{eq::s0-lower-bound} implies that
\bens
\min(\theta^2_{0,i, {s^i_0 +1}} \sigma^2_V, \lambda^2) < \lambda^2 \; 
\text{ which implies that } \; \; 
|\theta_{0, i, {s^i_0 +1}}| < \lambda \sqrt{\theta_{0,ii}}
\eens
where $\theta_{0,i, {s^i_0 +1}}$ is understood to be $\theta_{0,i, {s^i_0 + 2}}$ 
if $i \leq s^i_0 + 1$. 
Hence by~\eqref{eq::beta-order},~\eqref{eq::tail-s0} holds.}
\silent{
XXXX need to change to $_0$ everywhere.
\ben
\label{eq::tail-s0}
|\theta_{ij} | < \lambda \sqrt{\theta_{ii}} \; \; 
\text{ and } \; |\beta^i_{j} | < \lambda \sigma_{V_i} \; \; 
\text{ for all } \;
j  \geq s^i_0 + 1 + \mathbb{I}_{i \leq s^i_0 + 1}.
\een
}

\subsection{Concentration bounds for the random design}
For  the random design $X$ generated by~(\ref{data}), let $\Sigma_{0,ii} =1$ for all $i$.
In preparation for showing the oracle results of 
Lasso in Theorem~\ref{thm:RE-oracle}, we first state  some
concentration bounds on $X$. Define for some $0< \theta < 1$
\begin{eqnarray}
\label{eq::good-random-design-diag}
\F(\theta) := \left\{X: \forall j = 1, \ldots, p,\;
1 - \theta \leq {\twonorm{X_j}}/{\sqrt{n} } \leq 1 + \theta \right\},
\end{eqnarray}
where $X_1, \ldots, X_p$ are the column vectors of the $n \times p$ design
matrix $X$. 
When all columns of 
$X$ have an Euclidean norm close to $\sqrt{n}$ as in \eqref{eq::good-random-design-diag} ,
it makes sense to discuss the RE condition in the form of~\eqref{eq::admissible}
as formulated in ~\citep{BRT09}. For the integer $1\leq s_0 < p$ as 
defined in~\eqref{eq::define-s0} and a positive number $k_0$,  
$RE(s_0, k_0, X)$ requires that the following holds for all 
$\upsilon \not = 0$,
\beq
\label{eq::admissible}
\inv{K(s_0, k_0, X)} \stackrel{\triangle}{=}
\min_{\stackrel{J \subset \{1, \ldots, p\},}{|J| \leq s_0}}
\min_{\norm{\upsilon_{J^c}}_1 \leq k_0 \norm{\upsilon_{J}}_1}
\; \;  \frac{\norm{X \upsilon}_2}{\sqrt{n}\norm{\upsilon_{J}}_2} > 0.
\eeq
The parameter $k_0 > 0$ is understood to be the 
same quantity throughout our discussion.
The following event $\RE$ provides an upper bound on 
$K(s_0, k_0, X)$ for a given $k_0 >0$
when $\Sigma_0$ satisfies $RE(s_0, k_0, \Sigma_0)$ condition:
\begin{eqnarray}
\label{eq::good-random-design-RE}
\RE(\theta) := \left\{X: RE(s_0, k_0, X) \; \text{ holds with } \; 
0 < K(s_0, k_0, X) \leq \frac{K(s_0, k_0, \Sigma_0)}{1-\theta} \right\}.
\end{eqnarray}
For some integer $m \leq p$, we define
the smallest and largest $m$-sparse eigenvalues of $X$ to be
\ben
\label{eq::eigen-admissible-s}
\Lambda_{\min}(m)  & := &
\min_{\upsilon \not= 0; m-\text{sparse}} \; \;
{\twonorm{X \upsilon}^2}/{(n \twonorm{\upsilon}^2)} \; \; \text{ and }  \\
\label{eq::eigen-max}
\Lambda_{\max}(m)  & := &
\max_{\upsilon \not=0; m-\text{sparse}} \;
{\twonorm{X \upsilon}^2}/{(n \twonorm{\upsilon}^2)},
\een
upon which we define the following event: 
\begin{eqnarray}
\label{eq::cond-M}
\M(\theta) := \left\{X:~\eqref{eq::phi-max-bound} \text{ holds } 
\forall m \leq \max(s, (k_0 + 1)s_0) \right\}, \text{  for which } \\
\label{eq::phi-max-bound}
0 < (1-\theta) \sqrt{\rho_{\min}(m)} \leq \sqrt{\Lambda_{\min}(m)} \leq
\sqrt{\Lambda_{\max}(m)} \leq (1+\theta) \sqrt{\rho_{\max}(m)}.
\end{eqnarray}
Formally, we consider the set of random designs that 
satisfy all events as defined, for some $0< \theta < 1$.
Theorem~\ref{thm:corn} shows concentration results that we 
need for the present work, which follows from Theorem~1.6
in~\cite{Zhou10sub} and Theorem~3.2 in~\cite{RZ11}.
\begin{theorem}
\label{thm:corn}
Let $0< \theta < 1$. Let $\rho_{\min}(s) >0$, where 
$s < p$ is the maximum node-degree in $G$. 
Suppose $RE(s_0, 4, \Sigma_0)$ holds for $s_0$ as in~\eqref{eq::omni-s0},
where $\Sigma_{0,ii} = 1$ for $i=1,\ldots, p$.
Let $f(s_0) = \min\left(4 s_0 \rho_{\max}(s_0) \log (5ep/s_0), s_0 \log p \right)$.
Let $c, \alpha, c' >0$ be some  absolute constants.
Then, for a random design $X$ as generated by~(\ref{data}), we have
\ben
\label{eq::X-prob}
\prob{\X} := \prob{\RE(\theta) \cap \F(\theta) \cap \M(\theta)}
 \geq 1- 3\exp(-c \theta^2 n/\alpha^4)
\een
as long as the sample size satisfies
\ben
\label{eq::sample-size}
n > \max\left\{ 
 \frac{9c' \alpha^4}{\theta^2} 
\max\left(36  K^2(s_0, 4, \Sigma_0) f(s_0), \log p \right),
\frac{80 s \alpha^4}{\theta^2} \log  \left(\frac{5ep}{s\theta} \right) \right\}.
\een
\end{theorem}

\begin{remark}
We note that the constraint $s< p/2$, which has appeared in~\citet[Theorem 1.6]{Zhou10sub} 
is unnecessary.
Under a stronger $RE$ condition on $\Sigma_0$,
a tighter bound on the sample size $n$, which is independent of
$\rho_{\max}(s_0)$,  is given in~\cite{RZ11} in order to guarantee $\RE(\theta)$.
We do not pursue this optimization here as we assume that 
$\rho_{\max}(s_0)$ is a bounded constant throughout this paper.
We emphasize that we only need the first term in
~\eqref{eq::sample-size} in order to obtain $\F(\theta)$ and $\RE(\theta)$;
The second term is used to bound sparse eigenvalues of order $s$.
\end{remark}

\subsection{Definitions of other various events}
Under (A1) as in Section~\ref{sec:theory},
excluding event $\X^c$ as bounded in Theorem~\ref{thm:corn}
and events $\C_a, \X_0$ to be defined in this subsection, 
we can then proceed to treat $X \in \X \cap \C_a$ as a 
deterministic design in regression and thresholding, for which 
$\RE(\theta) \cap \M(\theta) \cap \F(\theta)$ holds with $\C_a$, 
We then make use of event $\X_0$ in the MLE refitting stage for 
bounding the Frobenius norm. We now define two types of correlations
events  $\C_a$ and $\X_0$.

{ \bf Correlation bounds on $X_j$ and $V_i$.}
In this section, we first bound the maximum correlation between 
pairs of random vectors $(V_i, X_j)$, for all $i, j$ where $i \not= j$,
each of which corresponds to a pair of variables $(V_i, X_j)$ as defined 
in~\eqref{eq::regr} and~\eqref{eq::regressions}. Here we use 
$X_j$ and $V_i$, for all $i, j$, 
to denote both random vectors and their corresponding variables.

Let us define $\sigma_{V_j}  := \sqrt{\Var(V_j)} \geq v > 0$ as a shorthand.
Let $V'_j := V_j/\sigma_{V_j}, j = 1, \ldots, p$ be a standard normal random
variable.
Let us now define for all $j, k \not=j$,
$$Z_{jk} = \inv{n} \ip{V'_j, X_k} = \inv{n} \sum_{i=1}^n v'_{j,i} x_{k,i},$$ 
where for all $i = 1, \ldots, n$
$v'_{j,i}, x_{k,i}, \forall j, k \not=j$ are independent standard
normal random variables. For some $a \geq 6$, let event
\begin{eqnarray}
\label{eq::good-random-design}
\C_a := \left\{\max_{j, k} |Z_{jk}| <  \sqrt{1 + a} \sqrt{(2\log p)/n} \text{ where } a \geq 6 \right\}.
\end{eqnarray} 

{\noindent {\bf Bounds on pairwise correlations in columns of $X$.}}
Let $\Sigma_0 := (\sigma_{0,ij})$, where we denote 
$\sigma_{0,ii} := \sigma_i^2$. Denote by $\Delta = {X^T X}/{n} - \Sigma_0$.
Consider for some constant $C_3> 4 \sqrt{5/3}$,
\begin{eqnarray}
\label{eq::good-random-design-orig}
\X_0 :=  
\left\{\max_{j, k} |\Delta_{jk}| < C_3 \sigma_i \sigma_j \sqrt{{\log \max\{p,n\}}/{n}} < 1/2 \right\}.
\end{eqnarray} 
We first state Lemma~\ref{lemma:correlation}, which is used for bounding
a type of correlation events across all regressions; see proof of 
Theorem~\ref{thm:RE-oracle-all}. It is also clear that 
event $\C_a$ is equivalent to the event to be defined in~\eqref{eq::cross-cor}.
 Lemma~\ref{lemma:correlation} also justifies the choice of 
$\lambda_n$ in nodewise regressions (cf. Theorem~\ref{thm:RE-oracle-all}).
We then bound event $\X_0$ in Lemma~\ref{lemma:gaussian-covars}.
Both proofs appear in Section~\ref{sec:append-proof-cov}. 
\begin{lemma}
\label{lemma:correlation}
Suppose that $p < e^{n/4 C_2^2}$. Then with probability
at least $1- 1/p^2$, we have
\ben
\label{eq::cross-cor}
\forall j \not= k, \; \; 
\abs{\inv{n} \ip{V_j, X_k}} \leq  \sigma_{V_j} \sqrt{1 + a} \sqrt{(2\log p)/n}
\een
where $\sigma_{V_j} = \sqrt{\Var(V_j)}$ and $a \geq 6$.
Hence $\prob{\C_a} \geq 1 - 1/p^2.$
\end{lemma}

\begin{lemma}
\label{lemma:gaussian-covars} 
For a random design $X$ as in (\ref{eq::rand-des}) with 
$\Sigma_{0,jj} = 1, \forall j \in \{1, \ldots, p\}$, 
and for $p < e^{n/4 C_3^2}$, where $C_3 > 4 \sqrt{5/3}$,
we have 
\begin{eqnarray*}
\prob{\X_0} \ge 1 - 1/\max \{n,p\}^2.
\end{eqnarray*}
\end{lemma}
We note that the upper bounds on $p$ in
Lemma~\ref{lemma:correlation} and~\ref{lemma:gaussian-covars} clearly 
hold given (A1).
For the rest of the paper, we prove Theorem~\ref{thm:RE-oracle-all}
in Section~\ref{sec:append-p-regres} for nodewise regressions.
We proceed to derive bounds on selecting an edge set $E$ in 
Section~\ref{sec:append-proof-bias}.
We then derive various bounds on the maximum likelihood estimator 
given $E$ in Theorem~\ref{thm::frob-missing}-~\ref{thm::frob-missing-risk} 
in Section~\ref{sec:append-frob-missing},
where we also prove Theorem~\ref{thm:main}.
Next, we prove  Lemma~\ref{lemma:correlation} 
and~\ref{lemma:gaussian-covars} in Section~\ref{sec:append-proof-cov}.

\subsection{Proof of Lemma~\ref{lemma:correlation} and~\ref{lemma:gaussian-covars}}
\label{sec:append-proof-cov}
We first state the following large inequality bound on
products of correlated normal random variables.
\begin{lemma}{\textnormal{\citet[Lemma 38]{ZLW08}}}
\label{lemma:devi-non-diag}
Given a set of identical independent random variables 
$Y_1, \ldots, Y_n \sim Y$,
where $Y = x_1 x_2$, with $x_{1}, x_{2} \sim N(0, 1)$ and 
$\sigma_{12} = \rho_{12}$ with $\rho_{12} \leq 1$ being their correlation 
coefficient. Let us now define $Q = \inv{n} \sum_{i=1}^n Y_i
=: \inv{n} \ip{X_1, X_2} = \inv{n} \sum_{i=1}^n x_{1,i} x_{2,i}$.
Let $\Psi_{12} = (1 + \sigma^2_{12})/2$. For $0 \leq \tau \leq \Psi_{12}$, 
\begin{eqnarray}
\label{eq::boxcar-non-diag-I}
\prob{|Q - \expct{Q}| > \tau} \leq 
\exp\left\{-\frac{3 n \tau^2}{10(1 + \sigma_{12}^2)} \right\} 
\end{eqnarray}
\end{lemma}

\begin{proofof}{\textnormal{Lemma~\ref{lemma:correlation}}}
It is clear that event~\eqref{eq::cross-cor} is the same as event $\C_a$.
Clearly we have at most $p(p-1)$ unique entries 
$Z_{jk}, \forall j \not= k$.
By the union bound and by taking
$\tau =  C_2 \sqrt{\frac{\log p}{n}}$ 
in~\eqref{eq::boxcar-non-diag-I} with $\sigma_{jk} = 0, \forall j, k$,
where $\sqrt{2(1+a)} \geq C_2 > 2 \sqrt{10/3}$ for $a \geq 6$.
\begin{eqnarray*}
\label{eq::Z-bound}
1- \prob{\C_a} 
& = & \prob{\max_{jk} |Z_{jk}| \geq \sqrt{2 (1+a)} \sqrt{\frac{\log p}{n}}} \\
& \leq & \prob{\max_{jk} |Z_{jk}| \geq C_2 \sqrt{\frac{\log p}{n}}} 
 \leq 
(p^2 - p) \exp\left(- \frac{3 C_2^2 \log p}{10}\right) \\
& \leq &
p^2 \exp\left(- \frac{3 C_2^2 \log p}{10}\right)
= p^{- \frac{3C_2^2}{10} + 2}  <  \frac{1}{p^{2}}
\end{eqnarray*}
where we apply Lemma~\ref{lemma:devi-non-diag} with
$\rho_{jk} = 0, \forall j, k=1, \ldots, p, j \not=k$ and use the fact that 
$\expct{Z_{jk}} = 0$.
Note that $p < e^{n/4 C_2^2}$ guarantees that $C_2 \sqrt{\frac{\log p}{n}} < 1/2$.
\end{proofof}
In order to bound the probability
of event $\X_0$, we first state the following large inequality bound for the 
non-diagonal entries of $\Sigma_0$, which follows immediately
from Lemma~\ref{lemma:devi-non-diag} by plugging in 
$\sigma_{i}^2 = \sigma_{0,ii}  = 1, \forall i=1, \ldots, p$ and using the 
fact that 
$|\sigma_{0,jk}| = |\rho_{jk} \sigma_{j} \sigma_k| \leq 1, \forall j \not= k$,
where $\rho_{jk}$ is the correlation coefficient between variables
$X_j$ and $X_k$.
Let $\Psi_{jk} = (1 + \sigma^2_{0,jk})/2$. Then
\begin{eqnarray}
\label{eq::boxcar-non-diag}
\prob{|\Delta_{jk}| > \tau} \leq 
\exp\left\{-\frac{3 n \tau^2}{10(1 + \sigma_{0,jk}^2)} \right\} 
\leq
\exp\left\{-\frac{3 n \tau^2}{20} \right\} \; \text{ for }  \; 0 \leq \tau \leq \Psi_{jk}.
\end{eqnarray}
We now also state a large deviation bound for the $\chi^2_n$ 
distribution~\citep{Joh01}:
\begin{eqnarray}
\label{eq::chi-dev}
\prob{\frac{\chi^2_n}{n} - 1 > \tau} & \leq & 
\exp\left(\frac{-3 n \tau^2}{16}\right),\; \text{for} \;0 \leq \tau \leq \half.
\end{eqnarray}
Lemma~\ref{lemma:gaussian-covars} follows 
from~\eqref{eq::boxcar-non-diag} and \eqref{eq::chi-dev} immediately.

\begin{proofof}{\textnormal{Lemma~\ref{lemma:gaussian-covars}}}
Now it is clear that we have $p(p-1)/2$ unique non-diagonal
entries $\sigma_{0,jk}, \forall j \not= k$ and $p$ diagonal entries. 
By the union bound and by taking
$\tau =  C_3 \sqrt{\frac{\log \max\{p,n\}}{n}}$ in~\eqref{eq::chi-dev}
and~\eqref{eq::boxcar-non-diag} with $\sigma_{0,jk} \leq 1$, 
we have
\begin{eqnarray*}
\label{eq::Delta-bound}
\prob{(\X_0)^c} 
& = & 
\prob{\max_{jk} |\Delta_{jk}| \geq C_3 \sqrt{\frac{\log \max\{p,n\}}{n}}} \\
& \leq & 
p \exp\left(- \frac{3 C_3^2 \log \max\{p,n\}}{16}\right) + 
\frac{p^2 - p}{2} \exp\left(- \frac{3 C_3^2 \log \max\{p,n\}}{20}\right) \\
& \leq &
p^2 \exp\left(- \frac{3 C_3^2 \log \max\{p,n\}}{20}\right)
= \left(\max\{p,n\}\right)^{- \frac{3C_3^2}{20} + 2}  <  \frac{1}{(\max\{p,n\})^{2}}
\end{eqnarray*}
for $C_3 > 4 \sqrt{5/3}$,
where for the diagonal entries we use~\eqref{eq::chi-dev}, and
for the non-diagonal entries, we use~\eqref{eq::boxcar-non-diag}.
Finally, $p < e^{n/4 C_3^2}$ guarantees that 
$C_3 \sqrt{\frac{\log \max\{p,n\}}{n}} < 1/2$.
\end{proofof}

\section{Bounds for nodewise regressions}
\label{sec:append-p-regres}

In Theorem~\ref{thm:RE-oracle-all} and Lemma~\ref{lemma:threshold-RE}
we let $s^i_0$ be as in~\eqref{eq::define-s0} and $T^i_0$ denote locations 
of the $s^i_0$ largest coefficients of $\beta^i$ in absolute values.
For the vector $h^i$ to be defined in Theorem~\ref{thm:RE-oracle-all},
we  let $T^i_1$ denote the $s_0^i$ largest positions of $h^i$  in absolute 
values outside of $T_0^i$; Let $T^i_{01} := T^i_0 \cup T^i_1$. 
We suppress the superscript in $T^i_0, T^i_1$, and $T_{01}^i$
throughout this section for clarity.
\begin{theorem}\textnormal{\bf (Oracle inequalities of the nodewise regressions)}
\label{thm:RE-oracle-all}
Let $0< \theta < 1$. Let $\rho_{\min}(s) >0$, where 
$s < p$ is the maximum node-degree in $G$. 
Suppose $RE(s_0, 4, \Sigma_0)$ holds 
for $s_0 \leq s$ as in~\eqref{eq::omni-s0}, 
where $\Sigma_{0,ii} = 1$ forall $i$.
Suppose $\rho_{\max}(\max(s, 3s_0)) < \infty$.
The data is generated by
$X^{(1)},\ldots ,X^{(n)}\ \mbox{i.i.d.} \sim {\cal N}_p(0,\Sigma_0)$,
where the sample size $n$ satisfies~\eqref{eq::sample-size}.

Consider the nodewise regressions in~\eqref{eq::lasso}, 
where for each $i$, we regress $X_i$ onto the other
variables $\{X_k;\ k \neq i\}$ following~\eqref{eq::regr},
where $V_i \sim N(0, \Var(V_i))$ is independent of $X_j, \forall j \not= i$ 
as in~\eqref{eq::regressions}.

Let $\beta^i_{\init}$ be an optimal solution to~\eqref{eq::lasso}
for each $i$.
Let $\lambda_{n} = d_0 \lambda = d^i_0 \lambda \sigma_{V_i}$
where $d_0$ is chosen such that $d_0 \geq 2 (1+\theta) \sqrt{1+a}$ 
holds for some $a \geq 6$.
Let $h^i = \beta^i_{\init} - \beta^i_{T_0}$.
Then simultaneously for all $i$, on $\C_a \cap \X$, where 
$\X := \RE(\theta) \cap \F(\theta) \cap \M(\theta)$, we have
\ben
\nonumber
\label{eq::two-p-bound}
\twonorm{\beta^i_{\init} - \beta^i}
& \leq &  
\lambda \sqrt{ s^i_0} d_0 \sqrt{2D_0^2 + 2D_1^2 + 2}, \text{ where }
\\
\label{eq::extraone}
\twonorm{h_{T_{01}}} & \leq & D_0 d_0 \lambda \sqrt{s^i_0} \; \; \; \text{ and
}  \; \; \;
\norm{h^i_{T_0^{c}}}_1 = 
\norm{\beta^i_{\init, T_0^{c}}}_1
\; \leq \; D_1 d_0 \lambda s^i_0
\een
where $D_0, D_1$ are defined  in~\eqref{eq::D0-define} 
and~\eqref{eq::D1-define} respectively.
\end{theorem}
Suppose we choose for some constant $c_0 \geq 4 \sqrt{2}$ and $a_0 = 7$,
$$d_0 = c_0 (1+ \theta)^2 \sqrt{\rho_{\max}(s) \rho_{\max}(3s_0)},$$
where we assume that $\rho_{\max}(\max(s, 3s_0)) < \infty$ is reasonably bounded. Then
\bens
\label{eq::exam1}
D_0\;  \leq \; \frac{5 K^2(s_0, 4, \Sigma_0)}{(1-\theta)^2} \; \text{ and } 
D_1 \; \leq \;  \frac{49 K^2(s_0, 4, \Sigma_0) }{16 (1- \theta)^2}.
\eens
The choice of $d_0$ will be justified in Section~\ref{sec:app-lasso-bounds},
where we also the upper bound on $D_0, D_1$ as above.

\begin{proof}
Consider each regression function in~\eqref{eq::lasso} with
$X_{\cdot \setminus i}$ being the design matrix and $X_i$ the response vector,
where  $X_{\cdot \setminus i}$ denotes columns of $X$ excluding $X_i$.
It is clear that for $\lambda_{n} = d_0 \lambda$,
we have for $i=1, \ldots, p$  and $a \geq 6$,
$$\lambda_n = (d_0/\sigma_{V_i}) \sigma_{V_i} \lambda := 
d_0^i \sigma_{V_i} \lambda \geq d_0 \lambda \sigma_{V_i} \geq 
 2 (1+\theta) \lambda \sqrt{1+a} \sigma_{V_i} = 2 (1+ \theta) \basepen$$ 
such that~\eqref{eq::single-penalty} holds given that 
$\sigma_{V_i} \leq 1, \forall i$, 
where it is understood that $\sigma :=  \sigma_{V_i}$.

It is also clear that on $\C_a \cap \X$, event $\T_a \cap \X$ holds
for each regression when we invoke Theorem~\ref{thm:RE-oracle},
with $Y := X_i$ and $X := X_{\cdot \setminus i}$, for $i=1, \ldots,
p$.  By definition $d_0^i \sigma_{V_i} = d_0$.
We can then invoke bounds for each individual regression 
as in Theorem~\ref{thm:RE-oracle} to conclude.
\end{proof}

\section{Bounds on thresholding}
\label{sec:append-proof-bias}
In this section, we first show Lemma~\ref{lemma:threshold-RE},
following conditions in Theorem~\ref{thm:RE-oracle-all}.
\begin{lemma}
\label{lemma:threshold-RE}
Suppose $RE(s_0, 4, \Sigma_0)$ holds for $s_0$ be as in~\eqref{eq::omni-s0} 
and $\rho_{\min}(s) >0$, where $s < p$ is the maximum node-degree in $G$.
Suppose $\rho_{\max}(\max(s, 3s_0)) < \infty$.
Let  $S^i = \{j:  j \not=i, \; \beta_j^i \not=  0\}$.
Let  $c_0 \geq 4 \sqrt{2}$ be some absolute constant.
Suppose $n$ satisfies~\eqref{eq::sample-size}.
Let $\beta^i_{\init}$ be an optimal solution to~\eqref{eq::lasso} with
$$\lambda_{n} = d_0 \lambda \; \text{ where } \; d_0 = c_0 (1+ \theta)^2 \sqrt{\rho_{\max}(s) \rho_{\max}(3s_0)};$$
Suppose for each regression, we apply the same thresholding rule to obtain a subset 
$\sel^i$  as follows,
$$\sel^{i} = \{j: j \not= i, \; \abs{\beta^{i}_{j, \init}} \geq t_0 = f_0 \lambda\},
\; \text{ and } \; \drop^i := \{1, \ldots, i-1, i+1, \ldots, p\} \setminus \sel^i$$
where $f_0 := D_4 d_0$ for some constant $D_4$ to be specified.
Then we have on event $\C_a \cap \X$,
\ben
\label{eq::ideal-size}
|\sel^{i}| & \leq &  s^{i}_0 (1 + {D_1}/{D_4})
  \text{ and } \; \; |\sel^{i} \cup S^i|   \leq  s^i + ({D_1}/{D_4}) s^{i}_0 
\; \text{ and } \\
\label{eq::drop-norm-single}
\twonorm{\beta^i_{\drop}}
& \leq & d_0 \lambda \sqrt{s^i_0} \sqrt{1+  (D_0 + D_4)^2} 
\een 
where $\drop$ is understood to be $\drop^i$ and $D_0, D_1$ are the
same constants as in Theorem~\ref{thm:RE-oracle-all}.
\end{lemma}

We now show Corollary~\ref{cor:drop-norm}, which
proves Proposition~\ref{prop:bias} and the first statement of 
Theorem~\ref{thm:main}.
Recall $\Theta_{0}  = \Sigma_0^{-1}$.
Let $\Theta_{0,\drop}$ denote the submatrix of $\Theta_0$
indexed by $\drop$ as in~\eqref{eq::all-drop} with all other positions 
set to be 0. Let $E_0$ be the true edge set.
\begin{corollary}
\label{cor:drop-norm}
Suppose all conditions in Lemma~\ref{lemma:threshold-RE} hold.
Then on event $\C_a \cap \X$,
for $\tilde{\Theta}_0$ as in \eqref{eq::new-posi} and
 $E$ as in~\eqref{eq::non-drop}, we have for $S_{0,n}$ as in~\eqref{eq::omni-s0}
and $\Theta_0 =(\theta_{0,ij})$
\ben
\label{eq::on-edges}
 |E| & \leq & 2(1+{D_1}/{D_4}) S_{0,n} 
\; \text{ where } \;  |E \setminus E_0| \leq {2 D_1}/{D_4} S_{0,n}  \\
\nonumber
\fnorm{\Theta_{0,\drop}} 
& := &  \fnorm{\tilde{\Theta}_0  -\Theta_0} \\
\label{eq::bias}
& \leq & 
\sqrt{\min\{S_{0,n}  (\max_{i=1,...p} \theta_{0,ii}^2), s_0 \fnorm{\diag(\Theta_0)}^2\}}
\sqrt{(1+  (D_0 + D_4)^2)} d_0  \lambda \\
\nonumber
& := & \sqrt{S_{0,n}  \left(1 + (D_0 + D_4)^2  \right)} C_{\diag} d_0 \lambda
\een
where
$C^2_{\diag} :=
\min\{\max_{i=1,...p} \theta_{0,ii}^2, (s_0/S_{0,n} ) \fnorm{\diag(\Theta_0)}^2 \}$,
and $D_0, D_1$  are understood to be the same constants as in
Theorem~\ref{thm:RE-oracle-all}.
Clearly, for $D_4 \geq D_1$, we have~\eqref{eq::thm-E-bounds}.
\end{corollary}

\begin{proof}
By the OR rule in~\eqref{eq::est-edge},   
we could select at most $2 |\sel_i|$ edges. We have by~\eqref{eq::ideal-size}
\bens 
\size{E} & \leq & \sum_{i=1,...p} 2 (1 + {D_1}/{D_4}) s^{i}_0 
 = 2 \left (1 +  {D_1}/{D_4} \right) S_{0,n} ,
\eens
where $({2D_1}/{D_4}) S_{0,n}$ is an upper bound on $\size{E \setminus E_0}$
by~\eqref{eq::extra-sel}.
Thus
\bens
\nonumber
\fnorm{\Theta_{0,\drop}}^2 
& \leq & \sum_{i=1}^p \theta_{0,ii}^2 \twonorm{\beta^{i}_\drop}^2 
\leq
(1+  (D_0 + D_4)^2) d_0^2  \lambda^2  \sum_{i=1}^p \theta_{0,ii}^2 s^i_0 \\
& \leq & \min\{S_{0,n} (\max_{i=1,...p} \theta_{0,ii}^2), 
s_0 \fnorm{\diag(\Theta_0)}^2 \}(1+  (D_0 + D_4)^2) d_0^2  \lambda^2
\eens
\end{proof}

\begin{remark}
Note that if $s_0$ is small, then the second term in $C_{\diag}$ will 
provide a tighter bound. 
\end{remark}

\begin{proofof}{Lemma~\ref{lemma:threshold-RE}}
Let $T_0 := T_0^i$ denote the $s^i_0$ largest coefficients of $\beta^i$ in 
absolute values. We have
\ben
\label{eq::extra-sel}
|\sel^i \cap T_0^c| \leq 
\norm{\beta^i_{\init, T_0^c}}_1 \inv{f_0 \lambda} \leq 
 D_1 d_0 s^i_0 /(D_4 d_0) \leq 
 D_1 s^i_0 /D_4
\een
by~\eqref{eq::extraone}, where $D_1$ is understood to be the same constant that appears 
in~\eqref{eq::extraone}. Thus we have
\bens
\abs{\sel^i} = |\sel^i \cap T_0^c| + |\sel^i \cap T_0| 
\leq s^i_0 (1+ {D_1}/{D_4}).
\eens
Now the second inequality in~\eqref{eq::ideal-size} clearly 
holds given~\eqref{eq::extra-sel} and the following: 
$$|\sel^{i} \cup S^i| \leq |S^i| + |\sel^{i} \cap (S^i)^c|   
\leq s^i + |\sel^{i} \cap (T_0^i)^c|.$$
We now bound $\twonorm{\beta^i_\drop}^2$ following essentially
the arguments as in~\cite{Zhou09th}. We have
\bens
\twonorm{\beta^i_\drop}^2 
& = &  \twonorm{\beta^i_{T_0 \cap \drop} }^2 + 
\twonorm{\beta^i_{T_0^c \cap \drop}}^2,  
\eens
where for the second term, we have
$\twonorm{\beta^i_{T_0^c \cap \drop}}^2 \leq
\twonorm{\beta^i_{T_0^c}}^2  \leq s^i_0 \lambda^2 \sigma^2_{V_i}$
by definition of $s^i_0$ as in~\eqref{eq::define-s0} and~\eqref{eq::tail-s0};
For the first term, we have by the triangle inequality and \eqref{eq::extraone},
\bens
\twonorm{\beta^i_{T_0 \cap \drop} }
& \leq & 
\twonorm{(\beta^i - \beta^i_{\init})_{T_0 \cap \drop}}
+ \twonorm{(\beta^i_{\init})_{T_0 \cap \drop}} \\
& \leq & 
\twonorm{(\beta^i - \beta^i_{\init})_{T_0}} 
+ t_0 \sqrt{\size{T_0 \cap \drop}}  \leq 
\twonorm{h_{T_0}} + t_0 \sqrt{s^i_0} \\
& \leq & 
D_0 d_0 \lambda \sqrt{s^i_0} + D_4 d_0 \lambda \sqrt{s^i_0} 
\leq  (D_0 + D_4) d_0 \lambda \sqrt{s^i_0}.
\eens
\end{proofof} 
\section{Bounds on MLE refitting}
\label{sec:append-frob-missing}
Recall the maximum likelihood estimate
$\hat \Theta_n$ minimizes over all 
$\Theta \in \Set_n$ the empirical risk:
\beq
\label{eq::refit}
\hat \Theta_n(E)  = 
\arg \min_{\Theta \in \Set_n}  \hat{R}_n(\Theta) := 
\arg \min_{\Theta \in \PDcone \cap \SubE} 
\big\{ {\rm tr}(\Theta\hat{\Gamma}_n) - \log |\Theta| \big\}
\eeq
which gives the ``best'' refitted sparse estimator given 
a sparse subset of edges $E$ that 
we obtain from the nodewise regressions  and thresholding.
We note that the estimator~\eqref{eq::refit} remains to be a convex 
optimization problem, as the constraint set is  the intersection the positive 
definite cone $\PDcone$ and the linear subspace $\SubE$.
Implicitly, by using $\hat\Gamma_n$ rather than $\hat{S}_n$ in~\eqref{eq::refit},
we force the diagonal entries in $(\hat \Theta_n(E))^{-1}$ to be identically $1$.
It is not hard to see that  the estimator~\eqref{eq::refit} is equivalent to
~\eqref{eq::est-mle}, after we replace $\hat{S}_n$ with $\hat\Gamma_n$.
%
\begin{theorem}
\label{thm::frob-missing}
Consider data generating random variables as in expression (\ref{data}) 
and assume that $(A1)$,~\eqref{eq::non-singular}, and
~\eqref{eq::non-singular-2} hold. Suppose $\Sigma_{0, ii} =1$ for all $i$.
Let $\event$ be some event such that $\prob{\event} \geq 1- d/p^2$ for 
a small constant $d$. Let $S_{0,n}$ be as defined in~\eqref{eq::omni-s0};
Suppose on event $\event$:
\begin{enumerate}
\item
We obtain an edge set $E$ 
such that its size $|E| = \lin(S_{0,n})$ is a linear function in $S_{0,n}$.
\item
And for $\tilde{\Theta}_0$ as in \eqref{eq::new-posi} and
for some constant $C_{\bias}$ to be specified, we have
\beq
\label{eq::bias-thm}
\fnorm{\Theta_{0,\drop}}  :=  \fnorm{\tilde{\Theta}_0  -\Theta_0}
\leq C_{\bias} \sqrt{2 S_{0,n}  \log (p) /n} < \ul{c}/32.
\eeq
\end{enumerate}
Let $\hat\Theta_n(E)$ be as defined in~\eqref{eq::refit}.
Suppose the sample size satisfies for $C_3 \geq 4 \sqrt{5/3}$,
\ben
\label{eq::sample-frob}
n > \frac{106}{ \underline{k}^2}
\left(4C_3 +\frac{32}{31 \ul{c}^2} \right)^2
\max\left\{2 |E|\log \max(n,p), \; C^2_{\bias} 2 S_{0,n}  \log p \right\}.
\een
Then on event $\event \cap \X_0$, we have for 
$M = ({9}/(2 \underline{k}^2)) \cdot \left(4C_3 +{32}/(31 \ul{c}^2) \right)$
\begin{equation}
\label{eq::frob-con}
\fnorm{\hat \Theta_n(E) -\Theta_0} \leq 
(M + 1) \max\left\{\sqrt{{2|E|\log\max(n,p)}/{n}}, \; 
C_{\bias} \sqrt{{2 S_{0,n}  \log (p)}/{n}} \right\}.
\end{equation}
\end{theorem}
We note that although Theorem~\ref{thm::frob-missing} is meant for
proving Theorem~\ref{thm:main}, we state it as an independent result;
For example, one can indeed take $E$ from Corollary~\ref{cor:drop-norm},
where we have $|E| \leq c S_{0,n}$ for some constant $c$ for 
$D_4 \asymp D_1$.
In view of~\eqref{eq::bias}, we aim to recover $\tilde{\Theta}_0$
by $\hat{\Theta}_n(E)$ as defined in~\eqref{eq::refit}.
In Section~\ref{sec::frob-missing-proof}, 
we will focus in Theorem~\ref{thm::frob-missing}
on bounding for $W$ suitably chosen,
$$\fnorm{\hat{\Theta}_n(E) - \tilde{\Theta}_0} = 
O_P\left(W\sqrt{{S_{0,n} \log \max(n,p)}/{n}}\right).$$
By the triangle inequality, we conclude that 
\begin{equation*}
\label{eq::frob-missing}
\fnorm{\hat \Theta_n(E) -\Theta_0} \leq 
\fnorm{\hat{\Theta}_n(E) - \tilde{\Theta}_0} + 
\fnorm{\tilde{\Theta}_0  -\Theta_0} =
 O_P\left(W \sqrt{{S_{0,n}\log (n)}/{n}}\right) \ .
\end{equation*}
We now state bounds for the convergence rate on 
Frobenius norm of the covariance matrix and for KL divergence.
We note that constants have not been optimized.
Proofs of Theorem~\ref{thm::frob-missing-cov} 
and~\ref{thm::frob-missing-risk} appear in 
Section~\ref{sec::frob-missing-cov-proof} 
and~\ref{sec::frob-missing-risk-proof} respectively.

\begin{theorem}
\label{thm::frob-missing-cov}
Suppose all conditions, events, and bounds on $|E|$ and 
$\fnorm{\Theta_{0, \drop}}$ in Theorem~\ref{thm::frob-missing}
hold. Let $\hat\Theta_n(E)$ be as defined in~\eqref{eq::refit}.
Suppose the sample size satisfies for $C_3 \geq 4 \sqrt{5/3}$ and
$C_{\bias}, M$ as defined in Theorem~\ref{thm::frob-missing}
\ben
\label{eq::sample-frob-second}
n > \frac{106}{\ul{c}^2 \underline{k}^4}
\left(4 C_3 +\frac{32}{31 \ul{c}^2} \right)^2
\max\left\{2|E|\log \max(p,n), \; C^2_{\bias} 2 S_{0,n}  \log p \right\}.
\een
Then on event $\event \cap \X_0$, we have
$\varphi_{\min}(\hat\Theta_n(E)) >\ul{c}/2 >0$ and for 
$\hat\Sigma_n(E) = (\hat\Theta_n(E))^{-1}$,
\begin{equation}
\label{eq::frob-con-cov}
\fnorm{\hat \Sigma_n(E) - \Sigma_0} \leq 
\frac{2 (M + 1)}{\ul{c}^2}
\max\left\{\sqrt{\frac{2|E|\log \max(n,p)}{n}}, \; 
C_{\bias} \sqrt{\frac{2 S_{0,n}  \log (p)}{n}} \right\}.
\end{equation}
\end{theorem}

\begin{theorem}
\label{thm::frob-missing-risk}
Suppose all conditions, events, and bounds on $|E|$ and 
$\fnorm{\Theta_{0, \drop}} := \fnorm{\tilde\Theta_0 - \Theta_0}$
in Theorem~\ref{thm::frob-missing} hold.
Let $\hat\Theta_n(E)$ be as defined in~\eqref{eq::refit}.
Suppose the sample size satisfies~\eqref{eq::sample-frob}
for $C_3 \geq 4 \sqrt{5/3}$ and $C_{\bias}, M$ as defined in 
Theorem~\ref{thm::frob-missing}.
Then on event $\event \cap \X_0$, we have for 
$R(\hat \Theta_n(E)) - R(\Theta_0) \geq 0$,
\begin{equation}
\label{eq::frob-risk}
R(\hat \Theta_n(E)) - R(\Theta_0) \leq M (C_3 + 1/8)
\max\left\{{2 |E|\log \max(n,p)}/{n}, \; 
C^2_{\bias} {2 S_{0,n}  \log (p)}/{n} \right\}.
\end{equation}
\end{theorem}

\subsection{Proof of Theorem~\ref{thm:main}}
\label{sec:proof-main}
Clearly the sample requirement as in~\eqref{eq::sample-size}
is satisfied for some $\theta >0$ 
that is appropriately chosen, given~\eqref{eq::sample-frob}.
In view of Corollary~\ref{cor:drop-norm}, we have on 
$\event := \X \cap \C_a$: for $C_{\diag}$ as in~\eqref{eq::global-cons}
\ben
\nonumber
|E| & \leq & 2(1+\frac{D_1}{D_4})S_{0,n}  \leq 4 S_{0,n} \; \; \text{for $D_4 \geq D_1$ and }  \\
\nonumber
\fnorm{\Theta_{0,\drop}} & := &
  \fnorm{\tilde{\Theta}_0  -\Theta_0}  \leq
\label{eq::bias-local}
C_{\bias} \sqrt{2 S_{0,n}  \log( p) /n} \leq \ul{c}/32 \; \text{ where }  \\
\nonumber
C_{\bias}^2 & :=  &  
\min\left\{\max_{i=1,...p} \theta_{0,ii}^2, \frac{s_0}{S_{0,n} } \fnorm{\diag(\Theta_0)}^2
\right\} d^2_0 (1 + (D_0 + D_4)^2)  \\
& = & C^2_{\diag} d^2_0 (1 + (D_0 + D_4)^2)
\een
Clearly the last inequality in~\eqref{eq::bias-thm} hold so long as
$n > {32^2 C^2_{\bias} 2 S_{0,n}  \log (p) }/{\ul{c}^2},$
which holds given~\eqref{eq::sample-frob}.
Plugging in $|E|$ in~\eqref{eq::frob-con}, we have on $\event \cap \X_0$,
\begin{equation*}
\fnorm{\hat \Theta_n(E) -\Theta_0} \leq 
(M +1) \max\left\{\sqrt{\frac{ 4(1+{D_1}/{D_4})S_{0,n} )\log  \max(n,p)}{n}}, \; 
C_{\bias}  \sqrt{\frac{2 S_{0,n}  \log p}{n}} \right\}
\end{equation*}
Now if we take
$D_4 \geq D_1$, then we have~\eqref{eq::thm-E-bounds} on event $\event$;
and moreover on $\event \cap \X_0$,
\bens
\fnorm{\hat \Theta_n(E) -\Theta_0}
& \leq & 
(M +1) \max\left\{\sqrt{{8 S_{0,n} \log \max(n,p)}/{n}}, \; 
C_{\bias} \sqrt{{2 S_{0,n}  \log (p)}/{n}} \right\} \\
& \leq & 
W \sqrt{{S_{0,n}\log \max(n,p)}/{n}}
\eens
where $W \leq \sqrt{2} (M +1) \max\{C_{\diag} d_0 \sqrt{1 + (D_0 + D_4)^2}, 2\}$.
Similarly, we get the bound on 
$\fnorm{\hat \Sigma_n -\Sigma_0}$ with Theorem~\ref{thm::frob-missing-cov},
and the bound on risk following Theorem~\ref{thm::frob-missing-risk}.
Thus all statements in Theorem~\ref{thm:main} hold. \qed

\begin{remark}
Suppose event $\event \cap \X_0$ holds.
Now suppose that we take $D_4 = 1$, that is, if we take
the threshold to be exactly the penalty parameter $\lambda_n$:
$$t_0 =  d_0 \lambda := \lambda_n.$$ 
Then we have on event $\event$ by~\eqref{eq::on-edges}
$|E| \leq 2(1+D_1) S_{0,n} $ and $|E \setminus E_0| \leq 2 D_1 S_{0,n}$ and 
on event on $\event \cap \X_0$,  for 
$C'_{\bias} := C_{\diag} d_0 \sqrt{1 + (D_0 + 1)^2}$
\begin{equation*}
\fnorm{\hat \Theta_n(E) -\Theta_0} \leq 
M \max\left\{\sqrt{\frac{4(1+ D_1) S_{0,n}\log \max(n,p)}{n}}, \; 
C'_{\bias} \sqrt{\frac{2 S_{0,n}  \log p}{n}} \right\}
\end{equation*}
It is not hard to see that we achieve essential the same rate 
as stated in Theorem~\ref{thm:main}, with perhaps slightly more edges included in $E$.
\end{remark} 

\subsection{Proof of Theorem~\ref{thm::frob-missing}}
\label{sec::frob-missing-proof}
\begin{proofof2}
Suppose event $\event$ holds throughout this proof.
We first obtain the bound on spectrum of $\tilde{\Theta}_0$:
It is clear that by~\eqref{eq::non-singular} 
and~\eqref{eq::bias-thm}, we have on $\event$,
\ben
\label{eq::pdc}
\varphi_{\min}(\tilde\Theta_0) 
& \geq & 
\varphi_{\min}(\Theta_0) -  \twonorm{\tilde{\Theta}_0  -\Theta_0}
 \geq
\varphi_{\min}(\Theta_0) -  \fnorm{\Theta_{0,\drop}}  >  31 \underline{c}/32, \\
\label{eq::eigenmax}
\varphi_{\max}(\tilde\Theta_0) 
& < & \varphi_{\max}(\Theta_0) + \twonorm{\tilde\Theta_0 - \Theta_0} 
\leq  
\varphi_{\max}(\Theta_0) +  \fnorm{\Theta_{0,\drop}}  <
\frac{\ul{c}}{32} + \inv{\ul{k}}.
\een
Throughout this proof, we let $\Sigma_0 = (\sigma_{0, ij}) := \Theta_0^{-1}$.
In view of~\eqref{eq::pdc}, define $\tilde\Sigma_0 := (\tilde \Theta_0)^{-1}$.
We use $\hat{\Theta}_n := \hat{\Theta}_n(E)$ as a shorthand.

Given $\tilde\Theta_0 \in \PDcone \cap \SubE$ as guaranteed 
in~\eqref{eq::pdc}, let us define a new  convex set:
$$U_n(\tilde\Theta_0) :=  (\PDcone \cap \SubE) - \tilde\Theta_0 = 
\{ B - \tilde\Theta_0 | B \in \PDcone \cap \SubE \} \subset \SubE$$
which is a translation of the original convex set $\PDcone \cap \SubE$.
Let $\ul{0}$ be a matrix with all entries being zero. 
Thus it is clear that $U_n(\tilde\Theta_0) \ni \ul{0}$ given that 
$\tilde\Theta_0 \in \PDcone \cap \SubE$.
Define for $\hat{R}_n$ as in expression~\eqref{eq::refit}
\bens
\label{eq::convex-func}
\tilde Q (\Theta) &  := & \hat{R}_n(\Theta) -
\hat{R}_n(\tilde\Theta_0)  =  \tr (\Theta \hat{\Gamma}_n) - \log |\Theta| - 
\tr (\tilde \Theta_0 \hat{\Gamma}_n) + \log |\tilde\Theta_0| \\
& =  & 
\tr \left(( \Theta- \tilde \Theta_0) (\hat{\Gamma}_n- \tilde \Sigma_0)\right)-
(\log |\Theta| - \log |\tilde \Theta_0|)
+ \tr\left((\Theta- \tilde \Theta_0) \tilde \Sigma_0 \right).
\eens 
For an  appropriately chosen $r_n$ and a large enough $M > 0$, let 
\ben
\TT_n & = & \{ \Delta \in U_n(\tilde \Theta_0), \fnorm{\Delta} = M r_n\}, 
\; \text{ and } \\
\Pi_n & = & \{ \Delta \in U_n(\tilde \Theta_0), \fnorm{\Delta} < M r_n\}.
\een
It is clear that both $\Pi_n$ and $\TT_n \cup \Pi_n$ are convex.
It is also clear that  $\ul{0} \in \Pi_n$.
Throughout this section, we let
\beq
\label{eq::sphere}
r_n = 
\max\left\{\sqrt{\frac{2 |E| \log \max(n,p)}{n}},
C_{\bias} \sqrt{ \frac{2 S_{0,n}  \log p}{n}} \right\}.
\eeq
Define for $\Delta \in U_n(\tilde\Theta_0)$,
\ben
\tilde{G}(\Delta) :=  \tilde{Q}(\tilde \Theta_0 + \Delta) = 
\tr (\Delta (\hat{\Gamma}_n- \tilde \Sigma_0))-
(\log |\tilde \Theta_0 + \Delta| - \log |\tilde \Theta_0|)
+ \tr (\Delta \tilde \Sigma_0 )
\een
It is clear that $\tilde{G}(\Delta)$ is a convex function 
on $U_n(\tilde\Theta_0)$ and $\tilde{G}(\ul{0}) = \tilde{Q}(\tilde{\Theta}_0) = 0$.

Now, $\Hat{\Theta}_n$ minimizes $\tilde{Q}(\Theta)$, or equivalently 
$\hat\Delta = \hat{\Theta}_n - \tilde\Theta_0$ minimizes $\tilde{G}(\Delta)$.
Hence  by definition,
$$\tilde{G}(\hat\Delta) \leq \tilde{G}(\ul{0}) = 0$$
Note that $\TT_n$ is non-empty, while clearly $\ul{0} \in \Pi_n$.
Indeed, consider
$B_{\epsilon} := (1 + \epsilon) \tilde\Theta_0$, where  $\epsilon > 0$;
it is clear that $B_{\epsilon} - \tilde\Theta_0 \in \PDcone \cap \SubE$ and 
$\fnorm{B_{\epsilon} - \tilde\Theta_0} = |\epsilon| \fnorm{\tilde\Theta_0} = M r_n$ for $|\epsilon| = M r_n/\fnorm{\tilde\Theta_0}$.
Note also if $\Delta \in \TT_n$, then $\Delta_{ij} = 0 \forall (i,j : i \neq j) \notin E$;
Thus we have $\Delta \in \SubE$ and 
\beq
\label{eq::edge-count-missing}
\norm{\Delta}_0 = \norm{\diag(\Delta)}_0 + \norm{\offdiag(\Delta)}_0 
\leq p + 2 |E| \; \; \text{where } |E| = \lin(S_{0,n}).
\eeq
We now show the following two propositions.
Proposition~\ref{prop:posi-def} follows from standard results.
\begin{proposition}
\label{prop:posi-def}
Let $B$ be a $p \times p$ matrix.
If $B \succ 0$ and  $B + D \succ 0$, 
then  $B + v D \succ 0$ for all $v \in [0, 1]$.
\end{proposition}

\begin{proposition}
\label{prop:posi-def-interval}
Under~\eqref{eq::non-singular}, we have for all $\Delta \in \TT_n$
such that  $\fnorm{\Delta} = M r_n$ for $r_n$ as in \eqref{eq::sphere},
$\tilde\Theta_0 + v \Delta \succ 0, \forall v \in$ an open 
interval $I \supset [0, 1]$ on event $\event$.
\end{proposition}
\begin{proof}
In view of Proposition~\ref{prop:posi-def},
it is sufficient to show that 
$\tilde\Theta_0 + (1 + \ve) \Delta, \tilde\Theta_0 - \ve \Delta \succ 0$ 
for some $\ve > 0$.
Indeed, by definition of $\Delta \in \TT_n$, we have
$\varphi_{\min} (\tilde\Theta_0 + \Delta)  \succ 0$ on event $\event$; 
thus 
\bens
\varphi_{\min} (\tilde\Theta_0 + (1 + \ve) \Delta) 
& \geq & 
\varphi_{\min} (\tilde\Theta_0 + \Delta) - \ve \twonorm{\Delta}
> 0 \\
\text {and }
\varphi_{\min} (\tilde\Theta_0 - \ve \Delta)  
& \geq &
\varphi_{\min} (\tilde \Theta_0) - \ve \twonorm{\Delta} >
31 \ul{c}/32 - \ve \twonorm{\Delta} >0
\eens
for $\ve >0$ that is sufficiently small.
\end{proof}
Thus we have that $\log|\tilde\Theta_0 + v \Delta|$ is infinitely 
differentiable on the open interval $I \supset [0, 1]$ of $v$. 
This allows us to use the Taylor's 
formula with integral remainder to obtain the following:
\begin{lemma}
\label{lemma:sphere-pos-bound}
On event $\event \cap \X_0$,
$\tilde{G}(\Delta) > 0$ for all $\Delta \in \TT_n$.
\end{lemma}

\begin{proof}
Let us use $\tilde{A}$ as a shorthand for
$$\mvec{\Delta}^T \left( \int^1_0(1-v)
(\tilde\Theta_0 + v \Delta)^{-1} \otimes (\tilde\Theta_0 + v \Delta)^{-1}dv
\right) \mvec{\Delta},$$
where $\otimes$ is the Kronecker product (if
$W= (w_{ij})_{m \times n}$, $P=(b_{k\ell})_{p \times q}$,
then $W \otimes P = (w_{ij}P)_{m p \times nq}$),
and $\mvec{\Delta} \in \R^{p^2}$ is $\Delta_{p \times p}$ 
vectorized. Now, the Taylor expansion gives  
for all $\Delta \in \TT_n$,
\begin{eqnarray*}
\log|\tilde\Theta_0 + \Delta| - \log|\tilde\Theta_0| & = &
\frac{d}{dv}\log|\tilde\Theta_0 + v\Delta||_{v=0} \Delta + 
\int_0^1(1-v) \frac{d^2}{dv^2}  \log|\tilde\Theta_0 + v \Delta| dv \\
& = & {\rm tr}(\tilde\Sigma_0 \Delta) -\tilde{A}.
\end{eqnarray*}
Hence for all $\Delta \in \TT_n$,
\begin{eqnarray}
\label{eq::G2-missing}
\tilde{G}(\Delta)  =  
\tilde{A} +\tr \left( \Delta (\hat{\Gamma}_n - \tilde \Sigma_0) \right)
 =  \tilde{A} +\tr \left( \Delta (\hat{\Gamma}_n - \Sigma_0) \right) 
- \tr \left( \Delta (\tilde \Sigma_0 - \Sigma_0) \right)
\end{eqnarray}
where we first bound $\tr( \Delta (\tilde \Sigma_0 - \Sigma_0) )$ as follows:
by~\eqref{eq::bias-thm} and~\eqref{eq::pdc}, we have on event $\event$
\begin{eqnarray}
\nonumber
\abs{\tr (\Delta (\tilde{\Sigma}_0 - \Sigma_0))}
& = &
\abs{\ip{\Delta, (\tilde{\Sigma}_0 - \Sigma_0)}} 
 \leq \fnorm{\Delta} \fnorm{\tilde{\Sigma}_0 - \Sigma_0} \\ \nonumber
& \leq &
 \fnorm{\Delta} \frac{\fnorm{\Theta_{0,\drop}}}
{ \varphi_{\min}(\tilde{\Theta}_0) \varphi_{\min}(\Theta_0)} \\ 
\label{eq::trace-term}
& < &
\fnorm{\Delta} 
\frac{32 C_{\bias} \sqrt{2 S_{0,n}  \log p/n}}{31 \ul{c}^2}
\leq \fnorm{\Delta}  \frac{32 r_n }{31 \ul{c}^2}.
\end{eqnarray}
Now, conditioned on event $\X_0$, by~\eqref{eq::gamma-devi-bound}
and~\eqref{eq::sample-frob}
$$\max_{j,k}|\hat{\Gamma}_{n,jk} - \sigma_{0,jk}| \leq 4 C_3 \sqrt{\log \max(n,p)/n} 
=: \delta_n$$
and thus on event $\event \cap X_0$, we have 
$ \abs{\tr \big(\Delta  (\hat{\Gamma}_n -\Sigma_0)\big)}
\le \delta_n \abs{\offdiag(\Delta)}_1$, where
$\abs{\offdiag(\Delta)}_1 \leq 
\sqrt{\norm{\offdiag(\Delta)}_0} \fnorm{\offdiag(\Delta)}
\leq \sqrt{2|E|} \fnorm{\Delta}$, 
and 
\begin{eqnarray}  
\label{eq::trace-term-add}
\tr \left( \Delta (\hat{\Gamma}_n -\Sigma_0) \right)
& \geq & -  4 C_3 \sqrt{\log \max(n,p)/n} \sqrt{2 |E|} \fnorm{\Delta}  \geq -4 C_3 r_n  \fnorm{\Delta}.
\end{eqnarray}
Finally, we bound  $\tilde{A}$.
First we note that for $\Delta \in \TT_n $, we have on event $\event$,
\ben
\label{eq::first-M-bound}
\twonorm{\Delta} \leq \fnorm{\Delta} = M r_n < \frac{7}{16 \ul{k}},
\een
given~\eqref{eq::sample-frob}:
$n >  (\frac{16}{7} \cdot \frac{9}{2 \ul{k}})^2 
\left(4C_3 +\frac{32}{31 \ul{c}^2} \right)^2
\max\left\{(2 |E| )\log (n), \; C^2_{\bias} 2 S_{0,n}  \log p \right\}$.
Now we have by~\eqref{eq::eigenmax} and~\eqref{eq::non-singular-2}
following~\cite{RBLZ08} (see Page 502, proof of Theorem~1 therein):
on event $\event$,
\ben
\nonumber
 \tilde{A} & \geq  & 
\fnorm{\Delta}^2/\left(2\left(\varphi_{\max}(\tilde{\Theta}_0) 
+ \twonorm{\Delta}\right)^2\right)  \\
\label{eq::G-first-term}
& \geq & 
\fnorm{\Delta}^2/ \left(2 ( \frac{1}{\ul{k}}
 + \frac{\ul{c}}{32} + \frac{7}{16\ul{k}} )^2\right)
>  \fnorm{\Delta}^2 \frac{2\ul{k}^2}{9} 
\een
Now on event $\event \cap \X_0$,
for all $\Delta \in \TT_n$, we have 
by~\eqref{eq::G2-missing},\eqref{eq::G-first-term},
~\eqref{eq::trace-term-add}, 
and~\eqref{eq::trace-term}, 
\begin{eqnarray*}
\nonumber
\tilde{G}(\Delta) & > & 
 \fnorm{\Delta}^2 \frac{2\ul{k}^2}{9} 
- 4 C_3 r_n \fnorm{\Delta} 
- \fnorm{\Delta}  \frac{32 r_n }{31 \ul{c}^2} \\
&= &
\fnorm{\Delta}^2\left(\frac{2\underline{k}^2}{9} 
-  \inv{\fnorm{\Delta}}
\left(4C_3 r_n +   \frac{32 r_n }{31 \ul{c}^2}  \right) \right) \\
& = & 
\fnorm{\Delta}^2\left(\frac{2 \underline{k}^2}{9}
- \frac{1}{ M}\left(4C_3 + \frac{32}{31 \ul{c}^2} \right) \right)
\end{eqnarray*}
hence we have $\tilde{G}(\Delta) > 0$ for $M$ large enough, in particular
$M = (9 /{(2 \underline{k}^2)}) \left(4C_3 +{32}/{(31 \ul{c}^2)} \right)$
suffices.
\end{proof}
We next state Proposition~\ref{prop:outside-missing}, which follows
exactly that of Claim 12 of~\cite{ZLW08}.
\begin{proposition}
\label{prop:outside-missing}
Suppose event $\event$ holds.
If $\tilde{G}(\Delta) > 0, \forall \Delta \in \TT_n$,
then $\tilde{G}(\Delta) > 0$ for all $\Delta$ in
$$\W_n = \{\Delta: \Delta \in U_n(\tilde\Theta_0), \fnorm{\Delta} > M r_n\}$$
for $r_n$  as in~\eqref{eq::sphere};
Hence if $\tilde{G}(\Delta) > 0$  for all $\Delta \in \TT_n$,
then $\tilde{G}(\Delta) > 0$ for all $\Delta \in \TT_n \cup \W_n$.
\end{proposition}

Note that for $\hat\Theta_n \in \PDcone \cap \SubE$, we have
$\hat\Delta  = \hat\Theta_n - \tilde\Theta_0 \in U_n(\tilde\Theta_0)$.
By Proposition~\ref{prop:outside-missing} and the fact that 
$\tilde{G}(\hat\Delta) \leq \tilde{G}(\ul{0}) =0$ on event $\event$,
we have the following: on event $\event$, 
if $\tilde{G}(\Delta) > 0, \forall \Delta \in \TT_n$ then 
$\|\hat\Delta\|_F < M r_n$, given that 
$\hat\Delta \in  U_n(\tilde\Theta_0) \setminus (\TT_n \cup \W_n)$.
Therefore
\begin{eqnarray*}
\label{eq::end-of-proof}
\prob{\| \hat\Delta \|_F  \geq M r_n} & \leq &
\prob{\event^c} + \prob{\event} \cdot \prob{\| \hat\Delta \|_F  \geq M r_n | \event} \\
& = &
\prob{\event^c} + \prob{\event} \cdot (1 - \prob{\| \hat\Delta \|_F < M r_n | \event}) \\
&   \leq &
\prob{\event^c} +  \prob{\event} \cdot
(1 - \prob{\tilde{G}(\Delta) > 0, \forall \Delta \in \TT_n | \event}) \\
& \leq & 
\prob{\event^c} + \prob{\event} \cdot (1 - \prob{\X_0 | \event}) \\
& = & 
\prob{\event^c} + \prob{\X_0^c \cap \event}  
 \leq \prob{\event^c} + \prob{\X_0^c} \\
& \leq &
\frac{c}{p^2} +\inv{\max(n,p)^2} \leq \frac{c+1}{p^2}.
\end{eqnarray*}
We thus establish that the theorem holds. 
\end{proofof2}

\subsection{Frobenius norm for the covariance matrix}
\label{sec::frob-missing-cov-proof}
We use the bound on $\fnorm{\hat \Theta_n(E) - \Theta_0}$ as 
developed in Theorem~\ref{thm::frob-missing}; in addition,
we strengthen the bound on $M r_n$ in~\eqref{eq::first-M-bound} 
in~\eqref{eq::second-M-bound}.
Before we proceed, we note the following bound
on bias of $(\tilde{\Theta}_0)^{-1}$.
\begin{remark}
\label{eq::Sigma-drop}
Clearly we have on event $\event$, by~\eqref{eq::trace-term}
\begin{eqnarray}
\fnorm{(\tilde{\Theta}_0)^{-1} - \Sigma_0}
& \leq &
\frac{\fnorm{\Theta_{0,\drop}}}
{ \varphi_{\min}(\tilde{\Theta}_0) \varphi_{\min}(\Theta_0)} 
\leq 
\frac{32 C_{\bias} \sqrt{2 S_{0,n}  \log p/n}}{31 \ul{c}^2}
\end{eqnarray}
\end{remark}

\begin{proofof}{\textnormal{Theorem~\ref{thm::frob-missing-cov}}}
Suppose event $\event \cap \X_0$ holds.
Now suppose
$$n >  (\frac{16} {7 \ul{c}} \cdot \frac{9}{2 \ul{k}^2})^2 
\left(C_3 +\frac{32}{31 \ul{c}^2} \right)^2
\max\left\{2|E|\log \max(n,p), \; C^2_{\bias} 2 S_{0,n}  \log p \right\}$$
which clearly holds given~\eqref{eq::sample-frob-second}.
Then in addition to the bound in~\eqref{eq::first-M-bound}, 
on event $\event \cap \X_0$, we have
\ben
\label{eq::second-M-bound}
M r_n < 7 \ul{c}/{16},
\een
for $r_n$ as in~\eqref{eq::sphere}.
Then, by Theorem~\ref{thm::frob-missing}, for the same 
$M$ as therein, on event $\event \cap \X_0$, we have 
$$\fnorm{\hat \Theta_n(E) - \Theta_0} \leq 
(M + 1) \max\left\{\sqrt{{2 |E|\log \max(n,p)}/{n}}, \; 
C_{\bias} \sqrt{{2 S_{0,n}  \log (p)}/{n}} \right\}$$ given that 
sample bound in~\eqref{eq::sample-frob} is clearly satisfied.
We now proceed to bound $\fnorm{\hat{\Sigma}_n - \Sigma_0}$ 
given~\eqref{eq::frob-con}.
First note that by~\eqref{eq::second-M-bound},
we have on event $\event \cap \X_0$ for $M > 7$
\bens
\varphi_{\min}(\hat{\Theta}_n(E))
& \geq &  \varphi_{\min}(\Theta_0) - \twonorm{\hat \Theta_n - \Theta_0}
\geq  \varphi_{\min}(\Theta_0) - \fnorm{\hat \Theta_n - \Theta_0} \\
& \geq  & \ul{c} - (M + 1) r_n > \ul{c}/2.
\eens
Now clearly on event $\event \cap \X_0$,~\eqref{eq::frob-con-cov}
holds by~\eqref{eq::frob-con} and
\begin{eqnarray*}
\fnorm{\hat{\Sigma}_n(E) - \Sigma_0} 
& \leq &
\frac{\fnorm{\hat \Theta_n(E) - \Theta_0}}
{\varphi_{\min}(\hat{\Theta}_n(E)) \varphi_{\min}(\Theta_0)}
< \frac{2}{\ul{c}^2} \fnorm{\hat \Theta_n(E) - \Theta_0}
\end{eqnarray*}
\end{proofof}

\subsection{Risk consistency}
\label{sec::frob-missing-risk-proof}
We now derive the bound on risk consistency.
Before proving Theorem~\ref{thm::frob-missing-risk},
we first state two lemmas given the following decomposition
of our loss in terms of the risk as defined in~\eqref{eq::future-risk}: 
\ben
\label{eq::risk-decomp}
0 \leq R(\hat \Theta_n(E)) - R(\Theta_0) = 
(R(\hat \Theta_n(E)) - R(\tilde\Theta_0)) +
(R(\tilde\Theta_0)  - R(\Theta_0) )
\een
where clearly $R(\hat \Theta_n(E)) \geq R(\Theta_0)$ by definition.
It is clear that $\tilde\Theta_0 \in \Set_n$ for $\Set_n$ as defined
in~\eqref{eq::magic-set}, and thus 
$\hat{R}_n(\tilde\Theta_0) \geq \hat{R}_n(\hat\Theta_n(E))$ 
by definition of 
$\hat\Theta_n(E) \; = \; \arg\min_{\Theta \in S_n} \hat{R}_n(\Theta)$.

We now bound  the two terms on the RHS of~\eqref{eq::risk-decomp},
where clearly $R(\tilde\Theta_0) \geq R(\Theta_0)$.
\begin{lemma}
\label{lemma:risk-I}
On event $\event$, we have for $C_{\bias}, \Theta_0, \tilde\Theta_0$ 
as in Theorem~\ref{thm::frob-missing},
\bens
0 \leq R(\tilde\Theta_0)  - R(\Theta_0) 
 \leq 
(32/(31 \ul{c}))^2 C^2_{\bias} \frac{2 S_{0,n}  \log p}{2 n} 
\leq  (32/(31 \ul{c}))^2 \cdot {r_n^2}/{2} \leq M r_n^2/8
\eens
for $r_n$ as in~\eqref{eq::sphere}, where the last inequality holds
given that $M \geq 9/2 (4 C_3 + 32/(31 \ul{c}^2))$.
\end{lemma}

\begin{lemma}
\label{lemma:risk-II}
Under $\event \cap \X_0$, we have for 
 $r_n$ as in~\eqref{eq::sphere}
and $M, C_3$ as in Theorem~\ref{thm::frob-missing}
\bens
R(\hat\Theta_n(E)) - R(\tilde\Theta_0) \leq M C_3 r_n^2.
\eens
\end{lemma}

\begin{proofof}{\textnormal{Theorem~\ref{thm::frob-missing-risk}}}
We have on $\event \cap \X_0$, for $r_n$ is as in~\eqref{eq::sphere}
\bens
R(\hat \Theta_n(E)) - R(\Theta_0)  =
(R(\hat \Theta_n(E)) - R(\tilde\Theta_0)) +
(R(\tilde\Theta_0)  - R(\Theta_0) )
\leq M r_n^2 (C_3 + 1/8)
\eens
as desired, using Lemma~\ref{lemma:risk-I} and~\ref{lemma:risk-II}.
\end{proofof}

\begin{proofof}{\textnormal{Lemma~\ref{lemma:risk-I}}}
For simplicity, we use $\Delta_0$ as a shorthand for the rest of our proof:
$$\Delta_0 :=  \Theta_{0, \drop} = \tilde{\Theta}_0 - \Theta_0.$$
We use $\tilde{B}$ as a shorthand for
$$\mvec{\Delta_0}^T \left( \int^1_0(1-v)
(\Theta_0 + v \Delta_0)^{-1} \otimes (\Theta_0 + v \Delta_0)^{-1}dv
\right) \mvec{\Delta_0},$$
where $\otimes$ is the Kronecker product.
First, we have for $\tilde\Theta_0, \Theta_0 \succ 0$
\bens
R(\tilde\Theta_0)  - R(\Theta_0)  
& = & 
{\rm tr}(\tilde\Theta_0 \Sigma_0) - \log |\tilde\Theta_0| - 
{\rm tr}(\Theta_0 \Sigma_0) + \log |\Theta_0| \\
& = & 
{\rm tr}((\tilde\Theta_0 - \Theta_0) \Sigma_0)
- \left(\log |\tilde\Theta_0| -  \log |\Theta_0| \right)  := \tilde{B} \geq 0
\eens
where $\tilde{B} = 0$ holds when $\fnorm{\Delta_0} = 0$, and 
in the last equation, we bound the difference between two
$\log |\cdot|$ terms
using the Taylor's formula with integral remainder 
following that in proof of Theorem~\ref{thm::frob-missing};
Indeed, it is clear that on $\event$, we have
$$\Theta_0 + v  \Delta_0 \succ 0 \; \text{ for } \; v \in (-1, 2) \supset [0, 1]$$
given that $\varphi_{\min}(\Theta_0) \geq \ul{c}$ and 
$\twonorm{\Delta_0} \leq \fnorm{\Delta_0} \leq \underline{c}/32$ 
by~\eqref{eq::bias-thm}.
Thus $\log |\Theta_0 + v \Delta_0|$ is infinitely differentiable on the
open interval $I \supset [0, 1]$ of $v$. 
Now, the Taylor expansion gives
\begin{eqnarray*}
\log|\Theta_0 + \Delta_0| - \log|\Theta_0| & = &
\frac{d}{dv}\log|\Theta_0 + v\Delta_0||_{v=0} \Delta_0 + 
\int_0^1(1-v) \frac{d^2}{dv^2}  \log|\Theta_0 + v \Delta_0| dv \\
& = & {\rm tr}(\Sigma_0 \Delta_0) - \tilde{B}.
\end{eqnarray*}
We now obtain an upper bound on $\tilde{B} \geq 0$.
Clearly, we have on event $\event$,
Lemma~\ref{lemma:risk-I} holds given that  
$$\tilde{B} \leq \fnorm{\Delta_0}^2 \cdot
\varphi_{\max}\left( \int^1_0(1-v)
(\Theta_0 + v \Delta_0)^{-1} \otimes (\Theta_0 + v \Delta_0)^{-1}dv\right)$$
where $\fnorm{\Delta_0}^2 \leq C_{\bias}^2 2 S_{0,n} \log (p)/n$ and
\bens
\lefteqn{
\varphi_{\max}\left( \int^1_0(1-v)
(\Theta_0 + v \Delta_0)^{-1} \otimes (\Theta_0 + v \Delta_0)^{-1}dv\right)} \\
& \leq & 
\int^1_0(1-v) \varphi_{\max}^2(\Theta_0 + v \Delta_0)^{-1} dv 
\leq
\sup_{v \in [0, 1]} \varphi_{\max}^2(\Theta_0 + v \Delta_0)^{-1} \int^1_0 (1-v) dv \\
& = &
\half \sup_{v \in [0, 1]} \inv{\varphi_{\min}^2(\Theta_0 + v \Delta_0) } 
= \inv{2 \inf_{v \in [0, 1]} \varphi_{\min}^2(\Theta_0 + v \Delta_0) }\\
& \leq &
\inv{2 \left(\varphi_{\min}(\Theta_0) - \twonorm{\Delta_0}\right)^2}
 \leq \inv{2 \left(31\ul{c}/32\right)^2}
\eens
where clearly for all ${v \in [0, 1]}$, we have
$\varphi_{\min}^2(\Theta_0 + v \Delta_0) \geq 
\left(\varphi_{\min}(\Theta_0) - \twonorm{\Delta_0}\right)^2
\geq \left(31\ul{c}/32\right)^2$, given 
$\varphi_{\min}(\Theta_0) \geq \ul{c}$ and  
$\twonorm{\Delta_0} \leq \fnorm{\Theta_{0,\drop}}
\leq \underline{c}/32$ by~\eqref{eq::bias-thm}.
\end{proofof}

\begin{proofof}{\textnormal{Lemma~\ref{lemma:risk-II}}}
Suppose $R(\hat\Theta_n(E)) - R(\tilde\Theta_0) <0$, then we are done.

Otherwise, assume $R(\hat\Theta_n(E)) - R(\tilde\Theta_0) \geq 0$ throughout 
the rest of the proof. Define
$$\hat\Delta := \hat\Theta_n(E) - \tilde\Theta_0,$$
which by Theorem~\ref{thm::frob-missing}, we have
on event $\event \cap \X_0$, and for $M$ as defined therein, 
$$\fnorm{\hat\Delta} := \fnorm{\hat\Theta_n(E) - \tilde\Theta_0}
\leq M r_n.$$
We have  by definition
$\hat{R}_n(\hat\Theta_n(E)) \leq \hat{R}_n(\tilde\Theta_0)$, and hence
\bens
0 \leq  R(\hat\Theta_n(E)) - R(\tilde\Theta_0)
& = & 
R(\hat\Theta_n(E)) - \hat{R}_n(\hat\Theta_n(E)) + \hat{R}_n(\hat\Theta_n(E))
 - R(\tilde\Theta_0) \\
& \leq & 
R(\hat\Theta_n(E)) - \hat{R}_n(\hat\Theta_n(E)) 
+ \hat{R}_n(\tilde\Theta_0) -  R(\tilde\Theta_0) \\
& = &
{\rm tr}(\hat\Theta_n(E)(\Sigma_0 - \hat{\Gamma}_n)) 
-{\rm tr}(\tilde\Theta_0(\Sigma_0 - \hat{\Gamma}_n))  \\
& = & 
{\rm tr}((\hat\Theta_n(E) -\tilde\Theta_0)(\Sigma_0 - \hat{\Gamma}_n)) = 
{\rm tr}(\hat\Delta(\Sigma_0 - \hat{\Gamma}_n))
\eens
Now, conditioned on event $\event \cap \X_0$,
following the same arguments around~\eqref{eq::trace-term-add},
we have 
\begin{eqnarray*}  
\abs{\tr \left(\hat \Delta (\hat{S}_n -\Sigma_0) \right)}
& \leq & \delta_n\abs{\offdiag(\hat\Delta)}_1
\leq  \delta_n \sqrt{2|E|}  \fnorm{\offdiag(\hat\Delta)} \\
& \leq & 
M r_n C_3 \sqrt{2 |E|\log \max(n,p) /n} \leq M C_3 r_n^2
\end{eqnarray*}
where $\norm{\offdiag(\hat\Delta)}_0 \leq 2 |E|$ by definition,
and $r_n$ is as defined in~\eqref{eq::sphere}.
\end{proofof}

\section{Proof of Theorem~\ref{thm::frob-missing-omega}}
\label{sec:append-MLE-refit-general}
We first bound $\prob{\X_0}$ in 
Lemma~\ref{lemma:gaussian-covars-general}, which follows exactly that of 
Lemma~\ref{lemma:gaussian-covars} as the covariance matrix $\Psi_0$ 
for variables $X_1/\sigma_1, \ldots, X_p/\sigma_p$ satisfy the condition that 
$\Psi_{0,ii} = 1, \forall i \in \{1, \ldots, p\}$.
\begin{lemma}
\label{lemma:gaussian-covars-general} 
For $p < e^{n/4 C_3^2}$, where $C_3 > 4 \sqrt{5/3}$,
we have for $X_0$ as defined in~\eqref{eq::good-random-design-orig}
\begin{eqnarray*}
\prob{\X_0} \ge 1 - 1/\max \{n,p\}^2.
\end{eqnarray*}
\end{lemma}

On event $\X_0$, the following holds
for $\tau = C_3 \sqrt{\frac{\log\max\{p,n\}}{n}} <1/2$, where we assume
$p < e^{n/4 C_3^2}$,
\ben
\label{eq::new-X0}
\forall i, \; \; \size{\frac{\twonorm{X_i}^2}{\sigma_i^2 n} - 1} 
 & \leq &  \tau \\
\label{eq::new-X0-2}
\forall i \not = j, \; \; \size{\inv{n}\ip{X_i/\sigma_i, X_j/\sigma_j} - \rho_{0, ij}}
& \leq & \tau.
\een
Let us first derive the large deviation bound for 
$\size{\hat{\Gamma}_{n, ij} - \rho_{0, ij}}$.
First note that on event $\X_0$
$\sqrt{1- \tau} \leq \twonorm{X_i}/(\sigma_i \sqrt{n})  \leq \sqrt{1 + \tau}$
and for all $i\not= j$ 
\ben
\nonumber
\lefteqn{\size{\hat{\Gamma}_{n, ij} - \rho_{0, ij}} = 
\size{\frac{\hat{S}_{n,ij}}{\hat{\sigma}_i \hat{\sigma}_j}
 -\rho_{0, ij}} \; := \;
\size{\hat{\rho}_{ij}  -\rho_{0, ij}}} \\ \nonumber
 & = & 
\size{\frac{\inv{n}\ip{X_i/\sigma_i, X_j/\sigma_j} - \rho_{0, ij}}{(\twonorm{X_i}/(\sigma_i\sqrt{n}))\cdot( \twonorm{X_j}/(\sigma_j\sqrt{n}))}
 + \frac{\rho_{0, ij}}
{(\twonorm{X_i}/(\sigma_i\sqrt{n}))\cdot( \twonorm{X_j}/(\sigma_j\sqrt{n}))} 
 - \rho_{0, ij}} \\ \nonumber
 & \leq & 
\size{\frac{\inv{n}\ip{X_i/\sigma_i, X_j/\sigma_j} - \rho_{0, ij}}{(\twonorm{X_i}/(\sigma_i\sqrt{n}))\cdot( \twonorm{X_j}/(\sigma_j\sqrt{n}))}}
 +\size{ \frac{\rho_{0, ij}}
{(\twonorm{X_i}/(\sigma_i\sqrt{n}))\cdot( \twonorm{X_j}/(\sigma_j\sqrt{n}))} 
 - \rho_{0, ij}} \\ 
\label{eq::gamma-devi-bound}
& \leq &
\frac{\tau}{1-\tau} + |\rho_{0, ij}|\size{\inv{1-\tau} -1} \leq \frac{2\tau}{1-\tau}
< 4 \tau.
\een

\begin{proofof}{Theorem~\ref{thm::frob-missing-omega}}
For $\tilde{\Theta}_0$ as in \eqref{eq::new-posi}, we define
\bens
\tilde{\Omega}_0 & = & W \tilde{\Theta}_0 W = W (\diag(\Theta_0)) W + W \Theta_{0, E_0 \cap E} W \\
& = & \diag(W \Theta_0 W) + W \Theta_{0, E_0 \cap E} W = 
\diag(\Omega_0) + \Omega_{0, E_0 \cap E}
\eens
where $W = \diag(\Sigma_0)^{1/2}$.
Then clearly $\tilde{\Omega}_0 \in \Set_n$ as $\tilde{\Theta}_0 \in \Set_n$.
We first bound $\fnorm{\Theta_{0,\drop}}$ as follows.
\bens
\fnorm{\Theta_{0,\drop}} 
& \leq &
C_{\bias} \sqrt{2 S_{0,n}  \log (p) /n} < 
\frac{\underline{k}}
{\sqrt{144} \sigma_{\max}^2
\left(4C_3 +\frac{13}{12 \ul{c}^2 \sigma_{\min}^2} \right) }\\
& \leq &
\frac{ \underline{k}\ul{c}^2 \sigma_{\min}^2}
{(48 \ul{c}^2 \sigma_{\min}^2 C_3 + 13)\sigma_{\max}^2}
\leq 
\min\left\{\frac{\underline{k}}
{48 C_3\sigma_{\max}^2},
\frac{\ul{c} \sigma_{\min}^2} {13 \sigma_{\max}^2} \right\}
\leq 
\frac{\ul{c}} {13 \sigma_{\max}^2} 
\eens
\silent{
Thus
\bens
\label{eq::bias-thm-omega}
\fnorm{\Omega_{0,\drop}}  & := &  
 \fnorm{\tilde{\Omega}_0  -\Omega_0}  = 
\fnorm{W (\tilde{\Theta}_0 - \Theta_0) W} \\
& \leq & \max_{i} W_i^2  \fnorm{\Theta_{0,\drop}} \\
& \leq & \sigma_{\max}^2 C_{\bias} \sqrt{2 S_{0,n}  \log (p) /n} <
\ul{c} \sigma_{\min}^2/13.
\eens
}
Suppose event $\event$ holds throughout this proof.
We first obtain the bound on spectrum of $\tilde{\Theta}_0$:
It is clear that by~\eqref{eq::non-singular} 
and~\eqref{eq::bias-thm-omega}, we have on $\event$,
\ben
\label{eq::pdc-omega}
\varphi_{\min}(\tilde\Theta_0) 
& \geq & 
\varphi_{\min}(\Theta_0) -  \twonorm{\tilde{\Theta}_0  -\Theta_0}
 \geq
\varphi_{\min}(\Theta_0) -  \fnorm{\Theta_{0,\drop}}  >  \frac{12 \underline{c}}{13}, \\
\label{eq::eigenmax-omega}
\varphi_{\max}(\tilde\Theta_0) 
& < & \varphi_{\max}(\Theta_0) + \twonorm{\tilde\Theta_0 - \Theta_0} 
\leq  
\varphi_{\max}(\Theta_0) +  \fnorm{\Theta_{0,\drop}}  <
\frac{\ul{c}}{13  \sigma_{\max}^2} + \inv{\ul{k}}.
\een
Throughout this proof, we let $\Sigma_0 = (\sigma_{0, ij}) := \Theta_0^{-1}$.
In view of~\eqref{eq::pdc-omega}, define $\tilde\Sigma_0 := (\tilde \Theta_0)^{-1}$.
Then
\beq
\tilde\Omega_0^{-1} = W^{-1} (\tilde{\Theta}_0)^{-1} W^{-1} 
= W^{-1} \tilde{\Sigma}_0 W^{-1} :=
\tilde\Psi_0
\eeq
We use $\hat{\Omega}_n := \hat{\Omega}_n(E)$ as a shorthand.
Thus we have for $\tilde{\Omega}_0  = W \tilde{\Theta}_0 W$,
\ben
\nonumber
\varphi_{\max}(\tilde\Omega_0) & \leq &  \varphi_{\max}(W) \varphi_{\max}(\tilde\Theta_0)  \varphi_{\max}(W)
\leq \frac{\sigma_{\max}^2}{\ul{k}} + \frac{\ul{c}}{13} \\
\varphi_{\min}(\tilde\Omega_0)
\nonumber
& = & \inv{\varphi_{\max}(\tilde\Psi_0)} =
\inv{\varphi_{\max}(W^{-1} \tilde\Sigma_0 W^{-1})} 
= \inv{\varphi_{\max}(W^{-1})^2 \varphi_{\max}(\tilde\Sigma_0)}
\\
\label{eq::pdc-omega-2}
&= & \frac{\varphi_{\min}(W)^2}{\varphi_{\max}(\tilde\Sigma_0)} = 
\varphi_{\min}(W)^2 \varphi_{\min}(\tilde\Theta_0) 
\geq \sigma_{\min}^2 \frac{12 \ul{c}}{13}
\een
Given $\tilde\Omega_0 \in \PDcone \cap \SubE$ as guaranteed 
in~\eqref{eq::pdc-omega-2}, let us define a new  convex set:
$$U_n(\tilde\Omega_0) :=  (\PDcone \cap \SubE) - \tilde\Omega_0 = 
\{ B - \tilde\Omega_0 | B \in \PDcone \cap \SubE \} \subset \SubE$$
which is a translation of the original convex set $\PDcone \cap \SubE$.
Let $\ul{0}$ be a matrix with all entries being zero. 
Thus it is clear that $U_n(\tilde\Omega_0) \ni \ul{0}$ given that 
$\tilde\Omega_0 \in \PDcone \cap \SubE$.
Define for $\hat{R}_n$ as in expression~\eqref{eq::sample-risk-omega},
\bens
\label{eq::convex-func}
\tilde Q (\Omega) &  := & \hat{R}_n(\Omega) -
\hat{R}_n(\tilde\Omega_0)  =  \tr (\Omega \hat{\Gamma}_n) - \log |\Omega| - 
\tr (\tilde \Omega_0 \hat{\Gamma}_n) + \log |\tilde\Omega_0| \\
& =  & 
\tr \left(( \Omega- \tilde \Omega_0) (\hat{\Gamma}_n- \tilde \Psi_0)\right)-
(\log |\Omega| - \log |\tilde \Omega_0|)
+ \tr\left((\Omega- \tilde \Omega_0) \tilde \Psi_0 \right).
\eens 
For an  appropriately chosen $r_n$ and a large enough $M > 0$, let 
\ben
\TT_n & = & \{ \Delta \in U_n(\tilde \Omega_0), \fnorm{\Delta} = M r_n\}, 
\; \text{ and } \\
\Pi_n & = & \{ \Delta \in U_n(\tilde \Omega_0), \fnorm{\Delta} < M r_n\}.
\een
It is clear that both $\Pi_n$ and $\TT_n \cup \Pi_n$ are convex.
It is also clear that  $\ul{0} \in \Pi_n$.
Define for $\Delta \in U_n(\tilde\Omega_0)$,
\ben
\tilde{G}(\Delta) :=  \tilde{Q}(\tilde \Omega_0 + \Delta) = 
\tr (\Delta (\hat{\Gamma}_n- \tilde \Psi_0))-
(\log |\tilde \Omega_0 + \Delta| - \log |\tilde \Omega_0|)
+ \tr (\Delta \tilde \Psi_0 )
\een
It is clear that $\tilde{G}(\Delta)$ is a convex function 
on $U_n(\tilde\Omega_0)$ and $\tilde{G}(\ul{0}) = \tilde{Q}(\tilde{\Omega}_0) = 0$.

Now, $\Hat{\Omega}_n$ minimizes $\tilde{Q}(\Omega)$, or equivalently 
$\hat\Delta = \hat{\Omega}_n - \tilde\Omega_0$ minimizes $\tilde{G}(\Delta)$.
Hence  by definition,
$$\tilde{G}(\hat\Delta) \leq \tilde{G}(\ul{0}) = 0$$
Note that $\TT_n$ is non-empty, while clearly $\ul{0} \in \Pi_n$.
Indeed, consider
$B_{\epsilon} := (1 + \epsilon) \tilde\Omega_0$, where  $\epsilon > 0$;
it is clear that $B_{\epsilon} - \tilde\Omega_0 \in \PDcone \cap \SubE$ and 
$\fnorm{B_{\epsilon} - \tilde\Omega_0} = |\epsilon| \fnorm{\tilde\Omega_0} = M r_n$ for $|\epsilon| = M r_n/\fnorm{\tilde\Omega_0}$.
Note also if $\Delta \in \TT_n$, then $\Delta_{ij} = 0 \forall (i,j : i \neq j) \notin E$;
Thus we have $\Delta \in \SubE$ and 
\beq
\label{eq::edge-count-missing}
\norm{\Delta}_0 = \norm{\diag(\Delta)}_0 + \norm{\offdiag(\Delta)}_0 
\leq p + 2 |E| \; \; \text{where } |E| = \lin(S_{0,n}).
\eeq
We now show the following proposition.
\begin{proposition}
\label{prop:posi-def-interval}
Under~\eqref{eq::non-singular}, we have for all $\Delta \in \TT_n$
such that  $\fnorm{\Delta} = M r_n$ for $r_n$ as in \eqref{eq::sphere},
$\tilde\Omega_0 + v \Delta \succ 0, \forall v \in$ an open 
interval $I \supset [0, 1]$ on event $\event$.
\end{proposition}
\begin{proof}
In view of Proposition~\ref{prop:posi-def},
it is sufficient to show that 
$\tilde\Omega_0 + (1 + \ve) \Delta, \tilde\Omega_0 - \ve \Delta \succ 0$ 
for some $\ve > 0$.
Indeed, by definition of $\Delta \in \TT_n$, we have
$\varphi_{\min} (\tilde\Omega_0 + \Delta)  \succ 0$ on event $\event$; 
thus 
\bens
\varphi_{\min} (\tilde\Omega_0 + (1 + \ve) \Delta) 
& \geq & 
\varphi_{\min} (\tilde\Omega_0 + \Delta) - \ve \twonorm{\Delta}
> 0 \\
\text {and }
\varphi_{\min} (\tilde\Omega_0 - \ve \Delta)  
& \geq &
\varphi_{\min} (\tilde \Omega_0) - \ve \twonorm{\Delta} >
12  \sigma_{\min}^2  \ul{c}/13 - \ve \twonorm{\Delta} >0
\eens
for $\ve >0$ that is sufficiently small.
\end{proof}
Thus we have that $\log|\tilde\Omega_0 + v \Delta|$ is infinitely 
differentiable on the open interval $I \supset [0, 1]$ of $v$. 
This allows us to use the Taylor's 
formula with integral remainder to obtain the following:
\begin{lemma}
\label{lemma:sphere-pos-bound}
On event $\event \cap \X_0$,
$\tilde{G}(\Delta) > 0$ for all $\Delta \in \TT_n$.
\end{lemma}

\begin{proof}
Let us use $\tilde{A}$ as a shorthand for
$$\mvec{\Delta}^T \left( \int^1_0(1-v)
(\tilde\Omega_0 + v \Delta)^{-1} \otimes (\tilde\Omega_0 + v \Delta)^{-1}dv
\right) \mvec{\Delta},$$
where $\otimes$ is the Kronecker product (if
$W= (w_{ij})_{m \times n}$, $P=(b_{k\ell})_{p \times q}$,
then $W \otimes P = (w_{ij}P)_{m p \times nq}$),
and $\mvec{\Delta} \in \R^{p^2}$ is $\Delta_{p \times p}$ 
vectorized. Now, the Taylor expansion gives  
for all $\Delta \in \TT_n$,
\begin{eqnarray*}
\log|\tilde\Omega_0 + \Delta| - \log|\tilde\Omega_0| & = &
\frac{d}{dv}\log|\tilde\Omega_0 + v\Delta||_{v=0} \Delta + 
\int_0^1(1-v) \frac{d^2}{dv^2}  \log|\tilde\Omega_0 + v \Delta| dv \\
& = & {\rm tr}(\tilde\Psi_0 \Delta) -\tilde{A}.
\end{eqnarray*}
Hence for all $\Delta \in \TT_n$,
\begin{eqnarray}
\label{eq::G2-missing-omega}
\tilde{G}(\Delta)  =  
\tilde{A} +\tr \left( \Delta (\hat{\Gamma}_n - \tilde \Psi_0) \right)
 =  \tilde{A} +\tr \left( \Delta (\hat{\Gamma}_n - \Psi_0) \right) 
- \tr \left( \Delta (\tilde \Psi_0 - \Psi_0) \right)
\end{eqnarray}
where we first bound $\tr( \Delta (\tilde \Psi_0 - \Psi_0) )$ as follows:
by~\eqref{eq::bias-thm-omega} and~\eqref{eq::pdc}, we have on event $\event$
\begin{eqnarray}
\nonumber
\abs{\tr (\Delta (\tilde{\Psi}_0 - \Psi_0))}
& = &
\abs{\ip{\Delta, (\tilde{\Psi}_0 - \Psi_0)}} 
 \leq \fnorm{\Delta} \fnorm{\tilde{\Psi}_0 - \Psi_0} \\
\label{eq::trace-term-omega}
& \leq & 
\fnorm{\Delta}  \frac{13 r_n }{12 \sigma_{\min}^2  \ul{c}^2}
\end{eqnarray}
where we bound $\fnorm{\tilde{\Psi}_0  -\Psi_0}$ as follows:
\bens
\label{eq::bias-thm-psi}
\fnorm{\tilde{\Psi}_0  -\Psi_0} & = &  
\fnorm{W^{-1} (\tilde{\Sigma}_0 - \Sigma_0) W^{-1}}
\leq \max_{i} W_i^{-2}  \fnorm{\tilde{\Sigma}_0 - \Sigma_0} \\ 
& \leq & \inv{\sigma_{\min}^2}
\frac{\fnorm{\Theta_{0,\drop}}}{\varphi_{\min}(\tilde\Theta_0) \varphi_{\min}(\Theta_0) }\\
& \leq &  
\frac{C_{\bias} \sqrt{2 S_{0,n}  \log p/n}}{12 \sigma_{\min}^2 \ul{c}^2/13} 
\leq \frac{13 r_n }{12 \sigma_{\min}^2  \ul{c}^2}
\eens

Now, conditioned on event $\X_0$, by~\eqref{eq::gamma-devi-bound}
$$\max_{j,k}|\hat{\Gamma}_{n,jk} - \rho_{0,jk}| \leq 4 C_3 \sqrt{\log \max(n,p)/n} 
=: \delta_n$$
and thus on event $\event \cap X_0$, we have 
$ \abs{\tr \big(\Delta  (\hat{\Gamma}_n -\Psi_0)\big)}
\le \delta_n \abs{\offdiag(\Delta)}_1$, where
$\abs{\offdiag(\Delta)}_1 \leq 
\sqrt{\norm{\offdiag(\Delta)}_0} \fnorm{\offdiag(\Delta)}
\leq \sqrt{2|E|} \fnorm{\Delta}$, 
and 
\begin{eqnarray}  
\label{eq::trace-term-add-omega}
\tr \left( \Delta (\hat{\Gamma}_n -\Psi_0) \right)
& \geq & -  4 C_3 \sqrt{\log \max(n,p)/n} \sqrt{2 |E|} \fnorm{\Delta}  \geq -4 C_3 r_n  \fnorm{\Delta}.
\end{eqnarray}
Finally, we bound  $\tilde{A}$.
First we note that for $\Delta \in \TT_n $, we have on event $\event$,
\ben
\label{eq::first-M-bound-omega}
\twonorm{\Delta} \leq \fnorm{\Delta} = M r_n < \frac{3 \sigma_{\max}^2}{8 \ul{k}},
\een
given~\eqref{eq::sample-frob-omega}:
$n >  (\frac{8}{3} \cdot \frac{9}{2 \ul{k}})^2 \sigma_{\max}^4 
\left(4C_3 +\frac{13}{12 \sigma_{\min}^2 \ul{c}^2} \right)^2
\max\left\{2 |E| )\log \max(n,p), C^2_{\bias} 2 S_{0,n}  \log p \right\}$.
Now we have by~\eqref{eq::eigenmax-omega} and~\eqref{eq::non-singular-2}
following~\cite{RBLZ08} (see Page 502, proof of Theorem~1 therein):
on event $\event$,
\ben
\nonumber
 \tilde{A} & \geq  & 
\fnorm{\Delta}^2/\left(2\left(\varphi_{\max}(\tilde{\Omega}_0) 
+ \twonorm{\Delta}\right)^2\right)  \\
\label{eq::G-first-term-omega}
& > & 
\fnorm{\Delta}^2/ \left(2 \sigma_{\max}^4 \left(\frac{1}{\ul{k}}
 + \frac{\ul{c}}{13} + \frac{3}{8\ul{k}} \right)^2\right)
>  \fnorm{\Delta}^2 \frac{2\ul{k}^2}{9 \sigma_{\max}^4} 
\een
Now on event $\event \cap \X_0$,
for all $\Delta \in \TT_n$, we have 
by~\eqref{eq::G2-missing-omega},\eqref{eq::G-first-term-omega},
~\eqref{eq::trace-term-add-omega}, 
and~\eqref{eq::trace-term-omega}, 
\begin{eqnarray*}
\nonumber
\tilde{G}(\Delta) & > & 
 \fnorm{\Delta}^2 \frac{2\ul{k}^2}{9\sigma_{\max}^4} 
- 4 C_3 r_n \fnorm{\Delta} 
- \fnorm{\Delta}  \frac{13 r_n }{12 \sigma_{\min}^2  \ul{c}^2} \\
&= &
\fnorm{\Delta}^2\left(\frac{2\underline{k}^2}{9\sigma_{\max}^4} 
-  \inv{\fnorm{\Delta}}
\left(4C_3 r_n +   \frac{13 r_n }{12 \sigma_{\min}^2  \ul{c}^2}  \right) \right) \\
& = & 
\fnorm{\Delta}^2\left(\frac{2 \underline{k}^2}{9\sigma_{\max}^4}
- \frac{1}{ M}\left(4C_3 + \frac{13}{12 \sigma_{\min}^2 \ul{c}^2} \right) \right)
\end{eqnarray*}
hence we have $\tilde{G}(\Delta) > 0$ for $M$ large enough, in particular
$M = (9\sigma_{\max}^4 /{(2 \underline{k}^2)}) \left(4C_3 +{13}/{(12 \sigma_{\min}^2 \ul{c}^2)} \right)$
suffices.
\end{proof}
The rest of the proof follows that of Theorem~\ref{thm::frob-missing}, see  
Proposition~\ref{prop:outside-missing} and the bounds which follow.
We thus establish that the theorem holds. 
\end{proofof}

\section{Oracle inequalities for the Lasso}
\label{sec:app-lasso-bounds}
In this section, we consider
recovering $\beta \in \R^p$ in the following linear model:
\beq
\label{eq::linear-model}
Y = X \beta + \epsilon,
\eeq
where $X$ follows~\eqref{data} and 
$\e \sim N(0, \sigma^2 I_n)$.
Recall given $\lambda_n$, the Lasso estimator for $\beta \in \R^p$
is defined as:
\begin{eqnarray}
\label{eq::origin} \; \; 
\hat \beta = \arg\min_{\beta} \frac{1}{2n}\|Y-X\beta\|_2^2 + 
\lambda_n \|\beta\|_1,
\end{eqnarray}
which corresponds to the regression function in~\eqref{eq::lasso} by 
letting $Y := X_i$ and $X := X_{\cdot \setminus i}$
where  $X_{\cdot \setminus i}$ denotes columns of $X$ without $i$.
Define $s_0$ as the smallest integer such that
\ben
\label{eq::define-s0-local}
\sum_{i=1}^p \min(\beta_i^2, \lambda^2 \sigma^2) \leq 
s_0 \lambda^2 \sigma^2, \text{ where } \; \lambda = \sqrt{ 2 \log p/n}.
\een
For $X \in \F(\theta)$ as defined
in~\eqref{eq::good-random-design-diag}, define
\ben
\label{eq::low-noise}
{\T_a} = 
\biggl \{\e: \norm{\frac{X^T \e}{n}}_{\infty} \leq (1+ \theta) \basepen,
\; \text{ where } \;  X \in \F(\theta), \text{ for } 0< \theta < 1\biggr \},
\een
where $\basepen = \sigma \sqrt{1 + a} \sqrt{(2\log p)/n}$, where
$a \geq 0$. We have (cf. Lemma~\ref{lemma:gaussian-noise})
\ben
\label{eq::prob-T}
\; \; \; \;
\prob{\T_a}  \geq 
1 - (\sqrt{\pi \log p} p^a)^{-1};
\een
In fact, for such a bound to hold, we only need 
$\frac{\twonorm{X_j}}{\sqrt{n} } \leq 1 + \theta, \forall j$ 
to hold in $\F(\theta)$.

We now state Theorem~\ref{thm:RE-oracle}, which may be of independent 
interests as the bounds on $\ell_2$ and $\ell_1$ loss for the 
Lasso estimator are stated with respect to the {\em actual} sparsity
$s_0$ rather than $s = |\supp(\beta)|$ as in~\citet[Theorem 7.2]{BRT09}.
The proof is omitted as on event $\T_a \cap \X$, 
it follows exactly that of~\citet[Theorem 5.1]{Zhou10} for a 
deterministic design matrix $X$ which satisfies the RE condition, 
with some suitable  adjustments on the constants.
\begin{theorem}\textnormal{\bf (Oracle inequalities of the Lasso)~\cite{Zhou10}}
\label{thm:RE-oracle}
Let $Y = X \beta + \e$, for 
$\e$ being i.i.d. $N(0, \sigma^2)$ and
let $X$ follow~\eqref{data}.
Let $s_0$ be as in~\eqref{eq::define-s0-local} and $T_0$
denote locations of the $s_0$ largest coefficients of $\beta$ in absolute 
values.
Suppose that $RE(s_0, 4, \Sigma_0)$ holds with $K(s_0, 4, \Sigma_0)$ 
and $\rho_{\min}(s) >0$.
Fix some $1> \theta > 0$.
Let $\beta_{\init}$ be an optimal solution to~\eqref{eq::origin} with 
\ben
\label{eq::single-penalty}
\lambda_{n} = d_0 \lambda \sigma \geq 2 (1 + \theta) \basepen
\een
where $a \geq 1$ and $d_0 \geq 2 (1 + \theta) \sqrt{1 + a}$.
Let $h = \beta_{\init} - \beta_{T_0}$. 
Define
$$\X := \RE(\theta) \cap \F(\theta) \cap \M(\theta).$$
Suppose that  $n$ satisfies~\eqref{eq::sample-size}.
Then on $\T_a \cap \X$, we have
\bens
\twonorm{\beta_{\init} - \beta}
& \leq & \lambda_n  \sqrt{s_0} \sqrt{2D_0^2 + 2D_1^2 + 2}
:=  \lambda \sigma  \sqrt{s_0} d_0 \sqrt{2D_0^2 + 2D_1^2 + 2}, \\
\norm{h_{T_0^c}}_1 & \leq & D_1 \lambda_n s_0 :=
 D_1 d_0 \lambda \sigma s_0,
\eens
where $D_0$ and $D_1$ are defined in~\eqref{eq::D0-define} 
and~\eqref{eq::D1-define} respectively, and 
$\prob{\X \cap \T_a} 
\geq 1- 3\exp(-\bar{c} \theta^2 n/\alpha^4) -  (\sqrt{\pi \log p} p^a)^{-1}$.
\end{theorem}
Let $T_1$ denote the $s_0$ largest positions of $h$  in absolute values 
outside of $T_0$; Let $T_{01} := T_0 \cup T_1$. 
The proof of Theorem~\ref{thm:RE-oracle} yields the following bounds on
$\X \cap \T_a$:
$\twonorm{h_{T_{01}}} \leq D_0 d_0 \lambda \sigma \sqrt{s_0}$ where
\ben
\label{eq::D0-define}
\lefteqn{D_0 = \max \left\{\frac{D}{d_0}, \
2 \sqrt{2} (1+ \theta) 
\frac{K(s_0, 4, \Sigma_0) \sqrt{\rho_{\max}(s- s_0)}}{(1-\theta) d_0} 
+ \frac{3 \sqrt{2}  K^2(s_0, 4, \Sigma_0)}{(1-\theta)^2}\right\}} \\ \nonumber
& & \text{ where }  D =\frac{3 (1+\theta) \sqrt{\rho_{\max}(s - s_0)}}
{(1- \theta) \sqrt{\rho_{\min}(2s_0)}}
+ \frac{2 (1+\theta)^4 \rho_{\max}(3s_0)  \rho_{\max}(s - s_0)}
{d_0 (1-\theta)^2 \rho_{\min}(2s_0)}, \;
\een
and
\ben
D_1 =
\label{eq::D1-define} 
\max\left\{
\frac{4 (1+ \theta)^2 \rho_{\max}(s- s_0)}{d_0^2}
, \left(\frac{(1+ \theta) \sqrt{\rho_{\max}(s- s_0)}}{d_0} + 
\frac{3 K(s_0, 4, \Sigma_0)}{2(1- \theta)} \right)^2 \right\}.
\een
We note that implicit in these constants, we have used 
the concentration bounds for $\Lambda_{\max}(3s_0)$,
$\Lambda_{\max}(s-s_0)$ and $\Lambda_{\min}(2s_0)$ as derived in 
Theorem~\ref{thm:corn}, given that~\eqref{eq::phi-max-bound} holds for 
$m \leq \max(s, (k_0 + 1) s_0)$, where we take $k_0 > 3$. 
In general, these maximum sparse eigenvalues as defined above will
increase with $s_0$ and $s$; Taking this issue into consideration, 
we fix for $c_0 \geq 4 \sqrt{2}$, $\lambda_n = d_0 \lambda \sigma$ where
$$d_0 = c_0 (1+ \theta)^2 \sqrt{\rho_{\max}(s-s_0) \rho_{\max}(3s_0)} \geq 2 (1+ \theta) \sqrt{1 + a},$$
where the second inequality holds for $a = 7$ as desired, given 
$\rho_{\max}(3s_0), \rho_{\max}(s-s_0) \geq 1$.

Thus we have
for  $\rho_{\max}(3s_0) \geq \rho_{\max}(2s_0) \geq \rho_{\min}(2s_0)$
\bens
D/d_0 & \leq & 
\frac{3}
{c_0 (1+\theta) (1- \theta)  \sqrt{\rho_{\max}(3s_0)} \sqrt{\rho_{\min}(2s_0)}}
+ \frac{2}{c_0^2 (1-\theta)^2 \rho_{\min}(2s_0)} \\
& \leq & 
\frac{3 \sqrt{\rho_{\min}(2s_0)}}
{c_0 (1- \theta)^2 \sqrt{\rho_{\max}(3s_0)} \rho_{\min}(2s_0)}
+ \frac{2}{c_0^2 (1-\theta)^2 \rho_{\min}(2s_0)} \\
& \leq & 
\frac{2 (3 c_0 + 2)K^2(s_0, 4, \Sigma_0) }{c_0^2 (1- \theta)^2}
\leq 
\frac{7 \sqrt{2} K^2(s_0, 4, \Sigma_0)}{8 (1-\theta)^2}
\eens
which holds given that  $\rho_{\max}(3s_0) \geq 1$, and
$1 \leq \inv{\sqrt{\rho_{\min}(2s_0)}} \leq \sqrt{2}K(s_0, k_0, \Sigma_0)$,
and thus $\inv{K^2(s_0, k_0, \Sigma_0)} \leq 2$
as shown in Lemma~\ref{lemma:sparse-lower-eigen};
Hence
\bens
D_0 & \leq &
\max \left\{D/d_0, \frac{(4 + 3 \sqrt{2} c_0) 
\sqrt{\rho_{\max}(s-s_0) \rho_{\max}(3s_0)}
(1+\theta)^2 K^2(s_0, 4, \Sigma_0)} {d_0 (1-\theta)^2}
 \right\}, \\
& \leq &
 \frac{7 K^2(s_0, 4, \Sigma_0) }{\sqrt{2} (1-\theta)^2}  
< \frac{5 K^2(s_0, 4, \Sigma_0) }{(1-\theta)^2}  \; \text{ and } \\
D_1 
& \leq & 
\left(\frac{6}{4(1- \theta)} + \inv{4}\right)^2 K^2(s_0, 4, \Sigma_0)
\leq 
\frac{49 K^2(s_0, 4, \Sigma_0) }{16 (1- \theta)^2},
\eens
where for both $D_1$, we have used the fact that 
\bens
\frac{2 (1+\theta)^2 \rho_{\max}(s-s_0)}{d_0^2} & = & 
\frac{2}{c_0^2 (1+\theta)^2 \rho_{\max}(3s_0)}  \leq 
\frac{2}{c_0^2 (1+\theta)^2 \rho_{\min}(2s_0)}  \\
& \leq & 
\frac{4 K^2(s_0, 4, \Sigma_0)}{c_0^2 (1+\theta)^2} \leq 
\frac{K^2(s_0, 4, \Sigma_0)}{8}.
\eens
\section{Misc bounds}
\begin{lemma}
\label{lemma:gaussian-noise}
For fixed design $X$ with $\max_j \|X_j\|_2 \le (1 + \theta) \sqrt{n}$, 
where $0< \theta < 1$,
we have for $\T_a$ as defined in~\eqref{eq::low-noise},  where $a >0$, 
$\prob{\T^c_a} \leq (\sqrt{\pi \log p} p^a)^{-1}.$
\end{lemma}
\begin{proof}
Define random variables: $Y_j = \inv{n} \sum_{i=1}^n \e_i X_{i,j}.$
Note that $\max_{1 \le j \le p} |Y_j| = \|X^T \epsilon/n\|_{\infty}$. 
We have $\expct{(Y_j)} = 0$ and
$\var{(Y_j)} = \twonorm{X_j}^2\sigma^2/n^2 \leq (1 + \theta) \sigma^2/n$. 
Let $c_1 = 1+ \theta$.
Obviously, $Y_j$ has its tail probability dominated by that of
$Z \sim N(0, \frac{c_1^2 \sigma^2}{n})$:
\begin{eqnarray*}
\prob{|Y_j| \geq t} \leq  \prob{|Z| \geq t} \leq 
\frac{2 c_1 \sigma}{\sqrt{2 \pi n} t} \exp\left(\frac{-n t^2}{2 c_1^2
    \sigma_{\e}^2}\right).
\end{eqnarray*}
We can now apply the union bound to obtain:
\bens
\prob{\max_{1 \leq j \leq p} |Y_j| \geq t } & \leq & 
p \frac{c_1 \sigma}{\sqrt{n} t} \exp\left(\frac{-n t^2}{2 c_1^2
    \sigma^2}\right) \\
&=& 
\exp\left(-\left(\frac{n t^2}{2 c_1^2 \sigma^2} + 
\log \frac{t \sqrt{\pi n}}{\sqrt{2} c_1 \sigma} - \log p\right) \right).
\eens
By choosing $t = c_1 \sigma \sqrt{1+a}\sqrt{2 \log p/n}$, 
the right-hand side is bounded by $(\sqrt{\pi \log p} p^a)^{-1}$ for $a \geq 0$.
\end{proof}

\begin{lemma}
\label{lemma:sparse-lower-eigen}(\cite{Zhou10sub})
Suppose that $RE(s_0, k_0, \Sigma_0)$ holds for $k_0 >0$, then
for $m = (k_0 +1)s_0$,
\bens
\sqrt{\rho_{\min}(m)}
& \geq &  \inv{\sqrt{2+k_0^2}K(s_0, k_0, \Sigma_0)}; \; \text{ and clearly}  \\
\text{if } \; \Sigma_{0,ii} = 1, \forall i, \; \; \text{then} \; 
1 \geq 
\sqrt{\rho_{\min}(2s_0)}
& \geq & \inv{\sqrt{2}K(s_0, k_0, \Sigma_0)} \; \; \text{ for }\; k_0 \geq 1.
\eens
\end{lemma}

\end{appendix}

\bibliography{jmlrZ}
\end{document}